\documentclass[11pt]{article}

\usepackage[letterpaper,margin=1in]{geometry}

\usepackage{amsthm,amsmath,amssymb,amsfonts,amssymb,mathtools}
\usepackage{color}
\usepackage{xcolor}
\usepackage{empheq}
\usepackage{dsfont}
\usepackage{mathrsfs}
\usepackage{graphicx}
\usepackage{enumitem}
\usepackage{subfig}
\usepackage[T1]{fontenc}
\usepackage{microtype}
\usepackage{tikz}
\usetikzlibrary{patterns}
\usetikzlibrary{arrows,shapes,automata,backgrounds,petri,positioning}
\usetikzlibrary{shadows}
\usetikzlibrary{calc}
\usetikzlibrary{spy}
\usepackage{pgf,pgfplots}
\usetikzlibrary{angles, quotes}
\usepackage{pgfmath,pgffor}

\usepackage{framed}
\usepackage{multirow}
\usepackage{algorithm2e}
\usepackage[noend]{algpseudocode}
\SetAlCapSkip{1em}

\usepackage[colorlinks,citecolor=blue,linkcolor=magenta,bookmarks=true]{hyperref}

\usepackage[nameinlink]{cleveref}
\crefname{ineq}{Inequality}{Inequality}
\creflabelformat{ineq}{#2{\upshape(#1)}#3}
\crefname{sub}{Subsection}{Subsection}
\creflabelformat{Subsection}{#2{\upshape(#1)}#3}
\crefname{sdp}{SDP}{SDP}
\creflabelformat{sdp}{#2{\upshape(#1)}#3}
\crefname{lp}{LP}{LP}
\creflabelformat{lp}{#2{\upshape(#1)}#3}

\newtheorem{theorem}{Theorem}[section]
\newtheorem{question}[theorem]{Question}

\newtheorem{lemma}[theorem]{Lemma}
\newtheorem{informal theorem}[theorem]{Theorem (informal statement)}

\newtheorem{proposition}[theorem]{Proposition}

\newtheorem{claim}[theorem]{Claim}
\newtheorem{fact}[theorem]{Fact}

\newtheorem{remark}[theorem]{Remark}
\newtheorem{infproposition}[theorem]{(Informal) Proposition}

\newtheorem{definition}[theorem]{Definition}
\newcommand{\eqdef}{\coloneqq}
\crefname{question}{question}{questions}

\def\colorful{0}

\ifnum\colorful=1
\newcommand{\new}[1]{{\color{red} #1}}

\newcommand{\newblue}[1]{{\color{blue} #1}}
\else
\newcommand{\new}[1]{{#1}}

\newcommand{\newblue}[1]{{#1}}

\fi

\newcommand{\wperp}{(\vec w^*)^{\perp_\vec w}}

\newcommand\snorm[2]{\left\| #2 \right\|_{#1}}
\renewcommand\vec[1]{\mathbf{#1}}
\DeclareMathOperator*{\pr}{\mathbf{Pr}}
\DeclareMathOperator*{\E}{\mathbf{E}}
\newcommand{\proj}{\mathrm{proj}}

\def\d{\mathrm{d}}
\newcommand{\normal}{\mathcal{N}}

\DeclareMathOperator*{\argmin}{argmin}

\newcommand{\sample}[2]{#1^{(#2)}}
\newcommand{\tr}{\mathrm{tr}}
\newcommand{\bx}{\mathbf{x}}
\newcommand{\by}{\mathbf{y}}

\newcommand{\bw}{\mathbf{w}}

\newcommand{\p}{\mathbf{P}}

\newcommand{\R}{\mathbb{R}}

\newcommand{\Z}{\mathbb{Z}}

\newcommand{\eps}{\epsilon}

\newcommand{\poly}{\mathrm{poly}}

\newcommand{\var}{\mathbf{Var}}

\newcommand{\sgn}{\mathrm{sign}}
\newcommand{\sign}{\mathrm{sign}}

\newcommand{\calL}{{\cal L}}

\newcommand{\opt}{\mathrm{OPT}}
\newcommand{\D}{\mathcal{D}}

\newcommand{\Ind}{\mathds{1}}
\newcommand{\1}{\Ind}
\newcommand{\matr}{\vec}

\newcommand{\littlesum}{\mathop{\textstyle \sum}}

\newcommand{\wt}{\widetilde}
\newcommand{\wh}{\widehat}

\newcommand{\dotp}[2]{ #1 \cdot #2 }

\newcommand{\wstar}{\bw^{\ast}}
\newcommand{\x}{\vec x}

\newcommand{\ith}{^{(i)}}
\pgfplotstableread[]{
-2.	-1.0653
-1.9	-1.08428
-1.8	-1.07767
-1.7	-1.05018
-1.6	-1.01136
-1.5	-0.972763
-1.4	-0.944848
-1.3	-0.934522
-1.2	-0.943636
-1.1	-0.968817
-1.	-1.00254
-0.9	-1.03516
-0.8	-1.0574
-0.7	-1.06273
-0.6	-1.04907
-0.5	-1.01949
-0.4	-0.981648
-0.3	-0.946106
-0.2	-0.92386
-0.1	-0.923521
0.	-0.948756
0.1	-0.996555
0.2	-1.05672
0.3	-1.11282
0.4	-1.14445
0.5	-1.13064
0.6	-1.05372
0.7	-0.90293
0.8	-0.677301
0.9	-0.386923
1.	-0.0524557
1.1	0.297285
1.2	0.629357
1.3	0.911752
1.4	1.11865
1.5	1.23495
1.6	1.25919
1.7	1.20417
1.8	1.09511
1.9	0.965274
2.	0.849927
2.1	0.779489
2.2	0.77322
2.3	0.834766
2.4	0.950635
2.5	1.09228
2.6	1.22174
2.7	1.30009
2.8	1.29725
2.9	1.20101
3.	1.02326
}\mytable

\title{Learning {\em General} Halfspaces with {\em General} Massart Noise\\ 
under the Gaussian Distribution}

\author{
Ilias Diakonikolas\thanks{Supported by NSF Medium Award CCF-2107079,
NSF Award CCF-1652862 (CAREER), a Sloan Research Fellowship, and
a DARPA Learning with Less Labels (LwLL) grant.}\\
UW Madison\\
{\tt ilias@cs.wisc.edu}\\
\and
Daniel M. Kane\thanks{Supported by NSF Medium Award CCF-2107547,
NSF Award CCF-1553288 (CAREER), and a Sloan Research Fellowship.}\\
UC San-Diego \\
{\tt dakane@ucsd.edu}\\
\and
Vasilis Kontonis\\
UW Madison\\
{\tt kontonis@wisc.edu }\\
\and
Christos Tzamos\\
UW Madison\\
{\tt tzamos@wisc.edu}
\and
Nikos Zarifis\thanks{Supported in part by NSF Award CCF-1652862 (CAREER) and a DARPA Learning with Less Labels (LwLL) grant.}\\
UW Madison\\
{\tt zarifis@wisc.edu}\\
}

\begin{document}

\maketitle
\begin{abstract}
We study the problem of PAC learning halfspaces on $\R^d$ with Massart noise 
under the Gaussian distribution. In the Massart model, an adversary is allowed to
flip the label of each point $\x$ with unknown probability $\eta(\x) \leq \eta$, for
some parameter $\eta \in [0,1/2]$.  The goal of the learner is to output a
hypothesis with misclassification error of $\opt + \eps$, where $\opt$ is the
error of the target halfspace.  This problem \new{had been previously studied} under two assumptions: 
(i) the target halfspace is {\em homogeneous} (i.e.,
the separating hyperplane goes through the origin), and (ii) the parameter $\eta$ is
{\em strictly} smaller than $1/2$. \new{Prior to this work, 
no nontrivial bounds were known for the general case when either of these
assumptions is removed. Here we study the general problem
and establish the following results:} 
\begin{itemize}[leftmargin = *]

\item For $\eta <1/2$, we give a learning algorithm for general halfspaces
with sample and computational complexity $d^{O_{\eta}(\log(1/\gamma))} \poly(1/\eps)$,
where $\gamma \eqdef \max\{\eps,  \min\{\pr[f(\bx) = 1],  \pr[f(\bx) = -1]\} \}$ \new{is the ``bias'' 
of the target halfspace $f$}. \new{Prior efficient algorithms could only handle the special 
case of $\gamma = 1/2$. Interestingly, we also establish a qualitatively matching 
lower bound of $d^{\Omega(\log(1/\gamma))}$ on the complexity of any Statistical Query (SQ) algorithm of the problem.}

\item For $\eta = 1/2$, we give a learning algorithm for general halfspaces 
with sample and computational complexity $O_\eps(1) \,d^{O(\log(1/\eps))}$. 
\new{This result is new even for the subclass of homogeneous halfspaces; prior algorithms
for homogeneous Massart halfspaces provide vacuous guarantees for $\eta=1/2$.
We complement our upper bound with a nearly-matching SQ lower bound of $d^{\Omega(\log(1/\eps) )}$, 
which holds even for the special case of homogeneous halfspaces.}

\end{itemize}
Taken together, \new{our results qualitatively characterize the complexity of
learning general halfspaces with general Massart noise under Gaussian marginals.}
Our techniques rely on determining the existence (or non-existence) of low-degree 
polynomials whose expectations distinguish Massart halfspaces from random noise.
\end{abstract}

\setcounter{page}{0}
\thispagestyle{empty}
\newpage

\section{Introduction}
This work focuses on the distribution-specific PAC learning of halfspaces in the presence of label
noise.  Before we describe our contributions, we provide the context for this work.

\subsection{Background} \label{ssec:background}
A {\em halfspace} or Linear Threshold Function (LTF) is any Boolean-valued function $f: \R^d \mapsto
\{\pm 1\}$ of the form $f(\x) = \sgn(\vec w^\ast \cdot \x - t^\ast)$, for a vector $\vec w^\ast \in
\R^d$ (known as the weight vector) and a scalar $t^\ast \in \R$ (known as the  threshold).
Halfspaces are a central class of Boolean functions in several areas of computer science, including
complexity theory, learning theory, and optimization~\cite{Rosenblatt:58, Novikoff:62,
MinskyPapert:68, Yao:90, GHR:92, FreundSchapire:97, Vapnik:98, CristianiniShaweTaylor:00,
ODonnellbook}. In this work, we focus on the algorithmic problem of learning halfspaces from labeled
examples, arguably one of the most extensively studied and influential problems in machine learning.

The computational problem of PAC learning halfspaces is known to be efficiently solvable without
noise (see, e.g.,~\cite{MT:94}) in the distribution-independent setting. The complexity of the
problem in the presence of noisy data crucially depends on the noise model and the underlying
distributional assumptions. In this work, we study the problem of distribution-specific PAC learning
of halfspaces in the presence of Massart noise.  Formally, we have the following definition.

\begin{definition}[Distribution-specific PAC Learning with Massart Noise]
\label{def:massart-learning} 
Let $\mathcal{C}$ be a concept class of Boolean functions over $X= \R^d$, $\mathcal{F}$ be a {\em
known family} of structured distributions on $X$,  $0 \leq \eta \leq 1/2$, and $0< \eps <1$.  Let
$f$ be an unknown target function in $\mathcal{C}$.  A {\em noisy example oracle},
$\mathrm{EX}^{\mathrm{Mas}}(f, \mathcal{F}, \eta)$, works as follows: Each time
$\mathrm{EX}^{\mathrm{Mas}}(f, \mathcal{F}, \eta)$ is invoked, it returns a labeled example $(\bx,
y)$, such that: (a) $\bx \sim \D_{\bx}$, where $\D_{\bx}$ is a fixed distribution in $\mathcal{F}$,
and (b) $y = f(\bx)$ with probability $1-\eta(\bx)$ and $y = -f(\bx)$ with probability $\eta(\bx)$,
for an {\em unknown} function $\eta(\bx) \leq \eta$.  Let $\D$ denote the joint distribution on
$(\bx, y)$ generated by the above oracle.  A learning algorithm is given i.i.d.\ samples from $\D$
and its goal is to output a hypothesis $h$ such that with high probability it holds $\pr_{(\bx, y)
\sim \D}[h(\bx) \neq y ] \leq  \opt +  \eps$, where $\opt = \min_{c \in \mathcal{C}} \pr_{(\bx, y)
\sim \D}[c(\bx) \neq y ]$.  
\end{definition}

\begin{remark} \label{rem:one-half}
{\em The noise rate parameter $\eta$ in \Cref{def:massart-learning} is allowed to be {\em equal} to
$1/2$. This is consistent with the original definition of the Massart model in~\cite{Massart2006},
and --- as we argue in the subsequent discussion --- is well-motivated in a number of practical
applications. On the other hand, prior algorithmic work in the theoretical machine learning
community imposed the crucial requirement that $\eta$ is {\em strictly smaller} than $1/2$.  This
distinction turns out to be very significant and serves as one of the main motivations for the
current work.}
\end{remark}

The Massart noise model in the above form was defined in~\cite{Massart2006}.  A very similar noise
model had been defined in the 80s by Sloan and Rivest~\cite{Sloan88, Sloan92, RivestSloan:94,
Sloan96}, and a related definition had been considered even earlier by Vapnik~\cite{Vapnik82}.  The
Massart model is a generalization of the Random Classification Noise (RCN) model~\cite{AL88} (where
the flipping probability is uniform) and is a special case of the agnostic model (where the label
noise is fully adversarial)~\cite{Haussler:92, KSS:94}.

The Massart model is a natural semi-random noise model that is more realistic than RCN.
Specifically, label noise can reflect computational difficulty, ambiguity, or random factors.  For
example, a cursive  ``e'' might be substantially more likely to be misclassified as ``a'' than an
upper case Roman letter.  Massart noise allows for such variations in misclassification rates 
without knowledge of which instances are more likely to be misclassified.  That is, Massart
noise-tolerant learners are less brittle than RCN tolerant learners.  Agnostic learning is, of
course, even more robust; unfortunately, agnostic learning is known to be computationally
intractable in many settings of interest~\cite{GR:06, FeldmanGKP06, Daniely16, DKZ20, GGK20,
DKPZ21}.

\new{\subsection{Motivation for this Work} \label{ssec:motivation}
The algorithmic task of PAC learning with Massart noise is a classical problem in computational
learning theory. In the distribution-independent setting, 
known efficient algorithms~\cite{DGT19,CKMY20} achieve error $\eta+\eps$ and 
we now know~\cite{DK20-SQ-Massart} that this error bound is close to best possible
for efficient Statistical Query (SQ) algorithms~\cite{Kearns:98}, 
even if $\opt = \E[\eta(\x)]$ is very small. This lower bound further motivates
the study of the distribution-specific setting (the focus of the current work).

The work of~\cite{AwasthiBHU15} initiated the algorithmic study of learning Massart
halfspaces under structured distributions. This work focused on the class of {\em homogeneous}
halfspaces, i.e., functions of the form $f(\x) = \sgn(\vec w^\ast \cdot \x)$, and 
gave a polynomial-time learning algorithm with $\opt+\eps$ error under 
the uniform distribution on the unit sphere. The~\cite{AwasthiBHU15} algorithm succeeds 
when the parameter $\eta$ is smaller than a sufficiently small constant ($\approx 10^{-6}$).
A sequence of subsequent works~\cite{AwasthiBHZ16, YanZ17, ZhangLC17,
BZ17, MangoubiV19, DKTZ20,ZSA20,ZL21} have given efficient learning algorithms 
for homogeneous Massart halfspaces that succeed for all $\eta<1/2$ and under 
weaker distributional assumptions. The current state-of-the-art algorithms~\cite{DKTZ20,ZSA20,ZL21} 
have sample and computational complexity $\poly(d, 1/\eps, 1/(1-2\eta))$ and succeed for all 
$\eta<1/2$ under a class of distributions including isotropic log-concave distributions.

To summarize the preceding discussion, {\em known efficient algorithms for Massart halfspaces
in the distribution-specific setting succeed under two crucial assumptions:
\begin{itemize}
\item[(i)] The target halfspace is homogeneous (i.e., has zero threshold), and
\item[(ii)] The upper bound parameter of the Massart noise satisfies $\eta<1/2$.
\end{itemize}}

Perhaps surprisingly, prior to the current work, no non-trivial bounds were known
for the {\em general case} of this learning problem, where either of these two assumptions is removed.
This represents a fundamental gap in our algorithmic understanding of learning halfspaces 
in the Massart model and serves as the main motivation of the current work.

In this work, we study the general version of this problem for the prototypical setting
that the examples are drawn from the Gaussian distribution. As our main contribution, {\em we essentially
characterize the complexity of the problem by giving the first efficient learning algorithms coupled with 
qualitatively matching SQ lower bounds. }
}

\new{In the following paragraphs, we provide a more detailed technical motivation of the regimes
we study followed by a detailed description of our results.}

\paragraph{Massart Learning of {\em General} Halfspaces}
\new{Suppose that the Massart noise rate $\eta$ is a constant {\em strictly smaller} than $1/2$. 
Even for this ``low-noise'' regime, all previous efficient learning algorithms that achieve error $\opt+\eps$
require that the unknown halfspace is {\em homogeneous}.
Superficially, this might seem like an innocuous assumption.  
After all, it seems straightforward to reduce a general halfspace
to a homogeneous one by adding an extra constant coordinate to every sample.
Given this, one could use a learner for the homogeneous case on the modified instance.
It turns out that this intuition is fundamentally flawed. While such a reduction is valid
in the distribution-independent setting, it does not work in the distribution-specific setting because
it alters the marginal distribution on the examples.}
In light of this state of affairs, it is natural to ask whether an efficient Massart learner exists
for {\em general} halfspaces in the low-noise noise regime where previous algorithms succeeded.  

\begin{question} \label{que:biased}
\emph{What is the complexity of learning \emph{general} halfspaces in the 
\emph{constant-bounded} Massart noise setting, i.e., when $\eta = 1/2-c$ for some universal
constant $c>0$?  }
\end{question}
 
\new{
As we will show, the complexity of the problem in this regime is characterized by the ``bias'' $\gamma$
of the target halfspace (see \Cref{def:bias}) and (inherently) scales {\em quasi-polynomially} 
with $1/\gamma$ (\Cref{thm:intro-bounded-massart-upper-bound,thm:intro-biased-sq-lb}). 
We note that previous algorithms only handle the special case of $\gamma = 1/2$.}

\paragraph{\new{Learning Halfspaces with {\em General} Massart Noise}} 
The original definition of the Massart noise model~\cite{Massart2006} allows 
the upper bound on the noise rate to be {\em equal} to $1/2$. On the other
hand, all known Massart learning algorithms crucially rely on the assumption that $\eta<1/2$. 
\new{The latter assumption might have been motivated by the random classification noise model,
where the $\eta = 1/2$ regime is not meaningful. In the Massart model, however,
it may well be the case that $\eta(\bx) = 1/2$, for a small probability subset 
of the domain and $\eta(\bx)$ is small otherwise, in which case the optimal misclassification
error $\opt$ will be similarly small.}

Understanding the complexity of learning halfspaces \new{with {\em general} Massart noise, 
i.e., in the regime where $\eta = 1/2$,} is both of theoretical
and practical significance. From the theoretical standpoint, the ``high-noise'' regime subsumes the
well-studied Tsybakov model~\cite{MT99,tsybakov2004optimal} and can be viewed as the strongest known
non-adversarial noise model in the literature.  Interestingly, the \new{general Massart noise model} 
has also been previously studied in the statistical learning theory literature 
under the name {\em benign noise}, see, e.g.,~\cite{hanneke09, HannekeY15}.  
In the ``benign noise'' model, the only assumption made about the
label noise is that the {\em Bayes optimal classifier} lies in the target class. (It is an easy
exercise, see \Cref{app:benign-massart}, that the benign noise model is equivalent to the
\new{general} Massart model.) In addition to its theoretical interest, the \new{general} Massart model 
naturally arises in a number of practical applications. A concrete example is that of \emph{human annotator
noise}~\cite{chhikara1984linear, beigman2009annotator, beigman2009learning,KB10}, where it has been
shown \cite{KB10} that human annotators (especially non-experts) often flip coins 
(corresponding to $\eta(\x) = 1/2$) when presented with a hard-to-classify example. 
The preceding discussion motivates the following question:

\begin{question} \label{que:high-noise}
\emph{What is the complexity of learning halfspaces \new{with {\em general} Massart noise?}}
\end{question}

\new{Interestingly, prior to this work, this question remained wide-open even for the special case 
of homogeneous halfspaces.
Specifically, previous Massart learning algorithms (for homogeneous halfspaces) 
provide {\em vacuous guarantees} for $\eta = 1/2$, 
because they require sample complexity scaling polynomially 
with the parameter $1/\beta$, where $\beta := 1-2\eta$. }
It is worth noting that a dependence on $\beta$
is not information-theoretically required for the problem.
Specifically, $O(d/\eps^2)$ samples suffice to achieve error $\opt+\eps$ (see, e.g.,~\cite{Massart2006}), 
alas with an exponential time algorithm.  

At a high-level, the $\beta$-dependence in previous approaches
is due to the fact that previous algorithms solve the (harder) {\em parameter recovery} problem, i.e.,
they approximate the hidden weight vector $\vec w^\ast$ (e.g., within small angle). 
While the sample complexity of PAC learning (\Cref{def:massart-learning}) is independent of $\beta$, 
parameter learning requires at least $1/\beta$ samples. 
Consequently, a genuinely new approach is required to handle the $\beta  = 0$ regime.  
That is, to answer \Cref{que:high-noise}, we need to understand to what extent it is
possible to PAC learn without relying on parameter recovery.

\new{As we will show, handling the $\beta = 0$ case comes at a cost. Our results 
(\Cref{thm:intro-high-noise,thm:intro-homogeneous-sq-lb}) establish an
information-computation tradeoff for this problem (that persists even for homogeneous halfspaces) 
scaling {\em quasi-polynomially} with the parameter $1/\eps$.}

\subsection{Our Results} \label{ssec:results}

\new{In this work, we answer~\Cref{que:high-noise,que:biased} by providing
both efficient learning algorithms and nearly-matching lower bounds in the Statistical Query (SQ) model.
Perhaps surprisingly, we show that the complexity of our learning problem scales quasi-polynomially
with $1/\eps$, where $\eps$ is the excess error.}
A conceptual implication of our findings is that \Cref{que:high-noise,que:biased}, 
while seemingly orthogonal, are in fact intimately connected \new{at a technical level}.

\paragraph{Learning with Constant-Bounded \new{Massart} Noise}
\new{Here we address the regime of learning general halfspaces with $\eta$-Massart noise, 
where $\eta \leq 1/2-c$ for some constant $c>0$.
It turns out that the complexity of the problem in this setting depends on 
the ``bias'' of the target halfspace.}

\begin{definition}[$(1-\gamma)$-biased Halfspace] \label{def:bias} 
\new{For $\gamma \in [0,1/2]$,} we say that the halfspace \new{$f$} is at most $(1-\gamma)$-biased 
\new{(with respect to $\D$)} if 
$\min(\pr_{(\x,y) \sim \D}[f(\x) = +1],  \pr_{(\x,y) \sim \D}[f(\x) = -1]) \geq \gamma$.
\end{definition}

\new{Note that homogeneous halfspaces correspond to the special case of $\gamma = 1/2$. 
Recall that prior literature gave $\poly(d/\eps)$-time learners for homogeneous Massart 
halfspaces with $\eta<1/2$. Our first main result is a learning algorithm for general halfspaces 
whose complexity scales quasi-polynomially with $1/\gamma$.}
More specifically, we establish the following theorem (see also \Cref{thm:massart-upper-bound}):

\begin{theorem}[Learning General Massart Halfspaces with Constant-Bounded Noise] 
\label{thm:intro-bounded-massart-upper-bound}
Let $\D$ be a distribution on $\R^d\times\{\pm1\}$ with standard normal $\x$-marginal that satisfies
the constant-bounded Massart noise condition with respect to an at most $(1-\gamma)$-biased
halfspace. There exists an algorithm that draws $N = d^{O(\log( \min(1/\gamma, 1/\eps)))}
\poly(1/\eps)$ samples from $\D$, runs in time $\poly(N, d)$, and computes a halfspace $h$ such that
with high probability it holds \(\pr_{(\x, y) \sim \D}[h(\x) \neq y] \leq \opt  + \eps\).
\end{theorem}

\new{
Qualitatively, \Cref{thm:intro-bounded-massart-upper-bound} yields a $\poly(d/\eps)$-time
algorithm for halfspaces with bias bounded away from zero, specifically $\gamma \geq c$
where $c>0$ is an absolute constant. Furthermore, for general halfspaces it yields
a PTAS with quasi-polynomial dependence on $1/\eps$, i.e., with runtime $d^{O(\log(1/\eps))}$.

Perhaps surprisingly, we provide strong evidence that the upper bound of 
\Cref{thm:intro-bounded-massart-upper-bound} is best possible 
for any value of $\gamma$. Specifically, we prove a matching lower bound in 
the Statistical Query (SQ) model of~\cite{Kearns:98}. We recall that 
SQ algorithms are a broad class of algorithms that are only allowed
to query expectations of bounded functions of the distribution rather than directly access samples; 
see \Cref{ssec:SQ-prelims} for the definition and additional discussion.}
Formally, we prove  (see also \Cref{thm:lower_bound}):

\begin{theorem}[SQ Lower Bound for Massart Halfspaces with Constant-Bounded Noise] \label{thm:intro-biased-sq-lb}
Let $\D$ be a distribution on $\R^d\times\{\pm1\}$ with standard normal $\x$-marginal that
satisfies the Massart noise condition with parameter $\eta = 1/2 - \Omega(1) \new{<1/2}$ with respect to an at most
$(1-\gamma)$-biased halfspace. For any $\gamma > \eps$, any SQ algorithm that \newblue{for any such distribution $\D$} 
learns a hypothesis $h:\R^d \mapsto \{\pm 1\}$ such that $\pr_{(\x,y)\sim \D}[h(\x)\neq y]\leq \opt+\eps$, either
requires queries with tolerance at most $d^{-\Omega(\log(1/\gamma))}$ or makes at least
$2^{d^{\Omega(1)}}$ statistical queries.
\end{theorem}

\new{
Informally, \Cref{thm:intro-biased-sq-lb} shows that no SQ algorithm can learn
the subclass of $(1-\gamma)$-biased halfspaces in the constant-bounded Massart noise model 
with sub-exponential in $d^{\Omega(1)}$ many queries, 
unless it uses queries of very small tolerance  -- that would require at least $d^{\Omega(\log(1/\gamma))}$
samples to simulate (as long as $\gamma >\eps$\footnote{Of course, if $\gamma \leq \eps$,
one of the two constant functions suffices.}). 
This ``fine-grained'' lower bound that can be viewed as an information-computation 
tradeoff for the problem (within the class of SQ algorithms) and matches the upper bound of 
\Cref{thm:intro-bounded-massart-upper-bound}. 
As a corollary, we obtain that learning general halfspaces with
constant-bounded Massart noise has SQ complexity $d^{\Theta(\log(1/\eps))}$.

\Cref{thm:intro-biased-sq-lb,thm:intro-bounded-massart-upper-bound} together qualitatively
characterize the complexity of learning general halfspaces in the constant-bounded Massart
setting. We view this inherent quasi-polynomial dependence as rather surprising.
Even though the class of general halfspaces has one additional parameter compared to the homogeneous case 
(the unknown threshold), the learning problem becomes harder and exhibits a quasi-polynomial dependence 
on the inverse of the bias parameter $\gamma$.
}

\new{
\begin{remark} \label{rem:eta-close-to-half}
{\em While the algorithm of \Cref{thm:intro-bounded-massart-upper-bound} 
is essentially optimal in the constant-bounded Massart regime, its running time
has quasi-polynomial dependence on $1/\beta$, where $\beta := 1-2\eta$. 
As a result, the algorithm is not optimal (as far as we know) when the parameter 
$\eta$ is {\em very} close to $1/2$, e.g., $\eta = 1/2-1/d$. 
Characterizing the complexity of the learning problem when $\eta$ is very close to (but not equal to) $1/2$ 
is left as an open problem for future work.}
\end{remark}

}

\paragraph{Learning with \new{{\em General} Massart Noise}}
\new{Our second main result essentially characterizes the complexity
of learning halfspaces with general Massart noise, i.e., the $\eta = 1/2$ case.
On the positive end, we develop an algorithm for this general setting with the following guarantees.}

\begin{theorem}[Learning \new{General} Halfspaces \new{with General Massart Noise}] \label{thm:intro-high-noise} 
Let $\D$ be a distribution on $\R^{d} \times \{\pm 1\}$ whose $\x$-marginal is the standard normal
and such that $\D$ satisfies the Massart noise condition for $\eta = 1/2$ with respect to a general 
target halfspace. There is an algorithm that draws $N = d^{ O(\log(1/\eps))}$ samples
from $\D$, runs in time $2^{\poly(1/\eps)} \poly(N, d)$, and computes a halfspace $h$ such that with
high probability $\pr_{(\x, y) \sim \D}[h(\x) \neq y] \leq \opt  + \eps$.
\end{theorem}

\new{
We reiterate that no nontrivial algorithm was known in the general Massart model, 
even for homogeneous halfspaces (our algorithm works for general halfspaces).
Qualitatively, \Cref{thm:intro-high-noise} gives a PTAS with 
runtime $O_{\eps}(1) d^{O(\log(1/\eps))}$ for the general problem. 
Note that if $\eps \geq 1/\log^{\Omega(1)}(d)$, the runtime of our algorithm is $d^{O(\log(1/\eps))}$.

It is worth comparing \Cref{thm:intro-high-noise} with the algorithm of \Cref{thm:intro-bounded-massart-upper-bound},
which handles the case where $\eta <c$ for some constant $c<1/2$. 
The latter result yields a $\poly(d/\eps)$ time algorithm when the bias of the target halfspace is a positive constant. 
On the other hand, the algorithm 
of \Cref{thm:intro-high-noise} has a quasi-polynomial dependence in $1/\eps$, even for the subclass
of homogeneous halfspaces. 

Perhaps surprisingly, we establish an SQ lower bound suggesting that this quasi-polynomial dependence
is necessary in the general Massart model, even for homogeneous halfspaces.
}

\begin{theorem}[SQ Lower Bound for Homogeneous Halfspaces \new{with General Massart Noise}] 
\label{thm:intro-homogeneous-sq-lb}
Let $\D$ be a distribution on $\R^d\times\{\pm 1\}$ whose $\x$-marginal is the standard normal and
such that $\D$ satisfies the Massart noise condition for $\eta = 1/2$ with respect to a homogeneous
halfspace. Then any SQ algorithm that \newblue{for any such distribution $\D$} 
outputs a hypothesis $h$ with $\pr_{(\x,y)\sim \D}[h(\x)\neq y]\leq\opt +\eps$, 
either requires queries with tolerance at most $d^{-\Omega(\log(1/\eps))}$ or
makes at least $2^{d^{\Omega(1)}}$ statistical queries.
\end{theorem}

(For a more detailed statement, see \Cref{thm:lower_bound_benign}.)
Informally, \Cref{thm:intro-homogeneous-sq-lb} shows that no SQ algorithm can learn
homogeneous halfspaces \new{in the general Massart model} with sub-exponential in $d^{\Omega(1)}$ many queries, 
unless it uses queries of very small tolerance  -- that would require at least $d^{\Omega(\log(1/\eps))}$
samples to simulate. \new{Note that the sample complexity of the algorithm establishing 
\Cref{thm:intro-high-noise} (which can  be implemented in the SQ model) matches our SQ lower bound. 
Furthermore, this implies that the runtime of our algorithm is essentially optimal (within SQ algorithms) 
for all \new{$\eps \geq 1/\log^{\Omega(1)}(d)$}.}

\begin{remark} \label{rem:agnostic-separation}
{\em It is worth noting that \Cref{thm:intro-high-noise,thm:intro-homogeneous-sq-lb} provably separate 
the SQ complexity of learning halfspaces with \new{general} Massart noise from the SQ complexity 
of the corresponding agnostic learning problem. For agnostically learning halfspaces under Gaussian marginals, the
$L_1$-regression algorithm~\cite{KKMS:08} has complexity $d^{O(1/\eps^2)}$ (and is known to be
implementable in SQ). Moreover, a matching SQ lower bound of $d^{\Omega(1/\eps^2)}$ is
known~\cite{DKPZ21}. That is, while agnostic learning requires runtime $d^{\poly(1/\eps)}$, 
Massart learning can be achieved in time $O_{\eps}(1) \, d^{O(\log(1/\eps))}$.}
\end{remark}

\begin{remark}\label{rem:larger-than-half}
{\em Interestingly, the Massart model remains meaningful even for the ``very large'' noise setting,
where $\eta>1/2$. For this extreme regime, it is not hard to show
(see~\Cref{app:semi-random-agnostic}) that this model becomes {\em equivalent} to the agnostic
model.}
\end{remark}

\begin{figure}[h!]
    \centering
\begin{tikzpicture}
    \pgftransformscale{.8}

\draw[very thick] (8,0) -- (-8,0);
    
\node at (0,2.5)[text width=3cm,align=center] {\scriptsize{Homogeneous}};
    \node at (0,2)[text width=3cm,align=center] {\scriptsize{$\eta=1/2-\Omega(1)$}};
    \node at (0,0.8)[text width=3cm,align=center] {\small{$\poly(d/\eps)$} };
\draw[dashed,color=red] (-2.3,0) parabola bend (0,4.5) (8,0);
    \node at (4,2.4)[text width=3cm,align=center] {\scriptsize{Homogeneous}};
    \node at (4,2)[text width=3cm,align=center] {\scriptsize{$\eta=1/2$}};
    \node  at (4.3,1)[text width=3cm,align=center] {$d^{\Omega(\log(1/\eps))}$};
    \node  at (4.3,0.5)[text width=3cm,align=center] {\scriptsize{\Cref{thm:intro-homogeneous-sq-lb}} };
\draw [dashed, color=blue](2.3,0) parabola bend (0,4.5) (-8,0);
    \node at (-4,2.4)[text width=3cm,align=center]  {\scriptsize{General}};
    \node at (-4,2)[text width=3cm,align=center] {\scriptsize{$\eta=1/2-\Omega(1)$}};
    \node  at (-4.3,1)[text width=3cm,align=center] {$d^{\Theta(\log(1/\eps))}$};
    \node  at (-4.7,0.5)[text width=5cm,align=center] {\scriptsize{\Cref{thm:intro-bounded-massart-upper-bound}, 
    \Cref{thm:intro-biased-sq-lb} } 
    };
    
\draw (-8,0) parabola bend (0,7) (8,0);
    \node at (0,6.4) [text width=3.8cm,align=center]  {\scriptsize{General, $\eta=1/2$}};
    \node at (0,5.5) [text width=3.8cm,align=center]  {$d^{O(\log(1/\eps))} 2^{\poly(1/\eps)} $};
    \node at (0,5) [text width=3.8cm,align=center]  {\scriptsize{\Cref{thm:intro-high-noise}} };
\draw (-8,0) parabola bend (0,9) (8,0);
    \node at (0,8.2) {\scriptsize{General, $\eta > 1/2$}};
    \node at (0,7.6) [text width=3.8cm,align=center]  {$d^{\Theta(1/\eps^2)}$};
    
    \end{tikzpicture}

    \caption{Overview of the (SQ) complexity of learning halfspaces with Massart noise.
    (1) For homogeneous halfspaces and $\eta < 1/2$, efficient algorithms were previously known
    under Gaussian or isotropic log-concave marginals. (2) For $\eta>1/2$, the problem is equivalent 
    to agnostic learning, for which tight upper and lower bounds were previously known.
    (3) The remaining regimes (general halfspaces and/or $\eta = 1/2$) are characterized in the current paper.}
    \label{fig:classes}
\end{figure}
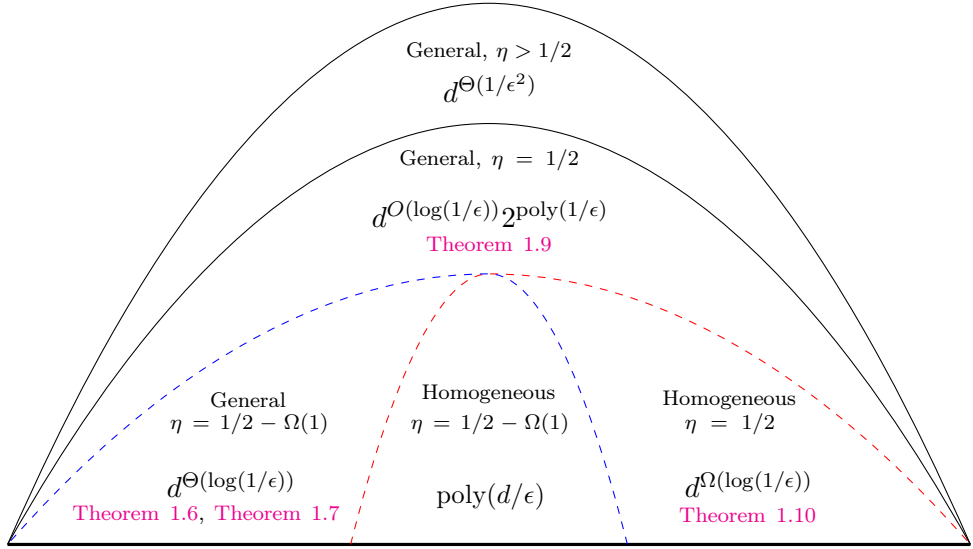

\subsection{Brief Overview of Techniques} \label{ssec:tec-short}

\paragraph{Leveraging Low-Degree Moments} 
The unifying theme of both our upper and lower bound techniques is the use of low-degree moments.
For our purposes, the low-degree moments of a distribution correspond to the values of
$\E_{(\x,y)\sim \D}[p(\x)y]$ for low-degree polynomials $p$. It is clear that one can compute
approximations of \new{the (up to)} $k$-degree moments of such a distribution \new{with} $\poly(d^k/\eps)$ samples
and time.  At a high-level, we are looking for a moment that will provide us with information about
the correlation between $\x$ and $y$. In more detail, we would like to find polynomials $p$ such
that $\E_{(\x,y)\sim \D}[p(\x)y] \neq \E_{\x\sim \D_\x}[p(\x)]\E_{y\sim \D_y}[y]$, or equivalently
find polynomials $p$ with $\E_{\x\sim \D_\x}[p(\x)] = 0$ and $\E_{(\x,y)\sim \D}[p(\x)y] \neq 0$. 
If no such polynomial exists, we can leverage this fact to prove SQ lower bounds; if such
polynomials exist, we can hope to \new{use them} as a starting point for an algorithm.

\paragraph{Learning with Constant-Bounded Massart Noise: \Cref{thm:intro-bounded-massart-upper-bound}} 
\new{At a high-level, our algorithm fits in the certificate-based framework developed in}
\cite{DKTZ20b,DKKTZ20,DKKTZ21b}. (A similar framework was developed independently in \cite{CKMY20}
to properly learn Massart halfspaces in the distribution-free setting, matching the error of
the~\cite{DGT19} algorithm.) The framework relies on the following fact: if the true halfspace is
given by $f(\x) = \sgn(\ell^\ast(\x))$, where $\ell^\ast(\x) = \vec w^\ast \cdot \x  - t^\ast$, then
for any affine function $\ell(\x)$ we have that $\E_{(\x,y)\sim \D}[\ell(\x) y T(\x)] \geq 0$ for
all non-negative functions $T(\x)$ if and only if $\ell = \ell^\ast$.  We can think of this as an
(infinite) linear program that can be used to solve for $\vec w^\ast$ and $t^\ast$. In order to
solve this program, we need a separation oracle. In particular, for any hypothesis $\vec w$ and $t$
that is too far from the truth, i.e., $\sgn(\ell(\x))$ has error greater than $\opt+ \eps$, we need
to be able to find an explicit non-negative function $T(\x)$ such that the above constraint is
violated. This essentially amounts to finding a \new{function $T(\x)$} that concentrates on the
values of $\x$ for which our current hypothesis $\ell(\x)$ is incorrect.

\new{
Our main new idea here is that for any sub-optimal halfspace guess $h$, there exists a certificate
function of the form $T(\x) = \1_{S}(\x) e^{\vec v \cdot \x}$ (\emph{exponential-shift certificate}),
where $S$ is a thin strip around the current (known) halfspace $h$ and $\vec v$ is an appropriate
weight vector.  However, finding a good weight vector $\vec v$ is a non-convex (hard) optimization
problem in general.  We simulate the behavior of the exponential-shift certificate by taking $T(\x)$
to be the square of a low-degree polynomial $q(\x)$ restricted to an appropriately chosen thin strip
$S$.  To construct our polynomial certificate, we essentially need to prove that for any at most
$(1-\gamma)$-biased one-dimensional threshold function $f$, there exists a polynomial $q$ of degree
$O(\log(1/\gamma))$ such that $q^2(\x)$ concentrates on the positive values of $f$; see
\Cref{lem:sos-ratio}. To prove the existence of such a polynomial, we establish that considering the
$O(\log(1/\gamma))$-degree Taylor approximation of the exponential-shift $e^{\vec v \cdot \x}$ suffices; see \Cref{lem:taylor-exponential-polynomial}. We can efficiently compute a
polynomial certificate since, once we have a fixed strip $S$, we can compute the low-degree moments
of $y$ restricted to $S$. From there, finding an appropriate polynomial $q$ amounts to finding a
negative eigenvalue of a quadratic form. For more details, we refer the reader to
\Cref{ssec:general-bounded-overview} and \Cref{sec:constant-bounded-arbitrary}.
}

\paragraph{Learning with \emph{General} Massart Noise: \Cref{thm:intro-high-noise}} 
\new{
In the case of learning with \emph{general} Massart noise ($\eta = 1/2$), the certificate approach 
fails, as it requires polynomials of large degree --- polynomial in $1/\eps$ --- to get concentration in
the disagreement region. This is prohibitively large as it would result in a runtime of
$d^{\poly(1/\eps)}$, which we can readily get by simply running the agnostic learner of
\cite{KKMS:08}.  
The main idea to obtain an improved bound is to bypass the limitations of the certificate approach
by relying on \emph{correlation} properties rather than \emph{concentration}.  As we show, there is
always some polynomial $p$ of low degree, logarithmic in $1/\eps$, that correlates with the label
$y$, i.e., $\E_{(\x,y)\sim \D}[p(\x)y] > 0$. Our algorithm exploits this weaker property and turns it
into an iterative process that constantly improves the current guess.  

Since low-degree polynomials are able to get non-trivial correlation, if we
compute the low-order moment tensors of $y$, i.e., $\E_{(\x, y) \sim \D}[ \x^{\otimes k} y]$, we can use
them to compute a low-dimensional subspace $V$ onto which $\vec w^\ast$ has non-trivial
($\poly(\eps)$) projection; see
\Cref{prop:nontriavial_angle} and \Cref{lem:tensor-flattening}.  This means that picking a random
element $\vec v\in V$ will, with reasonable probability (say $1/3$), have non-trivial, i.e.,
$\poly(\eps)$, correlation with $\vec w^\ast$. 
Our algorithm improves an initial guess $\vec w$ of the optimal weight vector as follows: by
applying the above technique to a thin strip perpendicular to our current guess $\vec w$, we can --
with some non-trivial probability -- find a vector that correlates non-trivially with the projection
of $\vec w^\ast$ on the orthogonal complement of $\vec w$.  This, in turn, will allow us to compute
a new guess $\vec w'$ with a slightly better correlation with $\vec w^\ast$; see
\Cref{lem:corr-improv}.  Repeating this process $\poly(1/\eps)$ times will produce a vector with at
most $\opt+ \eps$ error; see \Cref{ssec:high-noise}.  Each iteration of the above algorithm requires
$d^{O(\log(1/\eps))}$ samples and time to compute the moments of order $O(\log(1/\eps))$. We then
need to run the algorithm many times to find a trial in which we get lucky for $\poly(1/\eps)$
rounds in a row. This latter operation increases the complexity by a $2^{\poly(1/\eps)}$ factor.

The crucial structural result that we exploit is that for any $\eps$-biased halfspace $f(\x)$ we can
construct a \emph{mean-zero} polynomial $p$, which is a function of $\vec w^\ast \cdot \x$, and 
matches the sign of $f(\x)$ everywhere, see \Cref{infpro:sign-matching-polynomial} and
\Cref{sub:sign-mathcing-polynomial}. 
We then show that, even when Massart noise with $\eta = 1/2$
is applied to $f$, this polynomial $p$ will achieve non-trivial correlation with $f$, i.e.,
$\E_{(\x,y)\sim \D}[p(\x)y]>0$.  This implies that some low-order moment of our distribution must
have a non-trivial component in the $\vec w^\ast$-direction; see \Cref{sub:chow-tensors}. Flattening
the moment tensors and performing SVD on the resulting matrices, we can efficiently construct a
subspace (spanned by the top eigenvectors) onto which $\vec w^\ast$ has non-trivial projection.
}

\paragraph{SQ Lower Bounds:~\Cref{thm:intro-homogeneous-sq-lb,thm:intro-biased-sq-lb}}
\new{The connection between low-degree moments and the correlation between $\x$ and $y$} is even
more evident when trying to establish SQ lower bounds. Using the framework introduced in
\cite{DKS17-sq}, we can show the following: if there exists a Massart halfspace whose degree-$k$
moments have $\x$ independent of $y$, then distinguishing a random rotation of this distribution
from one in which $\x$ is {\em truly} independent of $y$ is roughly $d^{\Omega(k)}$-hard in the SQ
model. Thus, proving SQ lower bounds amounts to finding Massart halfspace distributions that fool
low-order moments in this way. 

We note that the moment-matching and noise conditions amount to a system of linear inequalities in
terms of the noise function $\eta(\vec x)$.  To construct our hard examples, we use linear
programming (LP) duality to show the existence of solutions to this system.  It should be noted that
LP duality has been previously used to provide SQ lower bounds for agnostic learning of halfspaces,
see, e.g.,~\cite{DFTWW15, DKPZ21}.  In the current setting, however, we have to construct instances
that satisfy the much more restrictive (constant-bounded) Massart noise assumption. 

To achieve this, we show that it is possible to add constant-bounded Massart noise to a classifier
$f(\x)$ in order to cause it to fool polynomials of degree at most $k$ if and only if there is no
polynomial $p$ of degree at most $k$ with $\E_{\x\sim \D_\x}[p(\x)]=0$, such that the sign of
$p(\x)$ agrees everywhere with the optimal halfspace $f(\x)$.  In order to prove that there exists
no such low-degree, zero-mean, sign-matching polynomial, we establish a much more general result:
the sign of any low-degree, zero-mean polynomial $p(\x)$ cannot be too biased, and therefore $p(\x)$
cannot match the sign of any heavily-biased function; see \Cref{lem:infisibility-dual}.  In
particular, if $f$ is a halfspace with bias $(1-\gamma)$, this is true for all
zero-mean polynomials of degree $k$ up to (roughly) $\log(1/\gamma)$, implying the SQ lower bound of
\Cref{thm:intro-biased-sq-lb}.
\new{We remark that our main structural lemma (\Cref{lem:infisibility-dual}) can readily be used to
establish SQ lower bounds for Massart learning other geometric concept classes such as intersections of
two homogeneous halfspaces. }

\new{It remains to provide a proof overview of \Cref{thm:intro-homogeneous-sq-lb} for 
learning homogeneous halfspaces with {\em general} Massart noise ($\eta = 1/2$).
Interestingly, instead of adapting the methodology described in the preceding paragraphs,}
we prove \Cref{thm:intro-homogeneous-sq-lb} via a ``reduction'' to our SQ lower bound 
of \Cref{thm:intro-biased-sq-lb}, i.e., for learning (general) halfspaces with {\em constant-bounded} Massart noise. 
\new{At a high-level,} this is achieved by using the noise to ``wash out'' any information except when $\x$ lies in a thin strip 
not passing through the origin. Restricted to this strip, $f(\x)$ is now a non-homogeneous halfspace, and our
lower bound for general halfspaces applies.

\subsection{Organization}

\new{The structure of this paper is as follows:
In \Cref{sec:technical-overview}, we provide a detailed overview of our techniques.  In
\Cref{sec:prelims}, we introduce the required notation and preliminaries. In \Cref{sec:constant-bounded-arbitrary}, we give our algorithm for
learning general halfspaces with constant-bounded Massart noise, establishing
\Cref{thm:intro-bounded-massart-upper-bound}.
In \Cref{sec:benign}, we
give our algorithm for learning halfspaces under general Massart noise, establishing
\Cref{thm:intro-high-noise}.   Finally, in \Cref{sec:massart-biased-lower-bound}, we
prove our SQ lower bounds, \Cref{thm:intro-biased-sq-lb} and \Cref{thm:intro-homogeneous-sq-lb}.
}

\section{Detailed Technical Overview}
\label{sec:technical-overview}

\subsection{Learning General Halfspaces with Constant-Bounded Massart Noise:
\Cref{thm:intro-bounded-massart-upper-bound}}
\label[sub]{ssec:general-bounded-overview}

Our learning algorithm for this regime leverages the certificate framework developed in 
\cite{DKTZ20b,DKKTZ20,DKKTZ21b}.  This framework makes essential use of
the Massart noise condition, i.e., the fact that $ \eta(\x) \leq 1/2$ for every $\x \in \R^d$. 
Using this fact, we have the following characterization of the affine
function $\ell^\ast(\x) = \vec w^\ast \cdot \x - t^\ast$ defining the
target halfspace $f(\x) = \sgn(\ell^\ast(\x))$. For every non-negative
function $T(\x): \R^d \mapsto \R^+$, it holds that 
\[
\E_{(\x, y) \sim \D}[\ell^\ast(\x) y~T(\x)]
=
\E_{(\x, y) \sim \D}[\ell^\ast(\x) \sgn(\ell^\ast(\x)) (1- 2 \eta(\x))~T(\x)]
\geq 0 \,.
\]
On the other hand, when $\pr_{(\x, y) \sim \D}[\sgn(\ell(\x)) \neq y] \geq  \opt + \eps$, 
there exists a non-negative function $T(\x)$ such that $ \E_{(\x, y) \sim \D}[\ell(\x) y~T(\x)] < 0 $. 
We say that such a $T$ is a ``certifying'' function (or simply a certificate) for the guess $\ell(\x)$, 
because it proves that $\ell(\x)$ is not optimal.
\begin{definition}[Certificate Oracle]
	Let $\D$ be a distribution on $\R^d\times\{\pm 1\}$ with standard normal
	$\x$-marginal satisfying the Massart noise condition with respect to some
	halfspace.  Fix $\eps,\delta\in(0,1]$.  Given any affine function $\ell(\x)$
	such that $\pr_{(\x,y) \sim \D}[ \sgn(\ell(\x)) \neq y] \geq \opt + \eps$, a
	$\rho$-certificate oracle returns a non-negative function  $T:\R^d \mapsto \R_+$ 
	such that $\E_{(\x, y) \sim \D}[\ell(\x) y ~ T(\x)] \leq - \rho \|\ell(\x)\|_2$.
\end{definition}
In \cite{DKTZ20b,DKKTZ20,DKKTZ21b}, it was shown that under Massart label noise, 
the problem of efficiently learning an optimal halfspace can be reduced to the
problem of efficiently computing certifying functions.  At a high-level, the
certifying functions can be viewed either as separation oracles for a convex
program (that can be solved via the ellipsoid method) or as loss functions in an online
convex optimization problem (that can be solved via online gradient descent).  
In our setting, we need to adapt the proofs of \cite{DKTZ20b} slightly to work for 
general halfspaces. 
For completeness, we provide the details of this reduction in \Cref{ssec:online}.  
\begin{proposition}
	\label{pro:certificate-reduction}
	Let $\D$ be a distribution on $\R^d\times\{\pm 1\}$ whose $\x$-marginal is the standard normal. Assume that
	$\D$ satisfies the $\eta$-Massart noise condition with respect to some halfspace. Fix $\eps,\delta\in(0,1)$.
	Given a $\rho$-certificate oracle $\cal O$ with runtime $T_{\mathcal{O}}(\rho)$,
	there exists an algorithm that makes 
	$M = \poly(\frac{d}{\eps \rho})$ calls to $\mathcal O$, 
	draws $N=\poly(\frac{d}{\eps \rho})\log(1/\delta)$ samples from $\D$, 
	runs in $\poly(d,N, M) T_{\mathcal{O} }(\rho) $ time and computes a hypothesis $h$, 
	such that $\pr_{(\x,y)\sim\D}[h(\x)\neq y]\leq \opt+\eps$, with probability $1-\delta$.
\end{proposition}

Given \Cref{pro:certificate-reduction}, it remains to construct an efficient certificate oracle.
The first step is to restrict our search over some \emph{parametric class} of non-negative functions.

\paragraph{Non-Continuous Certificates}
When we do not restrict our search to continuous functions, we can use
certificates of the form $T(\x) = \1\{\sgn(\ell(\x)) \neq \sgn(\vec v \cdot \x -
b)\}$, for some $\vec v \in \R^d, b \in \R$.  It is not hard to see 
(\Cref{app:non-continuous-certificate}) that, for any $\eta \in [0,1/2]$,
taking $T(\x) = \1\{ \sgn(\ell(\x)) \neq f(\x)\}$, 
i.e., the indicator of  the disagreement region of $\sgn(\ell(\x))$ and $f(\x)$, 
we obtain that $ \E_{(\x, y) \sim \D}[\ell(\x) y~T(\x)] \leq - \Omega(\eps^2 ) \|\ell(\x)\|_2 $.
Observe that, when we consider non-continuous certificates, there exist
$\poly(\eps)$-certificates independent of both the noise level $\eta$ and the
bias of the halfspace $\gamma$.  However, finding such certificates is computationally hard in general.
In \cite{DKKTZ20}, the authors provided a polynomial-time oracle for
non-continuous certificates, i.e., subsets of the disagreement region, under two
crucial assumptions. Specifically, they assumed that (1) the target halfspace
is homogeneous and (2) the distribution $\D$ satisfies the Tsybakov noise
condition, which is significantly weaker than the $1/2$-Massart noise
assumption. In particular, the technique of \cite{DKKTZ20} to isolate a
subset of the disagreement region does not work when we relax either of the
above assumptions. In what follows, we describe an efficient
certificate oracle for general halfspaces 
under the assumption that $\eta = 1/2- \Omega(1)$, 
i.e., in the constant-bounded Massart noise regime.

\paragraph{The Exponential \-Shift Certificate}
In order to get a handle on the optimization problem of finding a certificate, the first step is to
consider smooth function classes. 
\new{
We will show how to construct a certificate function against the constant guess $-1$, when the true
halfspace $f(\x) = \sign(\vec w^\ast \cdot \x - t^\ast)$ is $(1-\gamma)$-biased: the probability of the negative region is very large $\pr_{\x \sim \normal}[f(\x) = -1] = 1-\gamma$; see
\Cref{fig:constant-hypotheses}.   Finding certifying functions against the constant guess $-1$ captures
many of the challenges of the general case.}  The idea is to find a continuous function $T(\x)$ that
puts more weight on the disagreement region of $f(\x)$ (colored blue in
\Cref{fig:constant-hypotheses}) than the agreement region (colored red in
\Cref{fig:constant-hypotheses}).  In particular, in order to find a certificate against the constant
guess $-1$, we want to find some function $T(\x)$ so that the following ratio is 
a sufficiently large constant (greater than $1/\beta$):
\begin{equation}
	\label{eq:ratio}
	\frac{
		\E_{\x \sim \normal}[T(\x) \1\{f(\x) = +1\}]
	}
	{
		\E_{\x \sim \normal}[T(\x) \1\{f(\x) = -1\}] 
	} 
\end{equation}
\new{
Our key idea is to use an exponential-shift certificate of the form $T(\x) = e^{\vec v \cdot \x}$.  By
multiplying the Gaussian density with the exponential function $T(\x)$ we essentially shift the mean
of the Gaussian from $\vec 0$ to $\vec v$.  Setting $\vec v$ to be a large multiple of $\vec
w^\ast$, we can re-center the Gaussian of the ratio of \Cref{eq:ratio} to lie well within the
disagreement (blue) region; see \Cref{fig:gaussian-shift}.  In order to make the ratio of
\Cref{eq:ratio} sufficiently large, it suffices to set $\vec v = c~\vec w^\ast$ for $c =
\Theta(\sqrt{\log(1/\gamma) } )$; see also the proof of \Cref{lem:taylor-exponential-polynomial}.  
}

\def\FunctionF(#1){(-1 + 18 *(1 + 2/27 *(-3 + #1))^2 - 48 *(1 + 2/27 *(-3 + #1))^4 +
	32 *(1 + 2/27 *(-3 + #1))^6)^2-0.5}

\begin{figure}
	\centering
	\subfloat[]{\begin{tikzpicture}[scale=1]
			\label{fig:constant-hypotheses}
			\coordinate (start) at (0.5,0);
			\coordinate (center) at (0,0);
			\coordinate (end) at (0.5,0.5);
			\draw[->] (-1.5,0) -- (4.5,0) node[anchor=north west,black] {};
			\draw[->] (0,-1) -- (0,3) node[anchor=south east] {};
			\draw[fill=black] (2.3,0) circle (0.03) node[above right] {};
			\draw[fill=blue,opacity=0.2,draw=none] (2.3,-1)--(2.3,3)--(4.5,3)--(4.5,-1);
			\draw[fill=red,opacity=0.2,draw=none] (2.3,-1.)--(2.3,3)--(-1.5,3)--(-1.5,-1);
			\draw[black,thick] (2.3,-1) -- (2.3,3);
			\draw[->] (0,-1) -- (0,3) node[right=2mm, below] {$\x_{2}$};
			\draw[->] (-1.5,0) -- (4.5,0) node[right=2mm, below] {$\x_{1}$};
			\draw[->] (2.3,1.3) -- (2.7,1.3) node[right=2mm,below] {$\vec w^\ast$};
			\draw (2.3,0) node[below right] {$t^\ast$ };
			\draw[fill=black] (3.4,0) circle (0.03) node[above right] {};
			\draw (3.2,0) node[below right] {$c$ };
			\draw (0,0) node[below left] {$0$ };
\end{tikzpicture}}
	\subfloat[]{\begin{tikzpicture}\label{fig:gaussian-shift}
			\begin{axis}[ restrict y to domain=-4:4,xlabel=$\x_1$,ylabel=, legend pos=north west,axis y line =middle,
				axis x line =middle,
				axis on top=true,
				xmin=-1,
				xmax=4,
				ymin=-0.35,
				ymax=1,
				height=2.2in,
				width=3in,
xticklabels={,,,,$t^\ast$,$c$},
				yticklabels={},
				]
				\addplot[color=red, samples=1000, domain=-4:4,dashed]
				{e^(-8*x^2)/1.5};
				\addplot[color=black, samples=1000, domain=-4:4]
				{e^(-8*(x-3)^2)/1.5};
				\addplot[color=blue, samples=10, domain=2:4]
				{0.2};
				\addplot[color=red, samples=10, domain=-2:2]
				{-0.2};
				\draw[->] (axis cs:1,0.5)--(axis cs:2,0.5);
				\node at (axis cs:1.65,0.6){$ e^{c\x_1}$};
				\node at (axis cs:-0.2,-0.1){$ 0$};
\end{axis}
	\end{tikzpicture}}
	\caption{
		\textbf{(a)} An instance where $f(\x) = \vec w^\ast \cdot \x - t^\ast$ is very biased and
		the constant hypothesis $\ell(\x) = -1$ agrees with $f(\x)$ ``almost everywhere'', i.e., in
		the red region.  A non-negative certifying function against the hypothesis $-1$ must put
		significantly more weight to the disagreement region (colored in blue).  \new{Notice that by
		Gaussian concentration we have that $t^\ast = O(\sqrt{ \log(1/\gamma) } )$.}
		\\
		\textbf{(b)} The certifying function $T(\x) = e^{ c ~ (\vec w^\ast \cdot \x)}$ essentially
		moves the mean of the Gaussian from $\vec 0$ to $c \vec w^\ast$.  Choosing $c =
		\Theta(\sqrt{\log(1/\gamma}) )$ implies that the mass of the blue region with respect to the
		shifted normal $\normal(c \vec w^\ast, \vec I)$ will be much larger than the mass of the red
		region, making \Cref{eq:ratio} true.
	}
	\label{fig:constant-hypotheses-caption}
\end{figure}
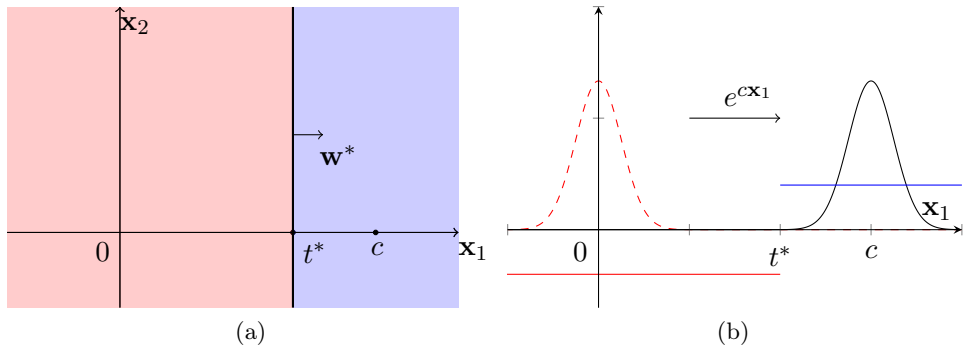

\paragraph{The Polynomial-Shift Certificate}
Even though we have shown that there exists an exponential-shift certificate, it
is still not easy to compute: minimizing $\E_{(\x, y) \sim \D}[\ell(\x) y e^{\vec v \cdot \x}]$ 
over $\vec v\in \R^d$ is a non-convex objective that, in general, is hard to optimize.
In order to circumvent this obstacle, we consider polynomial certificates 
of the form $q^2(\x)$ for some low-degree polynomial $q$. 
Such certificates were also used in \cite{DKTZ20b}
in order to learn homogeneous halfspaces with Tsybakov noise.  
The advantage of having $T(\x)= q^2(\x)$, for some low-degree polynomial $q$, 
is that we can efficiently compute certificates via semi-definite programming (SDP); see
\Cref{sub:certificate_optimization}.  

As with the exponential-shift certificate, we can provide a high-level overview of some of the ideas
involved by trying to construct a certificate against the constant hypothesis $-1$, when the optimal
halfspace is almost everywhere equal to $-1$, i.e., $\pr_{\x \sim \normal}[f(\x) = -1] = 1-\gamma$.
Our goal in this case is to find a polynomial that assigns much more weight to the positive region
(colored blue in \Cref{fig:constant-hypotheses}) than the negative region (colored red in
\Cref{fig:constant-hypotheses}).  The main structural result that enables our algorithm is the
following lemma, where we show that we can achieve the same effect of the exponential-shift
certificate discussed above using a low-degree polynomial
(see~\Cref{lem:taylor-exponential-polynomial} for a more detailed statement and proof).
\begin{lemma}[Low-Degree Polynomial Shift]
	\label{lem:sos-ratio}
	Let $f(\x)$ be an at most $(1-\gamma)$-biased halfspace. 
	There exists an absolute constant $C \geq 1$ such that
	\[
	e^{k/C - C \log(1/\gamma)}
	\leq 
	\max_{p: \mathrm{deg}(p) \leq k}
	\frac{\E_{\x \sim \normal}[p^2(\x) \1\{f(\x) = +1\}]}
	{\E_{\x \sim \normal} [p^2(\x) \1\{f(\x) = -1\}]}
	\leq 
	e^{C~ k \log k - \log(1/\gamma)/C}
	\,.
	\]
\end{lemma}
The lower bound of \Cref{lem:sos-ratio} implies that there exists a polynomial certificate, namely a
polynomial that makes \Cref{eq:ratio} true, of degree $O(\log(1/\gamma) )$.  The upper bound
implies that degree $\Omega(\log(1/\gamma) )$ is essentially necessary. 

It is not hard to prove the upper bound using the anti-concentration of Gaussian polynomials; see
\Cref{app:upper-bound-polynomial}. To prove the lower bound of \Cref{lem:sos-ratio}, we construct a
low-degree polynomial using the Taylor approximation to the exponential-shift certificate that we
defined previously. For $c = \Theta(\sqrt{\log(1/ \gamma) } )$, we consider the \emph{square
of the Taylor expansion} $S_k(c x)$ of the function $e^{c x}$. Using the Gaussian concentration and
the fast convergence rate of the Taylor polynomial of $e^x$, in
\Cref{clm:ell-one-approximation-taylor}, we show that for degree $k = \Theta(c^2)= \Theta(
\log(1/\gamma) )$, $S_k^2(c ~\vec w^\ast \cdot \x)$ is very close to $e^{2 c ~ (\vec w^\ast
\cdot \x)}$ in the $L_1$ sense.  Therefore, we can use $T(\x) = S_k^2(c~\vec w^\ast \cdot \x)$ in
the ratio of \Cref{eq:ratio} and obtain the same guarantees (up to constant factors) with the
exponential shift $e^{2 c ~ (\vec w^\ast \cdot \x)}$ that we discussed previously.

Our certificate against general (non-constant) hypotheses is the product of (the square of) a
polynomial and a band around the current guess $\ell(\x)$, i.e., $T(\x) = \1\{r_1 \leq \ell(\x) \leq
r_2\} q^2 (\x)$.  Notice that, since the current guess $\ell(\x)$ is known to the certificate
algorithm, in order to find a band of the form $\1\{r_1 \leq \ell(\x) \leq r_2\}$, we simply need to
perform a brute-force search over the two thresholds $r_1, r_2$.  Our technical contribution here is
the following proposition, showing the existence of low-degree polynomial certificates for general
halfspaces with constant-bounded Massart noise (see \Cref{ssec:low-degree-sos-certificate} for a
more detailed statement and proof). 

\begin{proposition}[Polynomial Certificate for Halfspaces with Constant-Bounded Massart Noise]
	\label{pro:technical-overview-certificate-Biased-Massart}
	Let $\D$ be a distribution on $\R^d \times \{\pm 1\}$ whose $\x$-marginal is the standard
	normal.  Assume that $\D$ satisfies the constant-bounded Massart noise condition with respect to
	some at most $(1-\gamma)$-biased target halfspace $f(\x)$.  Let $\ell(\x)$ be  any linear
	function such that $\pr_{(\x, y) \sim \D}[\sgn(\ell(\x)) \neq y ] \geq \opt + \eps$.  There
	exist $r_1, r_2 \in \R$ and polynomial $q(\x)$ of degree $k=\Theta(\log(1/\gamma))$
	with $\|q(\x)\|_2 = 1$ such that
	\[
	\E_{(\x, y) \sim \D} \left[ \ell(\x) y \ \1\{r_1 \leq \ell(\x) \leq r_2\} q^2(\x) \right]
	\leq -  \eps^2  \poly(\gamma)  ~ \| \ell(\x) \|_2 
	\,.
	\]
\end{proposition}

\subsection{Learning Halfspaces with General Massart Noise: \Cref{thm:intro-high-noise}}
\label{ssec:high-noise}
For the purpose of this description, we will assume that the target halfspace $f(\x) = \sgn(\vec w^\ast \cdot \x)$ 
is homogeneous. This special case captures the key ideas of our algorithm;
given such a result, the generalization to general halfspaces (under general Massart noise) is fairly straightforward. 
(This generalization is carried out in \Cref{app:general_benign}.)
To facilitate the intuition, we present the main ideas behind the algorithm of 
\Cref{thm:intro-high-noise} in the following paragraphs. Our main technical contribution in this context 
is the construction of a low-degree sign-matching polynomial. 
This is presented at the end of this subsection and in full detail in \Cref{sub:sign-mathcing-polynomial}.

At a high-level, our algorithm for learning halfspaces in the $\eta = 1/2$ regime
consists of a random walk on the unit $d$-dimensional sphere, where 
we iteratively perform a random step in order to update our current guess $\vec w^{(t)}$, i.e., 
\begin{equation}
	\label{eq:projeected-update-rule}
	\vec w^{(t+1)} \gets \frac{\vec w^{(t)} + \lambda \vec v}
	{\| \vec w^{(t)} + \lambda \vec v\|_2} \,.
\end{equation}
Our goal is to find a way to sample the update vector $\vec v$, so that
at every step there is some non-trivial probability
that we make progress towards the optimal direction $\vec w^\ast$.
In order to make progress towards $\vec w^\ast$, 
it suffices to have an update vector $\vec v$ that correlates with $\vec w^\ast$ 
and belongs in the orthogonal complement, $\vec w^\perp$, of $\vec w$.  
In what follows, we will denote by $\wperp$ the normalized projection of $\vec w^\ast$
onto the subspace $\vec w^\perp$; see \Cref{fig:orth_proj}.
Given such a vector $\vec v$, we can show that there exists a step size $\lambda$
such that the update rule of \Cref{eq:projeected-update-rule}
moves $\vec w^{(t)}$ closer to $\vec w^\ast$ by a non-trivial amount with constant probability; see \Cref{lem:corr-improv}.
Observe that a uniformly random unit direction in $\R^d$ has roughly $1/\sqrt{d}$ correlation 
with $\wperp$ with constant probability.
However, in order to hit $\vec w^\ast$, we need to perform roughly $d$ 
consecutive successful updates (see \Cref{lem:corr-improv}),
resulting in an algorithm with $2^{d^{\Omega(1)} }$ runtime.  

The main algorithmic result of this section is the following proposition, 
which shows that we can efficiently sample an update vector $\vec v$
that improves the current guess with non-trivial probability.
\begin{proposition}[Correlated Update Oracle] \label{prop:tech-overview-warm-start}
	Let $\D$ be a distribution on $\R^{d} \times \{\pm 1\}$, with standard
	normal $\x$-marginal, that satisfies the Massart noise condition for $\eta =1/2$,
	 with respect to a target halfspace $f(\x) = \sgn(\vec w^\ast \cdot \x)$.  
	 Let $\vec w \in \R^d$ be a unit vector such that $\pr_{(\x, y) \sim \D}[\sign(\vec w \cdot \x) \neq y] \geq \opt + \eps$, 
	 for some $\eps \in (0,1]$.  There exists an algorithm that draws
	 $N=d^{O(\log(1/\eps))}\log(1/\delta)$ samples from $\D$, runs in time
	 $\poly(N,d)$, and with probability at least $1-\delta$ returns a
	 distribution $\mathcal V$ on $\R^d$ such that 
	\[ 
	\pr_{\vec v \sim \mathcal V}
	\left[\wperp \cdot \vec v \geq \poly(\eps) \right] \geq \frac{1}{3}\,.
	\]
	Moreover, $\mathcal V$ has description size $\poly(d/\eps)$ and can be sampled in $\poly(d/\eps)$ time.
\end{proposition}

(See \Cref{prop:warm-start} for a more detailed statement.) 
Our plan is to construct a subspace $V$ of $\R^d$ such that
$\|\proj_V(\wperp)\|_2 \geq \poly(\eps)$. To sample good update 
vectors $\vec v$, as claimed in \Cref{prop:tech-overview-warm-start}, 
we can generate a random vector $\vec v$ on the unit sphere of $V$.  
However, we need to make sure that the dimension of $V$ is sufficiently small, 
namely at most $\poly(1/\eps)$.  

\paragraph{Improving the Constant Guess}
Let us assume for now that our current guess is $\vec w = \vec 0$.
Then, in order to make progress towards $\vec w^\ast$, one can simply
use the degree-one Chow parameters of $y$, i.e., $\E_{(\x, y) \sim \D}[y \vec x]$.
Observe that the degree-one Chow parameters have positive correlation
with the optimal direction $\vec w^\ast$, since 
\(
(\E_{(\x, y) \sim \D}[y \vec x ]) \cdot \vec w^\ast 
=
\E_{(\x, y) \sim \D}[(1-2 \eta(\x)) \sgn(\vec w^\ast \cdot \vec x)~(\vec x \cdot \vec w^\ast)] 
> 0\,.
\)
Therefore, the degree-one Chow parameters are a good first update to the guess 
$\vec w = \vec 0$.  

\paragraph{Projecting onto $\vec w^\perp$: \Cref{sub:projection}}
In order to further improve a non-trivial guess $\vec w$, 
we need to find a good update direction $\vec v$ that correlates non-trivially 
with $\wperp$. A natural attempt to do so would be to project $\x$ onto the orthogonal complement of 
the current guess, i.e., $\vec w^\perp$, and then compute the Chow parameters of the projected points.
However, by doing so, the optimal classifier of the projected examples will no longer be a halfspace;
see \Cref{fig:orth_proj}.
In particular, the noise function $\eta^\perp(\x^\perp)$ after projection
will be larger than $1/2$ for a large fraction of the points.
To make the dataset  nearly separable by a halfspace with normal vector 
$\wperp$, we condition on a thin band and then project $\x$ onto $\vec w^\perp$.  
Doing so, apart from a small region close to the optimal classifier, $f^\perp$ (see \Cref{fig:orth_proj}), 
the noise function will be at most $1/2$.
By making the band sufficiently thin, we can control the probability of this ``high-noise''
area. Notice that in the projected instance the optimal halfspace 
is no longer homogeneous.  When the band (that we condition on) 
is far from the origin, the resulting optimal halfspace $f^\perp$ of the projected instance 
will be potentially very biased; see \Cref{fig:orth_proj}.

Assuming that our current halfspace $\sgn(\vec w \cdot \x)$ is at least
$\eps$ suboptimal compared to $\vec w^\ast$, we show that there exists a thin
band conditional on which the current hypothesis is roughly
$\eps/\sqrt{\log(1/\eps)}$-suboptimal.  Moreover, this band is not very far from
the origin, which implies that the optimal halfspace $f^\perp$ conditional on
the band will not be very biased: its threshold will be at most
$O(\sqrt{\log(1/\eps)})$.  
	It is worth noting that a similar orthogonal projection step onto $\vec w^\perp$ was
	used in \cite{DKKTZ20} to learn homogeneous halfspaces with Tsybakov noise. 
	The major difference between the setting of the current paper and \cite{DKKTZ20} 
	is that in the general Massart regime it is not possible to control the
	distance of the band from the origin. Specifically, it could be the case that 
	$\eta(\x) = 1/2$ for all $\x$ close to $\vec w \cdot \x = 0$, forcing us to pick a band 
	whose optimal halfspace $f^\perp$ in the subspace $\vec w^\perp$ actually 
	has threshold $\Omega(\sqrt{\log(1/\eps)})$, see \Cref{fig:orth_proj}.  
	In contrast, in \cite{DKKTZ20}, the ``soft'' Tsybakov noise condition allows
	for the band to be picked arbitrarily close to origin resulting in nearly
	homogeneous halfspaces $f^\perp$.
For the details of this projection step, see \Cref{lem:band-projection}. In
what follows, we denote the distribution of the projected instance over $\vec
w^\perp \times \{\pm 1\}$ by $\D^\perp$. We elaborate further on this step in \Cref{sub:projection}.

\begin{figure}[ht]
	\centering
	\subfloat[
	]{
		\label{fig:orth_proj}
		\begin{tikzpicture}
\coordinate (start) at (0.5,0);
			\coordinate (center) at (0,0);
			\coordinate (end) at (0.5,0.5);
			\draw (-2,1) node[left] {};
			\draw (-2,0.5) node[left] {};
			\draw (2,2) node[above] {};
			\draw (3,1.2) node[above] {};
			\draw[black, thick,black](-2,1) -- (3.75,1);
			\draw[black, thick,black](-2,0.5) -- (3.75,0.5);
			\draw[|<->|] (-2,0.5)--(-2,1) node[midway,left] {$B$};
			\draw[->] (-2,0) -- (3.8,0) node[anchor=north west,black] {};
			\draw[->] (0,-1) -- (0,2.5) node[anchor=south east] {};
			\draw[fill=blue, opacity=0.5,draw=none] (1,1) -- (0.5 ,0.5)--(3.75,0.5)--(3.75,1);
			\draw[fill=red, opacity=0.5,draw=none] (1,1) -- (0.5 ,0.5)--(-2,0.5)--(-2,1);
			\draw[fill=green, opacity=0.5,draw=none] (1,1) -- (0.5 ,0.5)--(0.5,1)--(1,1);
			\draw[thick,->] (0,0) -- (-0.7,0.7) node[anchor= south east,below,left=0.1mm] {$\scriptstyle \wstar$};
			\draw[black] (-1,-1) -- (2,2);
			\draw[thick ,->] (0,0) -- (0,1) node[left=2mm,below] {$\scriptstyle \bw$};
			\draw[thick ,->] (0.5,1.8) -- (0,1.8) node[left=1mm] {$\scriptstyle (\bw^{\ast})^{\bot_{\vec w}}$};
			\node (spy2) at     (0.5,0.5) {};
			\node (spy3) at     (1,1) {};
\draw[black](0.5,-1) -- (0.5,2.5);
\draw [-latex, thick, black] (2,2) to[out=-240, in=-180] (2.8,2.5);
			\draw (3.3,2.5) node[] {$f(\x)$};
			\draw [-latex, thick, black] (0.5,-1) to[out=180, in=0] (-1.1,-1.5);
			\draw (-1.8,-1.5) node[] {$f^\perp(\x^\perp)$};
			\draw [-latex, thick, black] (0.7,0.9) to[out=-90, in=100] (1.1,-1.75);
			\draw (0.9,-2.1) node[right] {$\eta^\perp(\x^\perp)\geq 1/2$};
	\end{tikzpicture}}
	\centering
	\subfloat[
	]{ \label{fig:polynomial}
		\centering
		\begin{tikzpicture}[scale=0.8]
			\begin{axis}[ restrict y to domain=-4:4,xlabel=$x$,ylabel=$y$, legend pos=north west,axis y line =middle,
				axis x line =middle,
				axis on top=true,
				xmin=-2,
				xmax=3,
				ymin=-2,
				ymax=2,
				yticklabel style = {font=\tiny,xshift=0.5ex},
				xticklabel style = {font=\tiny,yshift=0.5ex}
				]
				\addplot[color=red, samples=1000, domain=-4:4]
				{x^9 +(1/14)*x^6-15/14};
				\addplot[color=blue, samples=10, domain=1:3]
				{1};
				\addplot[color=blue, samples=10, domain=-2:1]
				{-1};
				\addplot[color=orange] table{\mytable};
				\node [left, orange] at (axis cs: 2.9,0.7) {\tiny{$L_2$-Approximation}};
				\node [left, red] at (axis cs: 2.6,1.7) {\tiny{Sign-Matching} };
			\end{axis}
	\end{tikzpicture}}
	\caption{
\textbf{(a)}
		After we condition on the band $B$, we project $\x$ to the subspace $\wperp$
		and nearly maintain the Massart noise property with respect to the biased halfspace 
		$f^\perp(\x^\perp) = \sgn(\wperp \cdot \x^\perp + b)$. 
		In particular, it holds that  $\eta^\perp(\x^\perp) \leq 1/2$ everywhere apart from a small area (green).  
		Since the underlying distribution is the standard Gaussian, a band with large negative mass
		(blue) cannot be very far from the origin, and therefore we have that $|b| = O(\sqrt{\log(1/\eps)})$.
		\\
\textbf{(b)} 
		The sign-matching polynomial that corresponds to the red curve
		does not need to closely approximate the threshold function. Its degree scales as
		$\Theta(b^2)$, which is at most $O(\log(1/\eps))$ for an $\eps$-biased halfspace.
		On the other hand, to get an $L_2$ or $L_1$ approximation to error $\eps$ (orange curve), 
		it is known that $\poly(1/\eps)$ degree is necessary.
	}
\end{figure}
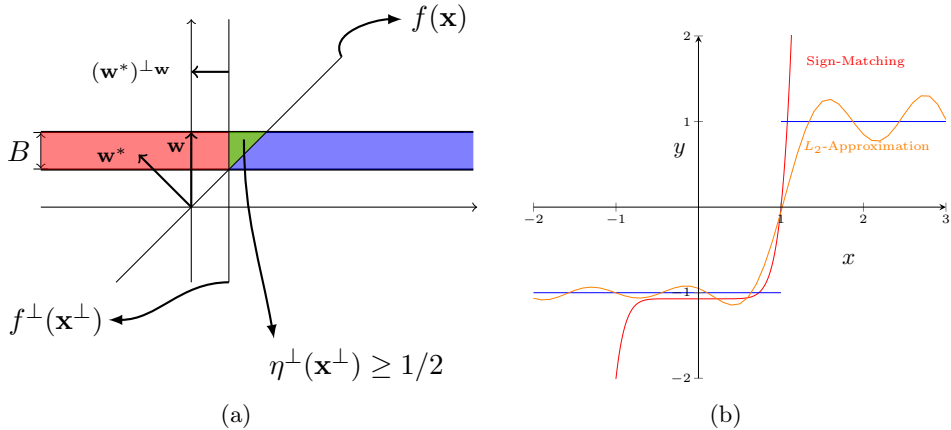

\paragraph{Using the Low-Order Chow Tensors: \Cref{sub:chow-tensors}}
To obtain a good update vector $\vec v$, a natural approach is
to use the degree-one Chow parameters.  However, since the optimal halfspace 
$f^\perp$ is biased, the degree-one Chow parameters 
are not guaranteed to correlate well with the direction of $f^\perp$.
We thus need to look at higher-order Chow parameter tensors of $\D^\perp$.
Recall that, to show \Cref{prop:tech-overview-warm-start}, we want to construct 
a subspace $V$ of $\R^d$ such that $\|\proj_V(\wperp)\|_2 \geq \poly(\eps)$.  
We now show how to find such a subspace using the low-order Chow tensors of $\D^\perp$.

Since we are working in Gaussian space, instead of considering the
moment-tensor $\vec x^{\otimes m}$, we use the degree-$m$ Hermite moment
tensor $\vec H^m(\x)$; this corresponds to replacing all the monomials 
of the tensor $\vec x^{\otimes m}$ by their corresponding Hermite monomials.
(For tensor notation, we refer to \Cref{sec:prelims}.)

\begin{definition}[Hermite Moment-Tensor]
	\label{def:hemite-moment-tensor}
	Let $\mathcal H$ be the linear operator that maps any $d$-variate monomial
	to the corresponding (normalized) $d$-variate polynomial in the Hermite basis, i.e.,
	$\mathcal H( \vec x^\alpha) = h_\alpha(\x)$.
	We define the degree-$m$ Hermite moment tensor as
	\(
	(\vec H^m)_{\alpha} = \mathcal H( (\vec x^{\otimes m} )_\alpha ) \, .
	\)
\end{definition}
Using the Hermite moment-tensors, we also define the order-$m$ Chow parameter
tensors of a distribution $\D$ on $\R^d \times \{\pm 1\}$.
\begin{definition}[Order-$m$ Chow Tensor of $\D$]
	\label{def:chow-parameters}
	Let $\D$ be a distribution on $\R^d \times \{ \pm 1 \}$ whose 
	$\x$-marginal is the standard normal distribution.  We define 
	the order-$m$ Chow tensor of $\D$ to be 
	\[
	\vec T^{m}(\D) = \E_{{(\x, y) \sim \D}}[ \vec H^m(\x) y ] \,.
	\]
	When it is clear from the context, we shall omit the distribution $\D$ 
	and simply write $\vec T^{m}$.
\end{definition}

Our high-level plan is to treat the above tensors as $d \times d^{m-1}$ matrices
and perform SVD to find the top few left singular vectors, i.e., the singular
vectors whose singular values are larger than some threshold.
To show that the subspace $V$ spanned by such eigenvectors 
contains a non-trivial part of $\wperp$, we need to construct a \emph{mean-zero polynomial} 
that only depends on the direction $\wperp$ and correlates well with the label $y$ (\Cref{lem:tensor-flattening}).  
This leads us to our main structural result.

\paragraph{The Sign-Matching Polynomial: \Cref{sub:sign-mathcing-polynomial}}
We show that there exists a mean-zero
polynomial $p$ that achieves non-trivial correlation with $\wperp$, i.e., 
$\E_{(\x^\perp, y) \sim \D^\perp}[p(\x) y] \geq \poly(\eps)$.  
Even though the noise of $\D^\perp$ is not exactly Massart -- 
recall that there exists a small region where $\eta^\perp(\x^\perp) > 1/2$ --
let us assume that $\eta^\perp(\x^\perp) \leq 1/2$ everywhere for simplicity.
In this case, we have 
\[
\E_{(\x^\perp, y) \sim \D^\perp}[ p(\x) y] 
=  
\E_{(\x^\perp, y) \sim \D_\x^\perp} \left[ p(\x) \sgn(\wperp \cdot \x^\perp + b) (1 - 2 \eta^{\perp}(\x^\perp)) \right] \,.
\]
\emph{
	Since $ 1 - 2 \eta^{\perp}(\x^\perp) \geq 0$, in order to achieve non-trivial correlation
	it suffices to find $p$ such that $p(\x)$ matches the sign of $f^\perp(\x^\perp)$.}  

	At this point, we made crucial use of the Massart noise condition; in particular, this is
	not possible in the agnostic model. In the agnostic model, to achieve
	non-trivial correlation, one needs to actually approximate the threshold
	function; see \Cref{fig:polynomial}.  More specifically, to guarantee positive correlation 
	in the agnostic model, we need a polynomial whose $L_1$ error with $f(\x)$ is $O(\eps)$. 
	Unfortunately, this cannot be done for any polynomial with degree $o(1/\eps^2)$; 
	see, e.g., Proposition 2.1 of \cite{DKPZ21}.  In the Massart noise setting, 
	we show that we can construct a \emph{zero-mean sign-matching} polynomial 
	of degree only $\log(1/\eps)$ that achieves $\poly(\eps)$ correlation with $y$. 
	We remark that the mean-zero condition, $\E_{\x \sim \D_\x}[p(\vec x)] = 0$ is crucial here.
	Non-zero mean polynomials, like the linear polynomial $\wperp \cdot \x^\perp + b$, 
	might give constant correlation, but do not reveal any information
	about the optimal direction. More concretely, we establish the following proposition; see
	\Cref{lem:correlation-polynomial} and \Cref{clm:zero-mean-polynomial} for
	the corresponding formal statements.

\begin{infproposition}[Sign-Matching Polynomial] \label{infpro:sign-matching-polynomial}
Let $b \in \R$.  There exists a zero mean, unit variance polynomial $p:\R \mapsto \R$ 
of degree $k = \Theta(b^2 + 1)$ such that the sign of $p$ matches 
the sign of the threshold function $\sign(z-b)$, i.e., $\sign(p(z)) = \sign(z - b)$, for all $z \in \R$.
\end{infproposition}

Notice that when we apply the above proposition, the threshold $b$ of the
corresponding halfspace $f^\perp$ will be at most $O(\sqrt{\log(1/\eps)})$
resulting in a polynomial of degree $O(\log(1/\eps))$; see also
\Cref{fig:orth_proj}. In the case of homogeneous halfspaces, this happens
because we use the orthogonal projection step; see \Cref{sub:projection}.  In
the case of general halfspaces, we may have such thresholds to start with.

\subsection{SQ Lower Bounds: \Cref{thm:intro-biased-sq-lb,thm:intro-homogeneous-sq-lb}}
\paragraph{Learning General Halfspaces with Constant-Bounded Massart Noise:
\Cref{thm:intro-biased-sq-lb}}
Our SQ lower bounds make essential use of the ``hidden-direction'' framework developed in
\cite{DKS17-sq}. Using this framework, we construct SQ lower bounds for
learning halfspaces in high dimensions using a carefully constructed one-dimensional Massart noise
instance. In particular, if we can construct a one-dimensional Massart noise
distribution $\D$ on $\R \times \{ \pm 1 \}$ such that 
$\E_{(z,y) \sim \D}[z^k y] = \E_{z \sim \D_z}[z^i] \E_{y \sim \D_y}[y]$, for all $i\in[k]$, 
then we can obtain a family of $2^{d^{\Omega(1)}}$ distributions on 
$\R^d \times \{ \pm 1\}$ whose pairwise correlation is $d^{-\Omega(k)}$. 
Using standard SQ lower bound arguments (see \Cref{lem:sq-from-pairwise}), 
having such a family of pairwise correlated distributions implies a $d^{\Omega(k)}$ 
SQ lower bound for learning the halfspace. (For a brief review on SQ lower bound machinery, 
we refer the reader to \Cref{ssec:SQ-prelims}.) 

Our main technical result is the following proposition (also see \Cref{pro:construction}).
\begin{proposition} \label{pro:tech-overview-sq-construction}
	Fix $\eta \in (0, 1/2)$ such that $\eta = 1/2 - \Omega(1)$ and 
	$\gamma \in (0,1/2)$.
	There exists a distribution $\D$
	on $(z,y) \in \R\times\{\pm 1\}$ whose $z$-marginal is the standard normal
	distribution with the following properties.
	\begin{itemize}
		\item  $\D$ satisfies the $\eta$-Massart noise condition
		with respect to a halfspace $f(z)$ with $\pr_{z\sim \D_z}[f(z)=+1]=\gamma$.
		
		\item  There exists an absolute constant $C$ such that for any integer 
		$k \leq C \log(1/\gamma)$, it holds 
		\(
		\E_{(z, y) \sim \D}[y z^k] =  \E_{y \sim \D_y}[y]  \E_{z \sim \D_z}[z^k]  \,.
		\)
	\end{itemize}
\end{proposition}

To prove the existence of the (constant-bounded) Massart noise instance of
\Cref{pro:tech-overview-sq-construction}, we use (infinite-dimensional) LP duality,
where our variable is the signal function $\beta(z) = 1- 2 \eta(z)$.
We have the following pair of primal and dual linear programs.
We denote $\mathcal{P}_k^0$ the linear space of zero-mean polynomials
of degree at most $k$.

\vspace{4mm}

\begin{tabular}{l|l} 
	\qquad \quad \quad \text{{\bf Primal}} &\qquad \quad \quad  \text{{\bf Dual}}\\
	\parbox{0.3\linewidth}{\begin{align*} 
		&\text{Find }    ~~~ \qquad  \beta(z) \in L^\infty(\R) 
		\\ 
		&\text{such that\ } ~  \E_{z \sim \normal}[f(z)p(z)\beta(z)]  = 0   
		 ~ \forall p \in{\cal P}_{k}^0  \\
		 & \qquad\qquad~~ \pr_{z \sim \normal}[\beta \leq \beta(z) \leq 1] = 1  
	\end{align*} } \quad \quad ~&~ \quad \parbox{0.3\linewidth}{\begin{align*} 
		 &\text{Find }  \qquad ~~   p(z) \in{\cal P}_{k}^0    \\ 
	   & \text{such that} ~ ~\beta\E_{z\sim \normal}[(f(z)p(z))^+]  > \E_{z\sim \normal}[(f(z)p(z))^-] \\
	\end{align*}}
	\end{tabular}
\vspace{3mm}

Using an infinite-dimensional variant of the theorem of the alternative for 
linear programming, to show that the primal problem is
feasible, it suffices to show that the dual is infeasible. To show that the
dual is infeasible, we prove a stronger statement: mean-zero polynomials of
low-degree cannot match the sign of very biased Boolean functions.  We prove the
following (see \Cref{lem:infisibility-dual} for the formal statement and proof).

\begin{lemma}\label{lem:tech-overview-infisibility-dual}
	Let $f: \R \mapsto \{\pm 1\}$ be any one-dimensional Boolean function,
	$\beta\in(0,1)$ and $k\in \Z_+$. There exists a universal constant $C>0$ such that 
	if $\pr_{z\sim \normal}[f(z)=1]\leq 2^{-Ck}(1-\beta)$, then for any mean-zero polynomial of degree at 
	most $k$ it holds 
	$\beta\E_{z\sim \normal}[(f(z)p(z))^+]  < \E_{z\sim \normal}[(f(z)p(z))^-]$.
\end{lemma}
We remark that in the above lemma we do not require $f$ to be a halfspace. 
Consequently, it is easy to adapt our argument to work for other Boolean concept
classes that depend on a low-dimensional subspace $V$. For example, using the
above lemma, we can obtain an SQ lower bound for learning intersections of 
$2$ homogeneous halfspaces with constant-bounded Massart noise.

\paragraph{Learning Homogeneous Halfspaces with General Massart Noise: \Cref{thm:intro-homogeneous-sq-lb}}
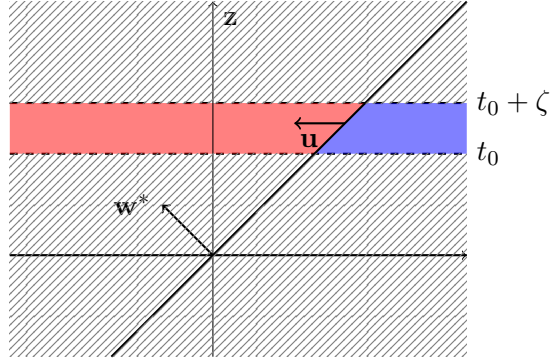
\begin{figure}[h]
	\captionsetup{width=.85\linewidth}
	\centering
	\begin{tikzpicture}[scale=1.35]
		\coordinate (start) at (0.5,0);
		\coordinate (center) at (0,0);
		\coordinate (end) at (0.5,0.5);
		\draw (-2,1) node[left] {};
		\draw (-2,0.5) node[left] {};
		\draw (2,2) node[above] {};
		\draw (3,1.2) node[above] {};
		\draw[black,dashed, thick](-2,1) -- (2.5,1);
		\draw[black,dashed, thick](-2,1.5) -- (2.5,1.5);
		\draw[fill=blue, opacity=0.5,draw=none] (1,1) -- (1.5 ,1.5)--(2.5,1.5)--(2.5,1);
		\draw[fill=red, opacity=0.5,draw=none] (1,1) -- (1.5 ,1.5)--(-2,1.5)--(-2,1);
		\draw[->] (-2,0) -- (2.5,0) node[anchor=north west,black,below left] {};
		\draw[->] (0,-1) -- (0,2.5) node[anchor=south east,below right] {$ \vec z$};
		\draw[thick,->] (0,0) -- (-0.5,0.5) node[anchor= south east,below,left=0.1mm] {$\wstar$};
		\draw[black,thick] (-1,-1) -- (2.5,2.5);
		\draw[black,thick] (-2,0) -- (2.5,0);
\draw[thick ,->] (1.3,1.3) -- (0.8,1.3) node[right=2mm,below] {$\vec u$};
		\draw (2.5,1.5) node[right] {$t_0 + \zeta$ };
		\draw (2.5,1) node[right] {$t_0$};
		\path[pattern=north east lines, pattern color=gray,opacity=1] (-2,1)--(2.5,1)  --
		(2.5,-1)--(-2,-1);
		\path[pattern=north east lines, pattern color=gray,opacity=1] (-2,1.5)--(2.5,1.5)  --
		(2.5,2.5)--(-2,2.5);
	\end{tikzpicture}
	\caption{
		The ``high-noise'' ($\eta = 1/2$) Massart distribution $\D$ on $((\x, z), y) \in \R^{d+1} \times \{\pm 1\}$ that we construct in our 
		``reduction'' from learning a biased halfspace with Massart noise.  
		The $\x$-marginal of $\D$ is the standard normal distribution.
		Conditional on $z$, $(x, y)$ has constant-bounded Massart noise when 
		$z \in [t_0, t_0 + \zeta]$; this corresponds to the blue/red area.
		When $z \notin [t_0, t_0 + \zeta]$ we set $y$ to be $\pm 1$ with probability
		$1/2$ independently of $\x$, i.e., $\eta((\x, z)) = 1/2$ in the gray area.
	}
	\label{fig:hardband}
\end{figure}
Our main idea to show an SQ lower bound for learning homogeneous halfspaces with general Massart noise
 is to use the fact that by setting $\eta(\x) = 1/2$, we
essentially remove all the useful signal from some parts of the space.  Our
``reduction'' to learning halfspaces with constant-bounded Massart noise works as
follows: We create a $(d+1)$-dimensional high-noise instance, i.e., with $\eta =1/2$, 
over $((\x, z), y ) \in \R^{d+1} \times \{ \pm 1\}$ from many $d$-dimensional constant-bounded 
Massart noise instances.  For every $(\x, z)$ outside of a thin slice (see \Cref{fig:hardband}), 
we set $\eta(\x) =1/2$. For $(\x, z)$ in the slice, we set the conditional distribution of $(\x, y)$ 
on $z$ to satisfy the constant-bounded Massart noise condition with respect to some
non-homogeneous optimal halfspace (denoted by $\vec u$ in \Cref{fig:hardband}). 
We show that any hypothesis that achieves error $\opt+ \eps$ in the
$(d+1)$-dimensional ``high-noise'' instance will perform well (on average) on the
$d$-dimensional constant-bounded Massart noise instances --- 
a problem that we have already showed to be hard in the SQ model. 
We refer to \Cref{sub:lower-bound-benign} for the detailed proof.
 
\section{Preliminaries}\label{sec:prelims}
\paragraph{Basic Notation}
For $n \in \Z_+$, let $[n] \eqdef \{1, \ldots, n\}$.
We use small
boldface characters for vectors and capital bold characters for matrices.  For $\bx \in \R^d$ and $i \in [d]$, $\bx_i$
denotes the $i$-th coordinate of $\bx$, and $\|\bx\|_2 \eqdef (\littlesum_{i=1}^d \bx_i^2)^{1/2}$ denotes the $\ell_2$-norm of $\bx$.
We will use $\bx \cdot \by $ for the inner product of $\bx, \by \in \R^d$ and $ \theta(\bx, \by)$ for the angle between $\bx, \by$.
We slightly abuse notation and denote $\vec e_i$ the $i$-th standard basis vector in $\R^d$.
For $\x\in \R^d$ and $V\subseteq \R^d$, $\x_{V}$ denotes the projection of $\x$ onto the subspace $V$. Note that
in the special case where $V$ is spanned from one unit vector $\vec v$, then we simply write $\x_{\vec v}$
to denote $\vec x \cdot \vec v$, i.e., the projection of $\vec x$ onto $\vec v$.
For a subspace $U\subset\R^d$, let $U^{\perp}$ be the orthogonal complement of
$U$. For a vector $\vec w\in\R^d$, we use $\vec w^\perp$ to denote the subspace spanned by vectors
orthogonal to $\vec w$, i.e., $\vec w^\perp=\{\vec u\in \R^d: \vec w \cdot \vec u=0\}$. For a matrix $\vec A\in \R^{d\times d}$,  $\tr(\vec A)$ denotes the trace of the matrix $\vec A$. For a square symmetric matrix $\vec M$, we say that $\vec M$ is positive semi-definite if only if all the eigenvalues of $\vec M$ are non-negative. For $m\in \Z_+$, we denote $\mathcal{S}^m$ the set of positive (symmetric) semi-definite matrices of dimension $m$.
We will use $\1_A$ to denote the characteristic function of the set $A$,
i.e., $\1_A(\x)= 1$ if $\x\in A$ and $\1_A(\x)= 0$ if $\x\notin A$. 

We write $E \gtrsim F$ for two expressions $E$ and $F$ to denote that $E \geq c \, F$, where $c>0$
is a sufficiently large universal constant (independent of the variables or parameters on which $E$ and $F$ depend).
Similarly, we write $E \lesssim F$ to denote that $E \leq c \, F$, where $c>0$
is a sufficiently small universal constant.

\paragraph{Probability Notation}
We use $\E_{x\sim \D}[x]$ for the expectation of the random variable $x$ according to the distribution $\D$ and
$\pr[\mathcal{E}]$ for the probability of event $\mathcal{E}$. For simplicity of notation, we may
omit the distribution when it is clear from the context. Let $\normal( \boldsymbol\mu, \vec \Sigma)$ denote the $d$-dimensional Gaussian distribution with mean $\boldsymbol\mu\in  \R^d$ and covariance $\vec \Sigma\in \R^{d\times d}$, we denote $\phi_d(\cdot)$ the pdf of the $d$-dimensional Gaussian and we use the $\phi(\cdot)$ for the pdf of the standard normal. In this work we usually consider the standard normal, i.e., $\mu = \vec 0$ and $\vec \Sigma = \vec I$, 
and therefore we denote it simply $\normal$; the dimension will be always clear from the context. Moreover, we denote as $\normal(\mu_1,\mu_2)$ the $2$-dimensional Gaussian distribution with mean $(\mu_1,\mu_2)$ and $\normal(\mu)$ the $1$-dimensional Gaussian distribution with mean $\mu$.  For $(\x,y)$ distributed according to $\D$, we denote $\D_\x$ to be the distribution of $\x$ and $\D_y$ to be the distribution of $y$. For unit vector $\vec v\in \R^d$, we denote $\D_{\vec v}$ the distribution of $\x$ on the direction $\vec v$, i.e., the distribution of $\x_{\vec v}$. For a set $B$ and a distribution $\D$, we denote $\D_B$ to be the distribution $D$ conditional on $B$.
We define the standard $L_p$ norms with respect to the Gaussian measure, i.e., $\|g(\x)\|_p = ( \E_{\x \sim \normal} [ |g(\x)|^p)^{1/p}$.  To distinguish them from vector norms we shall write $\|g(\x)\|_p$ instead
of $\|g\|_p$ when it is not clear from the context.
We also need the following fact providing upper and lower bounds on Gaussian tails
and the Gaussian anticoncentation property for intervals.
\begin{fact}[Gaussian Density Properties] 
\label{fct:gaussian-tails}
Let $\normal$ be the standard one-dimensional normal distribution. Then, the following properties hold:
\begin{enumerate}
    \item For any $t>0$, it holds $e^{-t^2}/4\leq \pr_{z\sim\normal}[z>t]\leq e^{-t^2/2}/2$.
    \item For any $a,b\in \R$ with $a\leq b$, it holds  $\pr_{z\sim\normal}[a\leq z\leq b]\leq (b-a)/\sqrt{2\pi}$.
\end{enumerate}
\end{fact}

\paragraph{Function Families (LTFs and Polynomials)}
We use ${\cal C}_{V}$ to denote the set of Linear Threshold Functions (LTFs)
with normal vector contained in $V\subseteq\R^d$,
i.e., ${\cal C}_{V}=\{\sign(\vec v \cdot \x +t): \vec v\in V, \snorm{2}{\vec v}=1, t \in \R\}$;
when $V=\R^d$, we simply write $\cal C$.
Moreover, we define ${\cal C}_0$ to be the set of homogeneous LTFs,
i.e., ${\cal C}_{0}=\{\sign(\vec v \cdot \x): \vec v\in \R^d, \snorm{2}{\vec v}=1\}$.
We denote by $\mathcal{P}_{k,d}$ the space of polynomials on $\R^d$ of degree at most $k$.
We will use the following fact shows that for the probability that two
halfspaces disagree can be upper bounded by the difference of their
thresholds plus the angle of their corresponding normal vectors.
\begin{fact}[see, e.g., Fact 3.5 of \cite{DKKTZ21}]\label{fct:gaussian-halfspaces}
Let $\normal$ be the standard normal distribution in $\R^d$. Let $\vec v,\vec u$ be unit vectors in
$\R^d$ and $t_1,t_2\in \R$. It holds $\pr_{\x\sim \normal}[\sign(\vec u \cdot \x + t_1)\neq \sign(\vec v
\cdot \x + t_2)]\leq O(\theta(\vec u,\vec v)) + O(|t_1-t_2|)$.
\end{fact}

\paragraph{Multilinear Algebra} 
Let $\alpha=(\alpha_1,\alpha_2,\ldots,\alpha_d)$ be a $d$-dimensional
multi-index vector, where for all $i\in[d]$, $\alpha_i$ is non-negative integer.
We denote $|\alpha|=\sum_{i=1}^d \alpha_i$ and for a $d$-dimensional vector
$\vec w=(\vec w_1,\vec w_2,\ldots,\vec w_d)$, we denote $\vec w^\alpha =
\prod_{i=1}^d \vec w_i^{\alpha_i}$.
An order-$k$ tensor
$\matr A$ is an element of the $k$-fold tensor product of subspaces $\matr A
\in \mathcal{V}_1 \otimes \ldots \otimes \mathcal{V}_k$.
We will be exclusively working with subspaces of $\R^d$ so a tensor $A$ can
be represented by a sequence of coordinates, that is $A_{i_1,\ldots,i_k}$.
The tensor product of order $k$ tensor $\matr A$ and an order $m$ tensor $\matr
B$ is an order $k + m$ tensor defined as $(\matr A \otimes \matr
B)_{i_1,\ldots, i_k,j_1,\ldots,j_m} = \matr A_{i_1,\ldots,i_k} \matr
B_{j_1,\ldots, j_m}$.  Moreover, we denote by $\vec A^{\otimes k}$
the $k$-fold tensor product of $\vec A$ with itself.
We define the dot product of two tensors (of the same order) to be
$\dotp{\matr A}{\matr B} = \sum_{\alpha} \matr A_\alpha \matr B_\alpha$.
  We also denote the Frobenius norm of a tensor by 
  $\snorm{F}{\matr A} = \sqrt{\dotp{\matr A}{\matr A}}$.  
In this work, we will be using the (appropriately) normalized Hermite tensors (\Cref{def:hemite-moment-tensor}).  For these tensors we have 
the following fact.
\begin{fact}
    \label{fct:hermite-tensor}
    Let $\x, \vec v \in \R^d$ and let $h_m:\R \mapsto \R$ be the univariate
    normalized (with respect to the standard normal) degree-$m$ Hermite
    polynomial and $\vec H^m$ be the Hermite moment tensor of
    \Cref{def:hemite-moment-tensor}.  It holds
    \(
    h_m(\x \cdot \vec v) = \vec H^m \cdot \vec v^{\otimes m} 
    \).
\end{fact}

 \newcommand{\sloglog}{ \sqrt{\log(1/(\beta \gamma))}} 

\section{A Certificate-Based Algorithm For Learning General Halfspaces with Constant-Bounded Massart Noise}\label{sec:constant-bounded-arbitrary}
In this section we show that, when the noise upper bound $\eta$ is bounded away 
from $1/2$, we can learn at most $(1-\gamma)$-biased halfspaces with sample complexity 
and runtime $d^{O(\log(1/\gamma))} \poly(1/\eps)$. Our algorithmic result matches
our SQ lower bound, given in \Cref{thm:intro-biased-sq-lb}.
We first state the formal version of \Cref{thm:intro-bounded-massart-upper-bound}.
\begin{theorem}[Learning $(1-\gamma)$-biased Halfspaces with Constant-Bounded Massart Noise]
    \label{thm:massart-upper-bound}
    Let $\D$ be a distribution on $\R^d \times \{\pm 1\}$ whose 
    $\x$-marginal is the standard Gaussian.  
    Assume that $\D$ satisfies the $\eta$-Massart noise
    condition with respect to some $(1-\gamma)$-biased 
    optimal halfspace and $\beta = 1-2\eta$. 
    Let $\eps,\delta\in(0,1]$. There exists an algorithm that
    draws $N = d^{O(\log(1/(\gamma \beta)))}\poly(1/\eps)\log(1/\delta)$ samples from $\D$, 
    runs in time $\poly(N, d)$, and computes a halfspace $h \in {\cal C}$ such that with probability 
    at least $1-\delta$ we have that
    $
    \pr_{(\x, y) \sim \D}[h(\x) \neq y] \leq \opt  + \eps\,.
    $
\end{theorem}

The main technical ingredient in our proof is \Cref{pro:certificate-Biased-Massart} 
where we show that, given a linear function $\ell(\x)=\vec w \cdot \x - t$ 
whose classification error is greater than 
$\opt +\eps$, there exists a low-degree polynomial certifying function.  We prove the existence
of such a certificate in \Cref{ssec:low-degree-sos-certificate}.
In \Cref{sub:certificate_optimization}, we show that we can efficiently compute
the low-degree polynomial certificate, and we bound the sample complexity and runtime
of the corresponding SDP of the optimization problem.
Finally, in \Cref{ssec:online}, we show that given an efficient certificate oracle,
we can use online gradient descent in order to learn the optimal halfspace, i.e.,
we provide the formal version and proof of \Cref{pro:certificate-reduction}.

\subsection{The Low-Degree Polynomial Certificate}
\label[sub]{ssec:low-degree-sos-certificate}
Here we show that given any $\eps$ suboptimal linear hypothesis
$\ell(\x) = \vec w \cdot \x - t$ there exists a low degree polynomial certifying function.
We establish the following. 

\begin{proposition}[Certificate for General Halfspaces with Massart Noise: $\eta < 1/2$.]
    \label{pro:certificate-Biased-Massart}
    Let $\D$ be a distribution on $\R^d \times \{\pm 1\}$ with standard normal $\x$-marginal. Assume that $\D$ satisfies the $\eta$-Massart noise
    condition with respect to some at most $(1-\gamma)$-biased optimal halfspace $f(\x)$,
    where $\gamma \in (0, 1/2]$.
    Let $\ell(\x)$ be  any linear function such that
    $\pr_{(\x, y) \sim \D}[\sgn(\ell(\x)) \neq y] \geq \opt + \eps$.
    There exist $r_1, r_2 \in \R$ and polynomial 
    $q(\x) = \sum_{|\alpha| \leq k} c_\alpha \x^\alpha$ of degree
    $k=\Theta(\log(\frac 1 {\gamma \beta } ))$ and $\|q(\x)\|_2 = 1$ such that
    \[
    \E_{(\x, y) \sim \D}[ \ell(\x) y \
    \1\{r_1 \leq \ell(\x) \leq r_2\} q^2(\x) ]
    \leq -  \eps^4  \poly(\beta \gamma)  \| \ell(\x) \|_2 
    \,.
    \]
    Moreover, the coefficients of $q$ are bounded, specifically 
    $\sum_{|\alpha| \leq k} |c_\alpha|  \leq d^{O(k)}$.
\end{proposition}

To prove the proposition, we distinguish different cases for the hypothesis $h(\x)$.
First we assume that $\vec w = 0$, i.e.,
the guess is not formally a halfspace but a constant function.  This
case may look easy to handle but captures many difficulties
of designing certificates for biased halfspaces under Massart noise:
when the constant guess is $-1$, we have to construct a polynomial that puts a large weight to the 
positive region of $f(\x)$, i.e., $\{\x : f(\x) = +1\}$.  Since the probability
of this region can be as small as $\gamma$, it may be far from the origin 
(i.e., $t^\ast = \Theta(\sqrt{\log(1/\gamma)})$); see \Cref{fig:constant-hypotheses}.
It can be seen that it is challenging to isolate such regions far from the origin
and the required polynomial degree is roughly $\log(\frac{1}{\beta \gamma})$.
The argument and the details for this case are given in \Cref{sub:constant-hypotheses}.
In \Cref{sub:large-threshold}, we show that, similarly to the case of constant hypotheses,
we can construct polynomial certificates against very biased hypotheses, i.e.,
when $\ell(\x) = \vec w \cdot \x - t$ and it holds $t/\|\vec w\|_2 
\geq C \sloglog $.
Finally, in \Cref{sub:small-threshold}, we handle the remaining hypotheses,
namely we provide certificates when $t/\|\vec w\|_2 \leq C \sloglog $.
Having \Cref{lem:zero-case}, \Cref{lem:large-threshold-hypotheses}
and \Cref{lem:small-threshold}, we immediately obtain \Cref{pro:certificate-Biased-Massart}.

\subsubsection{Certificate Against Constant Hypotheses}
\label[sub]{sub:constant-hypotheses}
Here we show that we can construct a polynomial certificate that 
certifies the non-optimality of any constant hypothesis, i.e., $h(\x)$ corresponds
to the constant function $+1$ (or $-1$) for every $\x \in \R^d$.  
In this case, we do not require a ``band''
$\1\{r_1 \leq \vec w \cdot \x \leq r_2\}$ as part of the certificate, and therefore we may set $r_1 = -\infty$, $r_2 = +\infty$ in \Cref{pro:certificate-Biased-Massart}.  We prove the following lemma:

\begin{lemma}[Certificate against Constant Hypotheses] 
    \label{lem:zero-case}
    Let $\D$ be a distribution on $\R^d \times \{\pm 1\}$ with standard normal $\x$-marginal.  Assume that $\D$ satisfies the $\eta$-Massart noise
    condition with respect to some (at least) $(1-\gamma)$-biased optimal halfspace.
For every constant hypothesis $s \in \{-1, +1\}$, there exists a polynomial 
    $q(\x) = \sum_{|\alpha| \leq k} c_\alpha \x^\alpha$ of degree
    $k = \Theta(\log(\frac{1}{\beta \gamma}))$ with $\|q\|_2 = 1$, 
    and sum of (absolute) coefficients $\sum_{|\alpha| \leq k} |c_\alpha| \leq d^{O(k)}$
    such that
    \[
    \E_{(\x, y) \sim \D}[ s  y \ q^2(\x) ] \leq  -  \poly(\beta \gamma) \,.
    \]
\end{lemma}
\begin{proof}
    Our plan is to pick the polynomial $q(\x)$ so that it takes larger values
    to the side of the optimal halfspace $f(\x) = \sgn(\vec w^\ast \cdot \x - t^\ast)$
    that has the opposite sign of $s$.  In the example of \Cref{fig:constant-hypotheses},
    in order to have a certificate for the constant hypothesis $s=-1$,
    we want to have a polynomial that takes much larger values in the blue area than the red area.
    We shall set $q(\x) = p(\vec w^\ast\cdot \x)$
    for some one-dimensional polynomial $p: \R \to \R$.
    Since the certificate that we construct only depends on the direction
    of the optimal halfspace, it follows that we can project the points on the subspace
    spanned by $\vec w^\ast$. The following claim shows that the projection of
    a distribution with $\eta$-Massart noise onto a lower dimensional subspace
    that contains the direction of the optimal halfspace also satisfies the $\eta$-Massart noise condition with respect 
    to the same optimal halfspace.  
    \begin{claim}[Projections preserve Massart Noise] 
        \label{clm:projections-preserve-massart-noise}
        Let $\D$ be a distribution on $\R^d \times \{\pm 1\}$  satisfying the $\eta$-Massart 
        noise condition with respect to some optimal halfspace $f(\x) : \R^d \mapsto \{\pm 1\}$.
        Let $V$ be any subspace of $\R^d$ that contains the normal vector of the optimal halfspace
        $f(\x)$.
        Then, for every function $g: \R^d \mapsto \R$ it holds 
        \[
        \E_{(\x, y) \sim \D}[g(\proj_V(\x)) y] = 
        \E_{\vec v \sim (\D_\x)_V}[g(\vec v) \beta_V(\vec v) f(\vec v)]\;,
        \]
        where $\beta_V(\x): \R^d \mapsto \R$ satisfies $\beta_V(\x) \in [1-2 \eta, 1]$.
    \end{claim}
    \begin{proof}
        Assume that the optimal halfspace is $f(\x) = \sgn(\vec w^\ast \cdot \vec x - t^\ast)$.
        From the fact that the distribution $\D$ satisfies the $\eta$-Massart noise condition, 
        it holds that $\E_{(\x,y)\sim \D}[y|\vec x] = f(\x) (1-\eta(\x)) - f(\x) \eta(\x) = 
        f(\x) (1- 2 \eta(\x)) = f(\x) \beta(\x)$.
        Therefore, we have
        \begin{align*}
            \E_{(\x, y)\sim \D}[ g(\proj_V(\x)) y]
            &= 
            \E_{\x \sim \D_\x}[ g(\proj_V(\x )) \beta(\x) f(\x)]
            \\
            &=
            \E_{\vec v \sim (\D_\x)_{V} }
            \Big[
            \E_{\vec u \sim (\D_\x)_{V^\perp}}
            [
            g(\vec v)
            \beta(\vec u + \vec v)
            \sgn(\vec w^\ast \cdot (\vec u + \vec v) - t^\ast)\
            \Big]
            \\
            &=
            \E_{\vec v \sim (\D_\x)_{V}}
            \Big[
            \sgn(\vec w^\ast \cdot \vec v - t^\ast)\
            g(\vec v)
            \
            \E_{\vec v \sim (\D_\x)_{V^\perp}}[\beta(\vec u + \vec v)]
            \Big] \,,
        \end{align*}
        where we used the fact that the subspace $V$ contains
        the normal vector of the optimal halfspace $\vec w^\ast$.
        Observe that the ``projected'' noise function
        $\beta_V(\vec v) :=
        \E_{\vec u \sim (\D_\x)_{V^\perp}}[\beta(\vec u + \vec v)]$
        again satisfies the $\eta$-Massart noise condition, i.e.,
        $\beta_V(\vec v)\in [1- 2 \eta, 1] = [\beta, 1]$. This concludes the proof of \Cref{clm:projections-preserve-massart-noise}.
    \end{proof}
    
    Since $q(\x) = p(\vec w^\ast \cdot \x)$, we can use \Cref{clm:projections-preserve-massart-noise} to project
    onto the subspace spanned by $\vec w^\ast$ and obtain that
    \begin{align}
        \label{eq:certificate-case-1-projected}
        \E_{(\x, y) \sim \D}[ s y q^2(\x)]
        =
        \E_{\vec v \sim (\D_\x)_{\vec w^\ast}}
        \Big[s~ \sgn(\vec w^\ast \cdot \vec v  - t^\ast)\
        p^2(\vec w^\ast \cdot \vec v)
        \
        \beta_{\vec w^\ast}(\vec v)
        \Big] \,,
    \end{align}
    where the ``projected'' noise function
    $ \beta_{\vec w^\ast}(\vec v) \in [\beta, 1]$.
    For simplicity, in what follows, we will continue denoting $\beta$ the projected noise function $\beta_{\vec w^\ast}$.
    Without loss of generality, we may assume that $\vec w^\ast = \vec e_1$ and in that
    case, using the fact that the projection of $\D_\x$ onto $\vec w^\ast$ is 
    a one-dimensional standard normal distribution, 
    we have that the expression of \Cref{eq:certificate-case-1-projected} can be simplified as 
    \begin{equation}
        \label{eq:certificate-case-1-one-d}
        \E_{(\x, y) \sim \D}[ s y q^2(\x)]
        =
        \E_{\x_1 \sim \normal}[ s \beta(\vec x_1) \sgn(\x_1 - t^\ast)\ p^2(\x_1)] \,.
    \end{equation}
    Moreover, to simplify the notation, assume that the constant guess is $s = -1$
    (the case of $s = +1$ is similar)
    and that the halfspace puts $\gamma$ mass on the positive side, i.e.,
    $\pr_{\x \sim \D_\x}[f(\x) = +1] = \gamma$; see also \Cref{fig:constant-hypotheses}. 
    In that case, we want to construct a univariate polynomial $p(\x_1)$ so that
    $\E_{\x_1 \sim \normal}[ \beta(\x_1) \sgn(\x_1 - t^\ast)\ p^2(\x_1)]
    $ is sufficiently positive.  Observe that the worst case noise function $\beta(\x_1)$
    is to set $\beta(\x_1) = \beta$ for all points $\x_1 \geq t^\ast$
    and $\beta(\x_1) = 1$ for all $\x_1 < t^\ast$.  In that case, we want to find
    a univariate polynomial $p$ such that 
    \begin{equation}
        \label{eq:beta-polynomial-ratio}
        \beta  \E_{\x_1 \sim \normal}[ \1 \{\x_1 \geq t^\ast\} p^2(\x_1) ] 
        \geq  \E_{\x_1 \sim \normal}[ \1\{\x_1 \leq t^\ast \} p^2(\x_1) ]  \,.
    \end{equation}
    To do that, we show our main structural result, i.e., that
    as long as the degree $k$ is larger than $(t^\ast)^2$,
    there exists a polynomial that can make the above inequality true. In fact,
    we can make the ratio of 
    $ \E_{\x_1 \sim \normal}[ \1 \{\x_1 \geq t^\ast\} p^2(\x_1) ] 
    /
    \E_{\x_1 \sim \normal}[ \1\{\x_1 \leq t^\ast \} p^2(\x_1) ] 
    $
    grow exponentially fast with respect to the degree-$k$ of the polynomial $p$.
    We prove the following lemma.
    \begin{lemma}[Polynomial Shift]
        \label{lem:taylor-exponential-polynomial}
        Let $t \in \R$ and $b\geq 1$. 
        There exists a univariate polynomial $p(x) = \sum_{i = 0}^k a_i x^i$ of
        degree $k=\Theta(t^2 + b^2)$, $L_2$ norm $\|p(x)\|_2 = 1$, and 
        sum of (absolute) coefficients $\sum_{i=0}^k |a_i| = O(1)$, such that
        \begin{equation*}
            \frac{
                \E_{x \sim \normal}[p^2(x) \1\{x \geq t\}]
            }
            {
                \E_{x \sim \normal}[p^2(x) \1\{x \leq t\}] 
            } 
            \geq e^{b^2} 
            ~~
            \text{ and }
            ~~
            \E_{x \sim \normal}[p^2(x) \1\{x \geq t\}]
            \geq e^{-O(k)}\, .
        \end{equation*}       
\end{lemma}
    \begin{proof}
        We first show that there exists a positive function that makes the ratio
        of \Cref{lem:taylor-exponential-polynomial} large. We shall consider the exponential 
        function $x \mapsto e^{c x}$
        for some constant $c \geq |t|$ to be determined later.  
        We have that
        \begin{equation}
            \label[ineq]{eq:expoenential-bound}
            \frac{
                \E_{x \sim \normal}[e^{c x} \1\{x \geq t\}]
            }
            {
                \E_{x \sim \normal}[e^{c x} \1\{x \leq t\}] 
            }
            =
            \frac{
                \pr_{x \sim \normal}[x \geq t - c]
            }
            {
                \pr_{x \sim \normal}[x \leq t - c]
            }
            \geq e^{(c-t)^2/2} \,,
        \end{equation}
        where the last inequality follows from the fact that $c\geq t$ and 
        standard lower bounds on the tail probability of the standard normal distribution.
        Therefore, for $c = \Theta(|t| + b)$ we immediately obtain the claimed bound.
        We now have to replace the function $e^{c x}$ by the square of a polynomial $p$.
        To do so we use the Taylor expansion of the exponential function.
        Denote by $S_k(x)$ the degree-$k$ Taylor expansion of $e^{x}$, i.e.,
        $S_k(x) = \sum_{i=0}^k x^i/i!$.  
        Using Taylor's remainder theorem and the fact that $e^{x} = \sum_{i=0}^\infty x^i/i!$ 
        for all $x\in \R$, we obtain the following fact.
        \begin{fact}[Taylor Expansion of $e^x$]
            \label{fct:taylor-expansion}
            Fix $R > 0$ and let $S_k(x)$ be the degree-$k$ Taylor expansion of $e^{x}$.  
            \begin{enumerate}
                \item
                For all $x \in [-R, R]$ we have that 
                $|S_k(x) - e^{x}| \leq e^R \frac{ R^{k+1} } { (k+1)! }
                $.
                \item   
                For all $x \in \R$ it holds that
                \label{eq:taylor-series-upper-bound}
                $|S_k(x)| \leq e^{|x|} $.
            \end{enumerate}
        \end{fact}
        Since we have to construct the square of a polynomial $p$ we cannot simply use $S_k(x)$.
        However, we can consider the function $e^{2 c x}$ and use $S_k^2(c x)$ 
        as our approximation.  We first show that 
        as long as $k$ is at least $\Omega(c^2)$, we have that 
        $S_k^2(c x)$ is a good approximation of $e^{2 c x}$ with respect to the $L_1$ norm.
        \begin{claim}[$L_1$ error of $S_k(x)$]
            \label{clm:ell-one-approximation-taylor}
            Assume that $k \geq 32 c^2$.  It holds 
            \begin{equation*}
                \| S_k^2(c x) - e^{2 c x} \|_1 \lesssim e^{-k/32} \,.
            \end{equation*}
        \end{claim}
        \begin{proof}
            For every $x \in [-R, R]$ it holds that
            \begin{equation}
                \label{eq:taylor-squared-error}
                |S_k^2(x) - (e^{x})^2| = |S_k(x) - e^{x}|  |S_k(x) + e^{x}|  
                \leq  e^R \frac{ R^{k+1} } { (k+1)! }~~2 e^R
                = 2 e^{2 R} \frac{ R^{k+1} } { (k+1)! } 
                \,,
            \end{equation}
            where we used the second item of \Cref{fct:taylor-expansion}, i.e, that $|S_k(x)| \leq e^{R}$.
            The above pointwise approximation guarantee together with the strong concentration
            properties of the Gaussian distribution (see, e.g., \Cref{fct:gaussian-tails}) allow us to get a bound for the $L_1$-approximation 
            error of $S_k^2(x)$.
            Using the error bound of \Cref{eq:taylor-squared-error},
            we bound the $L_1$-approximation error as follows:
            \begin{align}
                \label{eq:ell-one-approximation-taylor}
                \E_{x \sim \normal}[|S_k^2(c x) - e^{2 c x}|] 
                &\leq 
                \max_{|x| \leq R/c}|S_k^2(c x) - e^{2 c x}|
                + 
                \E_{x \sim \normal}[|S_k^2(c x) + e^{2 c x}| ~ \1\{|x| \geq R/c\}] 
                \nonumber
                \\
                &\leq
                2 e^{2 R} \frac{ R^{k+1} } { (k+1)! } 
                + 
                ( \| S_k^2(c x) \|_2 + \| e^{2 c x} \|_2) 
                \sqrt{\pr_{x \sim \normal}[|x| \geq R/c] }
                \nonumber
                \\
                &\leq
                2 e^{2 R} \frac{ R^{k+1}} { (k+1)! } 
                + 
                2 \sqrt{2} e^{4 c^2 - (R/c)^2/2}
                \,,
            \end{align}
            where for the last inequality we used the second item of \Cref{fct:taylor-expansion}
            to obtain that
            $  \| S_k^2(x) \|_2 \leq \| e^{2 c |x|} \|_2$, the fact
            that  $\| e^{2 c |x|} \|_2 \leq \sqrt{2} e^{4 c^2}$,
            and the tail probability upper bound 
            for the normal distribution, i.e.,
            $ \pr_{x \sim \normal}[|x| \geq (R/c)] \leq e^{-(R/c)^2/2}$,
            see \Cref{fct:gaussian-tails}.
            By choosing $k = 32 m c^2$, for any $m \geq 1$, 
            and  $R= k/8$  the estimate of \Cref{eq:ell-one-approximation-taylor} becomes
            \begin{equation*}
\| S_k^2(c x) - e^{2 c x} \|_1 \lesssim e^{- m c^2} \,. \qedhere
            \end{equation*}
        \end{proof}
        We are now ready to finish the proof of \Cref{lem:taylor-exponential-polynomial}.
        Recall that to make \Cref{eq:expoenential-bound} true,
        we chose $c = \Theta(b+|t|)$.  Let $k = 32 m c^2$, for 
        some sufficiently large absolute constant $m>0$.
        Given \Cref{eq:expoenential-bound}, we show that by replacing $e^{2cx}$ with $S_k^2(c x)$, i.e., the Taylor expansion of $e^{cx}$ squared, we get a similar formula for the lower-bound of the ratio. It suffices to show that
        \begin{align}
            \label{eq:exponential-polynomial-ratios}
            \frac{
                \E_{x \sim \normal}[e^{2 c x} \1\{x \leq t\}]
            }
            {
                \E_{x \sim \normal}[S_k^2(c x) \1\{x \leq t\}] 
            }
            \geq \frac{1}{2}
            ~~~
            \text{ and }
            ~~~
            \frac{
                \E_{x \sim \normal}[S_k^2(c x) \1\{x \geq t\}]
            }
            {
                \E_{x \sim \normal}[e^{ 2 c x} \1\{x \geq t\}] 
            }
            \geq \frac{1}{2}
            \,.
        \end{align}
        We start from the first inequality of \Cref{eq:exponential-polynomial-ratios}.  
        To prove it, it suffices to show that 
        \begin{equation}
            \label[ineq]{ineq:taylor-approximation-first-inequality}
            \E_{x \sim \normal}[|S_k^2(x) - e^{2 c x}| ~ \1\{ x \leq t\}]  \leq  
            \E_{x \sim \normal}[e^{2 c x} ~ \1\{ x \leq t\} ]  
            \,.
        \end{equation}
        From the $L_1$-bound of \Cref{clm:ell-one-approximation-taylor}, we have that 
        \[\E_{x \sim \normal}[|S_k^2(x) - e^{2 c x}| ~ \1\{ x \leq t\}] \leq \E_{x \sim \normal}[|S_k^2(x) - e^{2 c x}|]\lesssim e^{-mc^2}\;. \]
        Moreover, to bound the $\E_{x \sim \normal}[e^{2 c x} ~ \1\{ x \leq t\} ]$ from below, we use the lower bound for the tails of a Gaussian distribution (\Cref{fct:gaussian-tails}). First, note that $\E_{x \sim \normal}[e^{2 c x} ~ \1\{ x \leq t\} ] =
        e^{2c^2}\pr_{x \sim \normal}[  x \leq t-2c ]$. Therefore, we have that 
        \[
        \E_{x \sim \normal}[e^{2 c x} ~ \1\{ x \leq t\} ]  \geq \frac{1}{4}e^{2c^2} e^{-(t-2c)^2}\geq \frac{1}{4}e^{-7c^2}\;,
        \]
        where we used that $c\geq |t|$. Therefore, 
        \Cref{ineq:taylor-approximation-first-inequality} is true
        for $m$ being a large enough absolute constant.
        
        For the second inequality of \Cref{eq:exponential-polynomial-ratios}, 
        using the $L_1$-bound of \Cref{clm:ell-one-approximation-taylor}
        we obtain that $\E_{x\sim \normal}[S_k^2(c x) \1\{x \geq t\}] \geq 
        \E_{x\sim \normal}[e^{2 c x}  \1\{x \geq t\}] - \| S_k^2(x) - e^{2 c x}\|_1$,
        and therefore, taking $m$ to be larger than an absolute constant we have that, 
        there exists an absolute constant $C>0$, such that
        \[
        \frac{ \E_{x \sim \normal}[S_k^2(c x) \1\{x \geq t\}] }
        { \E_{x \sim \normal}[e^{ 2 c x} \1\{x \geq t\}]  }
        \geq 
        1 - \frac{Ce^{-m c^2}}{ \E_{x \sim \normal}[e^{ 2 c x} \1\{x \geq t\}]   }
        \geq \frac{1}{2}\,,
        \]
        where we used that $ \E_{x \sim \normal}[e^{ 2 c x} \1\{x \geq t\}]  
        = e^{2c^2}\pr_{x \sim \normal}[x \geq t - 2c] \geq e^{2c^2}/2$ since $c \geq t$.
        
        We next show that we can normalize the polynomial $p(x) = S_k(c x)$ without making
        the expectation of $p^2$ over $x \geq t$, i.e., $\E_{x \sim \normal}[p^2(x) \1\{x \geq t\}]$, too small.
        From the second item of \Cref{fct:taylor-expansion}
        we obtain that $\|S^2_k(c x)\|_2  \leq \|e^{2 c |x|} \|_2 = e^{O(c^2)}$.
        The result follows from the second inequality of \Cref{eq:exponential-polynomial-ratios}
        and the fact that  $\E_{x \sim \normal}[e^{2 c x} \1\{ x \geq t\}] \geq e^{2 c^2}/2$.

        Finally, we bound the coefficients of the polynomial $S_k^2(c x)$.  
        We have \( S_k^2(c x) = \left(\sum_{i = 0}^k \frac{c^i}{i!} x^i \right)^2 \).
        It holds that $\left(\sum_{i = 0}^k \frac{c^i}{i!} \right)^2 = e^{2 c}$.
        From the $L_1$ approximation guarantee of \Cref{clm:ell-one-approximation-taylor} 
        we obtain that $\|S_k^2(c x)\|_2 \geq \|S_k^2(c x )\|_1 \geq \| e^{2 c x}\|_1 - 1 \geq e^{2 c^2}/2$,
        since $c\geq 1$.
        We conclude that the sum of the absolute coefficients of $S_k^2(c x)/\|S_k^2(c x\|_2$ is at 
        most $2$, which implies that the sum of the absolute coefficients of $S_k(c x)$ is 
        at most $\sqrt{2}$.

    \end{proof}

    To conclude the proof of \Cref{lem:zero-case}, notice that using
    \Cref{lem:taylor-exponential-polynomial}, we obtain that there exists a
    polynomial $p$ of degree $(t^\ast)^2 + \log(1/\beta)$ such that
    \Cref{eq:beta-polynomial-ratio} is true.  Moreover, from the same lemma, we
    obtain that $\E_{\x_1 \sim \normal}[p^2(\x_1)] = e^{- O((t^\ast)^2 +
    \log(1/\beta))}$.  Since the halfspace is at least $(1-\gamma)$-biased, from
    standard upper bounds on the Gaussian tails, we obtain that $|t^\ast|
    \lesssim \sqrt{\log(1/\gamma)}$ obtaining the claimed bound for the degree
    of the certifying polynomial $p$.  
    It remains to bound the coefficients of the multivariate
    polynomial $q(\x) = p(\vec w^\ast \cdot x)$.
    We are going to use the following fact.
    \begin{fact}[Lemma 3.6 of \cite{DKTZ20b}]
    \label{fct:polynomial_norms}
    Let $p(t) = \sum_{i=0}^k c_i t^i$ be a  degree-$k$ univariate polynomial.
    Given $\vec w \in \R^d$ with $\snorm{2}{\vec w} \leq 1$, define the multivariate polynomial
    $q(\vec x) = p(\dotp{\vec w}{\vec x}) = \sum_{S: |S| \leq k} C_S \bx^S$. Then we have that
    $
    \sum_{S:|S| \leq k} C_S^2 \leq d^{2k} \sum_{i=0}^k c_i^2 \,.
    $
\end{fact}
From \Cref{lem:taylor-exponential-polynomial} we know that the 
sum of the absolute coefficients of $p$ is $O(1)$.  Thus using \Cref{fct:polynomial_norms}
we obtain that the sum of absolute coefficients of $p(\vec w \cdot \x)$ is at most $d^{O(k)}$.

\end{proof}

\subsubsection{Certificate Against ``Large Threshold''  Halfspaces}
\label[sub]{sub:large-threshold}
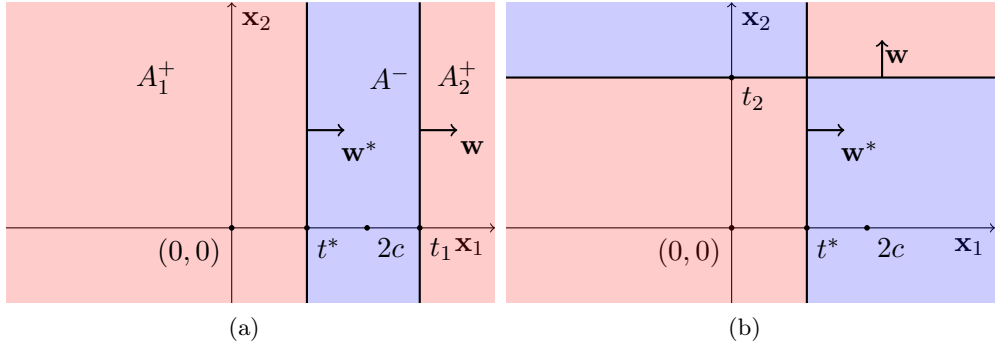
\begin{figure}
    \centering
    \subfloat[]{
        \begin{tikzpicture}[scale=1]
            \label{fig:large-hypotheses-1}
            \coordinate (start) at (0.5,0);
            \coordinate (center) at (0,0);
            \coordinate (end) at (0.5,0.5);
            \draw[->] (-3,0) -- (3.5,0) node[anchor=north west,black,below left] {$\x_1$};
            \draw[->] (0,-1) -- (0,3) node[anchor=south east,below right] {$\x_2$};
            \draw[fill=red,opacity=0.2,draw=none] (2.5,-1)--(2.5,3)--(3.5,3)--(3.5,-1);
            \draw[fill=blue,opacity=0.2,draw=none] (1,-1)--(1,3)--(2.5,3)--(2.5,-1);
            \draw[fill=red,opacity=0.2,draw=none] (1,-1)--(1,3)--(-3,3)--(-3,-1);
            \draw[black,thick] (2.5,-1) -- (2.5,3);
            \draw[thick ,->] (2.5,1.3) -- (3,1.3) node[right=2mm,below] {$\vec w$};
            \draw[black,thick] (1,-1) -- (1,3);
            \draw[thick ,->] (1.0,1.3) -- (1.5,1.3) node[right=2mm,below] {$\vec w^\ast$};
            \draw[fill=black] (1,0) circle (0.03) node[below right] {$t^\ast$};
            \draw[fill=black] (2.5,0) circle (0.03) node[below right] {$t_1$};
            \draw[fill=black] (1.8,0) circle (0.03) node[below right] {$2c$};
            \node at (-1,2) {$A_1^+$};
            \node at (3,2) {$A_2^+$}; 
             \node at (2.1,2) {$A^-$}; 
             \draw[fill=black] (0,0) circle (0.03) node[below left] {$(0,0)$};
    \end{tikzpicture}}
    \subfloat[]{
        \begin{tikzpicture}[scale=1]
            \label{fig:large-hypotheses-2}
                         \draw[fill=black] (0,0) circle (0.03) node[below left] {$(0,0)$};
            \coordinate (start) at (0.5,0);
            \coordinate (center) at (0,0);
            \coordinate (end) at (0.5,0.5);
            \draw[->] (-3,0) -- (3.5,0) node[anchor=north west,black,below left] {$\x_1$};
            \draw[->] (0,-1) -- (0,3) node[anchor=south east,below right] {$\x_2$};
            \draw[fill=red,opacity=0.2,draw=none] (1,2)--(1,3)--(3.5,3)--(3.5,2);
            \draw[fill=blue,opacity=0.2,draw=none] (1,-1)--(1,2)--(3.5,2)--(3.5,-1);
            \draw[fill=red,opacity=0.2,draw=none] (1,-1)--(1,2)--(-3,2)--(-3,-1);
            \draw[fill=blue,opacity=0.2,draw=none] (1,2)--(1,3)--(-3,3)--(-3,2);
            \draw[black,thick] (-3,2) -- (3.5,2);
            \draw[thick ,->] (2,2) -- (2,2.5) node[right=2mm,below] {$\vec w$};
            \draw[black,thick] (1,-1) -- (1,3);
            \draw[thick ,->] (1,1.3) -- (1.5,1.3) node[right=2mm,below] {$\vec w^\ast$};
            \draw[fill=black] (1,0) circle (0.03) node[below right] {$t^\ast$};
            \draw[fill=black] (0,2) circle (0.03) node[below right] {$t_2$};
                        \draw[fill=black] (1.8,0) circle (0.03) node[below right] {$2c$};
    \end{tikzpicture}}
    \caption{
        \textbf{(a)}  The guess $\vec w= \vec e_1 - t_1$ and $t_1$ is larger than $C \sloglog$ for
        some sufficiently large constant $C$.  We can pick $c$ such that $2 c - t^\ast$ 
        and $t_1 - 2 c$ are both large constants making the exponential shift
        certificate $e^{2 c \x_1}$ work, similarly to the constant case of \Cref{fig:constant-hypotheses}.
        \\
        \textbf{(b)}  The guess $\vec w = \vec e_2 - t_2$ and $t_1$ is larger than $C \sloglog$ for
        some sufficiently large constant $C$.  The region $\x_2 \geq t_2$ has probability
        at most $e^{-O(t_2^2)}$ and for $\x_2 \leq t_2$ we can treat $\x_2 - t_2$ as a constant
        negative guess.
    }
    \label{fig:large-hypotheses}
\end{figure}

We now deal with the case where the halfspace has a large threshold but is not
the constant hypothesis.  In particular, we assume that $h(\x) = \sgn(\vec w
\cdot \x - t)$ with $t/\|\vec w\| \geq C  \sloglog $ for some sufficiently large
absolute constant $C$. This case is, in fact, a generalization of the constant
hypothesis case that corresponds to $\vec w = \vec 0$.  In this case we will
show that roughly the same polynomial that we used for constant guesses in
\Cref{sub:constant-hypotheses} can also be made to work for very biased guesses.
We prove the next lemma.

\begin{lemma}[Certificate against ``Large Threshold'' Hypotheses] \label{lem:large-threshold-hypotheses}
    Let $\D$ be a distribution on $\R^d \times \{\pm 1\}$ with standard normal $\x$-marginal.  Assume that $\D$ satisfies the $\eta$-Massart noise
    condition with respect to some (at most) $(1-\gamma)$-biased optimal halfspace.
    Define the linear function $\ell(\x) = \sgn(\vec w \cdot \x - t)$
    and assume that $t/\snorm{2}{\vec w} \geq C \sloglog$
    for some sufficiently large absolute constant $C$.
    Then, there exists polynomial $q(\x) = \sum_{|\alpha| \leq k} c_\alpha \x^\alpha$ of 
    degree $\Theta(\log(\frac 1 {\beta \gamma}))$, $L_2$ norm $\|q(\x)\|_2 = 1$, and
    sum of absolute coefficients $\sum_{|\alpha| \leq k } |c_\alpha| = d^{O(k)}$, such that
    \[
    \E_{(\x, y) \sim \D}[ \ell(\x) y q^2(\x) ]
    \leq - \|\ell(\x)\|_2 ~ \poly(\beta \gamma)
    \,.
    \]
\end{lemma}
\begin{proof}
    Using \Cref{clm:projections-preserve-massart-noise}, we can
    project the distribution $\D$ on the $2$-dimensional subspace $V$ spanned by $\vec w$ and $\vec w^\ast$. We have
    \begin{equation}
        \label{eq:projection-large-threshold}
        \E_{(\x, y) \sim \D}[\ell(\x) y p^2(\vec w^\ast \cdot \x)]
        =
        \E_{\vec v \sim \D_V}[\ell(\vec v) \beta_V(\vec v) \sgn(\vec w^\ast \cdot \vec v - t^\ast) p^2(\vec w^\ast \cdot \vec v)]\,.
    \end{equation}
    In what follows, we again abuse notation and denote $\beta_V(\vec v)$ simply as $\beta(\vec v)$.
    By the spherical symmetry of the Gaussian distribution, we can, without loss of generality, 
    assume that $\vec w^\ast = \vec e_1$ and $\vec w = \| \vec w \|_2 (\cos \theta \vec e_1 + \sin \theta \vec e_2)$.  Moreover, 
    we may assume that $t>0$ and $\theta \in [0, \pi/2]$ (see, e.g., \Cref{fig:large-hypotheses})
    since the other cases are similar.
    Our certificate will be the same polynomial as in the previous case of \Cref{lem:zero-case}.
    In particular, we will again use the square of the Taylor 
    approximation $S_k^2(c \x_1)$ of the exponential $e^{2 c \x_1}$.  
    
    We decompose the problem of proving that this polynomial is a certificate for general
    halfspaces to the problem of showing that it is a certificate for 
    halfspaces with large thresholds simultaneously in both orthogonal directions. 
    In particular, using \Cref{eq:projection-large-threshold}, we have
    \begin{align}
        \label{eq:orthogonal-decomposition}
        \E_{(\x, y) \sim \D}[&\ell(\x) y p^2(\vec w^\ast \cdot \x)]\nonumber
        \\
        &= 
        \|\vec w\|_2  \cos \theta 
        \E_{(\vec x_1, \vec \x_2) \sim \normal_2}\left [ 
        \left( \vec x_1 - 
        \frac{ t}  { 2 \|\vec w\|_2 \cos \theta}
        \right)
        \beta(\vec x_1,  \vec x_2)   p^2(\vec x_1) 
        \sign(\vec x_1 - t^\ast) \right] 
        \nonumber
        \\
        &+ \|\vec w\|_2 \sin \theta  
        \E_{(\vec x_1, \vec \x_2) \sim \normal_2} \left[ 
        \left(
        \vec x_2 - \frac t { 2 \| \vec w\|_2 \sin \theta} 
        \right)
        \beta(\vec x_1,  \vec x_2)   p^2(\vec x_1) \sign(\vec x_1 - t^\ast) \right] \,.
    \end{align}
    To simplify the notation, set $t_1 =  t/  ({ 2 \|\vec w\|_2 \cos \theta})$ 
    and $t_2 = t /(2 \| \vec w\|_2 \sin \theta)$ and notice
    that by the assumptions of \Cref{lem:large-threshold-hypotheses}, it holds that
    both  $t_1$ and $t_2$ are larger than $C/2 \sloglog$, see \Cref{fig:large-hypotheses}.
    We will show that we can pick $c =\Theta(\log(1/(\beta \gamma))$ 
    and $k = \Theta(c^2)$ so that the polynomial $S_k^2(c \x_1)$
    simultaneously satisfies for $i=1,2$, the following inequality
    \[
    \E_{(\vec x_1, \vec \x_2) \sim \normal_2}[ 
    ( \vec x_i - t_i )
    \beta(\vec x_1,  \vec x_2)  \sign(\vec x_1 - t^\ast)  p^2(\vec x_1) ]  \leq 
    - t_i ~ \poly(\beta \gamma) \,.
    \]
    Since $t_i \geq 1$, it follows that we can replace the $t_i ~ \poly(\beta \gamma) $ above 
    by $(2 t_i + 1) \poly(\beta \gamma)$.
    Having these bounds and using \Cref{eq:orthogonal-decomposition}, we obtain 
    \begin{align*}
        \E_{(\x, y) \sim \D}[\ell(\x) y p^2(\vec w^\ast \cdot \x)]
        &= -(\|\vec w\|_2(\cos\theta+\sin\theta) +t) ~ \poly(\gamma\beta) \\
        &=-(\|\vec w\|_2+t) ~ \poly(\gamma\beta) 
        = - \|\ell(\x)\|_2~ \poly(\beta \gamma)
        \,,
    \end{align*}
    where for the last inequality we used 
    the fact that $\cos \theta + \sin \theta \geq 1$, for $\theta \in [0, \pi/2]$
    and that $\sqrt{\E_{\x \sim \normal}[\ell(\x)^2]} = \sqrt{\|\vec w\|_2^2 + t^2}$.
    We bound the first term in claim \Cref{clm:parallel-bound} and the second in \Cref{clm:orthogonal-bound}.
    
    We first bound the contribution of the direction that is orthogonal to the optimal, i.e., $\vec e_2$.
    At a high-level, we have that, since $t_2$ is a large multiple of $\sloglog$, 
    ``for most" values of $\vec x_2$, the quantity $\x_2 - t_2$ will be negative.
    Therefore, $\vec x_2 - t_2$ roughly corresponds to the constant guess $-1$, that we covered in \Cref{lem:zero-case}.
    \begin{claim}
        \label{clm:orthogonal-bound}
        It holds
        \[
        \E_{(\vec x_1, \vec \x_2) \sim \normal_2}[ 
        ( \vec x_2 - t_2 )
        \beta(\vec x_1,  \vec x_2)  \sign(\vec x_1 - t^\ast)  p^2(\vec x_1) ]  
        \leq - t_2~ \poly(\beta \gamma)\;.
        \]
    \end{claim}
    \begin{proof}
        When $\x_2 - t_2 < 0$, we can treat $\x_2 -t_2$ as a constant negative guess, and 
        use directly the bound obtained in the proof of \Cref{lem:zero-case}, for $s = -1$,
        to make the term 
        $\E_{\x_1 \sim \normal}[(-1) \beta(\vec x_1)  \sign(\vec x_1 - t^\ast)  p^2(\vec x_1) ]
        \leq - C'(\beta\gamma)^{\rho_1}$, for some absolute constants $\rho_1,C'>0$.  
        In particular, since in this case it holds $\x_2 -t_2 \leq 0$, 
        the ``worst-case'' noise function $\beta(\x_1, \x_2)$
        is $\beta(\x_1, \x_2) = \beta$ for all $(\x_1, \x_2)$ such that 
        $\x_1 \geq t^\ast$ and $\beta(\x_1, \x_2) = 1$ for $\x_1 \leq t^\ast$,
        see \Cref{fig:large-hypotheses-2}.
        Notice that this worst case $\beta(\x_1, \x_2)$ does not depend
        on $\x_2$.  Therefore, we obtain 
        \begin{align}
            \label{eq:large-threshold-orthogonal-bound-1}
            \E_{(\vec x_1, \vec \x_2) \sim \normal_2}
            [ &\1\{\x_2 \leq t_2\} ( \vec x_2 - t_2 ) \beta(\vec x_1,  \vec x_2)  \sign(\vec x_1 - t^\ast)  p^2(\vec x_1) ]   
            \nonumber
            \\
            &\leq \E_{\vec \x_2 \sim \normal}[ \1\{\x_2 \leq t_2\} ( \vec x_2 - t_2 )] 
            \E_{\vec x_1 \sim \normal}
            [(-1) \beta(\vec x_1) \sign(\vec x_1 - t^\ast)  p^2(\vec x_1) ]   
            \nonumber
            \\
            &\leq  - \E_{\vec \x_2 \sim \normal}[ \1\{\x_2 \leq t_2\} | \vec x_2 - t_2 | ] 
            ~~  C'(\beta\gamma)^{\rho_1}
            \nonumber
            \\
            &\leq - C'/2 ~ t_2 ~~ (\beta\gamma)^{\rho_1}\,,
        \end{align}
        where the last inequality follows by our assumption that $t_2\geq 1$.
        
        On the other hand, the probability of the region $\{\vec x_2\in \R: \vec x_2 - t_2 \geq 0\}$ is 
        exponentially small which allows us to control the contribution 
        of the region where $\x_2 \geq t_2$. 
        To bound the expectation of $p^4(\x_1)$, we use the following lemma
        known as Bonami-Beckner inequality or simply Gaussian hypercontractivity.
        \begin{lemma}[Gaussian Hypercontractivity]
            \label{lem:hypercontractivity}
            Let $p: \R \to \R$ be any polynomial of degree at most $\ell$.
            Then, for every $q \geq 2$, it holds
            \(
            \E_{x \sim \normal}[|p(x)|^q]
            \leq (q-1)^{q \ell/2} \Big(\E_{x \sim \normal}[p^2(x)]\Big)^{q/2} \,.
            \)
        \end{lemma}
        
        Using Cauchy-Schwarz and the fact that $\beta(\x)\leq 1$, we have 
        \begin{align}
            \label{eq:large-threshold-cauchy-schwarz}
            \E_{(\vec x_1, \vec \x_2) \sim \normal_2}[ 
            &\1\{\x_2 \geq t_2\} ( \vec x_2 - t_2 ) \beta(\vec x_1,  \vec x_2)  \sign(\vec x_1 - t^\ast)  p^2(\vec x_1) ]   
            \nonumber
            \\
            &\leq  \sqrt{\pr_{\x_2 \sim \normal}[\x_2 \geq t_2]}  
            \left(\E_{\vec x_1 \sim \normal }[(\vec x_2 - t_2)^4] \right)^{1/4}
            \left(\E_{\vec x_1 \sim \normal }[p^8(\x_1)]\right)^{1/4}
            \,.
        \end{align}
        Using Gaussian hypercontractivity, \Cref{lem:hypercontractivity},
        we obtain that $(\E_{\vec x_1 \sim \normal}[ p^8(\x_1)] )^{1/4} = e^{O(k)}$.
        Moreover, from standard bounds of Gaussian tails, \Cref{fct:gaussian-tails},
        we have that  $\pr_{\x_2 \sim \normal}[\vec \x_2 \geq t_2] \leq e^{-t_2^2/2}$.  
        Finally, we have that  
        $(\E_{\vec x_1 \sim \normal}[(\x_2 - t_2)^4])^{1/4} = (3 + 6 t_2^2 + t_2^4)^{1/4} \leq 2 t_2$, for all $t_2 \geq 1$.
        Combining the above bounds, we obtain that 
        \begin{align}
            \label{eq:large-threshold-tiny-probability-bound}
            \E_{(\vec x_1, \vec \x_2) \sim \normal_2}[ 
            \1\{\x_2 \geq t_2\} ( \vec x_2 - t_2 ) \beta(\vec x_1,  \vec x_2)  \sign(\vec x_1 - t^\ast)  p^2(\vec x_1) ]   
            \leq t_2 e^{-t_2^2/4} e^{O(k)} \,.
        \end{align}
        Recall that $k = O(\log(1/(\beta \gamma)))$, therefore there exists
        an absolute constant $\rho_2>0$ such that $e^{O(k)} \leq
        (1/(\beta\gamma))^{\rho_2}$.  Therefore, combining
        \cref{eq:large-threshold-orthogonal-bound-1} and
        \cref{eq:large-threshold-tiny-probability-bound}, we get that 
        \[
        \E_{(\vec x_1, \vec \x_2) \sim \normal_2}[ 
        ( \vec x_2 - t_2 )
        \beta(\vec x_1,  \vec x_2)  \sign(\vec x_1 - t^\ast)  p^2(\vec x_1) ]  
        \leq -Ct_2(\beta\gamma)^{\rho_1} + t_2 e^{-t_2^2/4} (\beta\gamma)^{\rho_2}\;.
        \]
        Hence, because $t_2\geq C\sloglog/2$, for $C$ sufficiently large positive absolute constant, 
        we have that
        $\E_{(\vec x_1, \vec \x_2) \sim \normal_2}[( \vec x_2 - t_2 ) \beta(\vec
        x_1,  \vec x_2)  \sign(\vec x_1 - t^\ast)  p^2(\vec x_1) ]  = -t_2~
        \poly(\beta \gamma) $.
    \end{proof}
    We now bound the contribution of the direction parallel to the optimal vector, i.e., $\vec e_1$.
    \begin{claim}
        \label{clm:parallel-bound}
        It holds
        \[
        \E_{(\x_1, \x_2) \sim \normal_2}
        [
        ( \vec x_1 - t_1 )
        \beta(\vec x_1,  \vec x_2)  \sign(\vec x_1 - t^\ast)  p^2(\vec x_1) ]  
        \leq - t_1 ~ \poly(\beta \gamma)\;.
        \]
    \end{claim}
    \begin{proof}
        Notice that we can marginalize out the direction $\vec e_2$ 
        in this case.  The analysis is similar to that of \Cref{lem:zero-case}.  
        However, we 
        now have to distinguish three different intervals for $\vec x_1$
        inside which  $(\vec x_1 - t_1) \sgn(\vec x_1 - t^\ast)$ has the same sign.
        We will crucially use the fact that $|t^\ast| = O(\sqrt{\log(1/\gamma)})$,
        since the Gaussian puts at least $\gamma$ mass on the positive side of the
        halfspace,  and that the constant $C$ is sufficiently large so that $t_1$ is at least a large 
        constant multiple of $t^\ast$.
        We have the region $A^+ = A_1^+ \cup A_2^+$, where 
        $A_1^{+} = \{\vec x_1 : \vec x_1 \leq t^\ast\}$
        and $A_2^{+} = \{\vec x_1 :  \vec x_1 \geq t_1\}$.
        In $A^+$ it holds $(\vec x_1 - t_1) \sgn(\vec x_1 - t^\ast) \geq 0$.
        We also define the disagreement region $A^- = \{\vec x_1 : t^\ast \leq \vec x_1 \leq t_1 \}$,
        where $(\vec x_1 - t_1) \sgn(\vec x_1 - t^\ast) \leq 0$.
        The worst case noise function puts $\beta(\x_1, \x_2) = \beta $ when $\x_1 \in A^-$
        and $\beta(\x_1, \x_2) = 1$ otherwise.  Thus, in order to show that $p$ is a certifying
        polynomial, we have to show that
        \begin{equation}
            \label[ineq]{eq:large-threshold-weight-ratio-bound}
            \frac{ \E_{\x_1 \sim \normal}[|\x_1-t_1| p^2(\x_1) \1\{\x_1 \in A^-\}] } 
            { \E_{\x_1 \sim \normal}[|\x_1 - t_1| p^2(\x_1) \1\{\x_1 \in A^+\}]  }
            \geq \frac{1}{\beta}
            \,.
        \end{equation}
        Similarly to the proof of \Cref{lem:taylor-exponential-polynomial}
        we will first use an exponential function to shift the mean of the
        Gaussian, i.e., use $e^{2 c \x_1}$ instead of $p^2(\x_1)$ in the
        ratio of \Cref{eq:large-threshold-weight-ratio-bound}.
        Our goal is to show that there exists some $c\geq t^\ast$
        such that the probability of the region $A^{-}$ is much
        larger than the probability of $A^{+}$.  Since 
        $|t_1|$ is larger than some sufficiently large constant multiple of $\sloglog$
        the interval $[t^\ast, t_1]$ will contain most of the mass of the shifted normal
        $\normal(c, 1)$ for $c$ some constant multiple of $\sloglog$, see 
        \Cref{fig:gaussian-shift}.

        We start by bounding from above the denominator of \Cref{eq:large-threshold-weight-ratio-bound}.
        \begin{align}
            \label{eq:denominator-upper-bound}
            \E_{\x_1 \sim \normal}[|\x_1-t_1| e^{2 c \x_1} \1\{\x_1 \in A^+\}]
            &\leq \| \x_1 -t_1\|_2  \sqrt{ \E_{\x_1 \sim \normal}[e^{4 c \x_1} \1\{\x_1 \in A^+\}] }
            \nonumber
            \\
            &\leq 2 t_1 
            e^{4 c^2}
            \left(
            \sqrt{\pr_{\x_1 \sim \normal}[ \x_1 \leq t^\ast - 4c] }
            +
            \sqrt{\pr_{\x_1 \sim \normal}[ \x_1 \geq t_1 - 4c] }
            \right)
            \nonumber
            \\
            &\leq 2 t_1 e^{4 c^2} (e^{-(t^\ast - 4c)^2/4} + e^{- (t_1 - 4 c)^2)/4} )
            \nonumber
            \\
            &\leq t_1 e^{4 c^2} \beta/16
            \,,
        \end{align}
        where we used that $\| \x_1 - t_1\|_2 = 
        (\E_{\vec x_1 \sim \normal}[(\x_1 - t_1)^2]^{1/2} = 
        (1 + t_1^2)^{1/2}\leq 2 t_1$ for all $t_1 \geq 1$
        and, for the last inequality we used the Gaussian tail upper bound, 
        see \Cref{fct:gaussian-tails}, and the fact that $c = \Theta(\sloglog)$, 
        i.e.,
        we get the constant $1/16$ by appropriately choosing some $c = \Theta(\sloglog)$.
        We next bound the numerator from below. We have 
        \begin{align}
            \label{eq:denominator-lower-bound}
            \E_{\x_1 \sim \normal}[|\x_1-t_1| e^{2 c \x_1} \1\{\x_1 \in A^-\}]
            &\geq
            \frac{t_1}{2} 
            \E_{\x_1 \sim \normal}[e^{2 c \x_1} \1\{t^\ast \leq \x_1 \leq t_1/2\}]
            \nonumber
\nonumber
            \\
            &=   \frac{t_1}{2} e^{4 c^2}
            \pr_{\x_1 \sim \normal}[ t^\ast \leq \x_1 + 2 c \leq t_1/2]
            \geq
            \frac{t_1 e^{4 c^2}}{4} 
            \,,
        \end{align}
        where for the last inequality we used the fact that 
        $
        \pr_{\x_1 \sim \normal}[ t^\ast \leq \x_1 + 2c \leq t_1/2] 
        $
        is at least $1/2$ 
        since $t_1/2 - t^\ast$ is at least a large absolute constant
        and $c = \Theta(\sloglog)$. 
        In particular, both $c-t^\ast$ and $t_1/2 - c$ are at least large absolute constants.
        See \Cref{fig:large-hypotheses-1}.
        Thus, it holds that
        \[
        \frac{ \E_{\x_1 \sim \normal}[|\x_1-t_1|  \1\{\x_1 \in A^-\} e^{2 c \x_1} ] } 
        { \E_{\x_1 \sim \normal}[|\x_1 - t_1|  \1\{\x_1 \in A^+\} e^{2 c \x_1} ]  }
        \geq \frac{4}{\beta} \,.
        \]
        Similarly to the proof of \Cref{lem:taylor-exponential-polynomial},
        we now replace the exponential function $e^{2 c x}$ by the
        square of the Taylor expansion of $e^{c x}$, i.e., $S_k^2(c x)$.  It suffices
        to bound the following ratios
        \begin{align}
            \label{eq:exponential-weighted-polynomial-ratios}
            \frac{
                \E_{\x_1 \sim \normal}[e^{2 c \x_1} |\x_1 - t_1| \1\{\x_1 \in A^+\}]
            }
            {
                \E_{\x_1 \sim \normal}[(S_k(c \x_1))^2 |\x_1 - t_1| \1\{ \x_1 \in A^+\}] 
            }
            \geq \frac{1}{2}
            ~
            \text{ and }
            ~
            \frac{
                \E_{\x_1 \sim \normal}[ (S_k(c \x_1))^2  |\x_1 - t_1| \1\{\x_1 \in A^-\}]
            }
            {
                \E_{\x_1 \sim \normal}[ e^{2 c \x_1}  |\x_1 - t_1| \1\{ \x_1 \in A^-\}] 
            }
            \geq \frac{1}{2}
            \,.
        \end{align} 
       We start by showing the first inequality of \Cref{eq:exponential-weighted-polynomial-ratios}.
        It suffices to show that 
        
        \[
        \E_{\x_1 \sim \normal}[|e^{2 c \x_1} - S_k^2(c \x_1)| |\x_1 - t_1| \1\{\x_1 \in A^+\}]
        \leq  \E_{\x_1 \sim \normal}[e^{2 c \x_1} |\x_1 - t_1| \1\{ \x_1 \in A^+\}] 
        \,.
        \]
        Using the $L_1$ approximation error bound for the Taylor polynomial
        of \Cref{clm:ell-one-approximation-taylor} and H{\"o}lder's inequality 
        we can get an $L_2$ approximation guarantee.
        Assuming that $k$ is some sufficiently large absolute constant multiple of $c^2$, i.e., $k = C' c^2$, it holds
        \begin{align}
            \label{eq:ell-two-holder}
            \| e^{2 c \x_1} - S_k^2(c \x_1)\|_2
            &\leq 
            \| e^{2 c \x_1} - S_k^2(c \x_1)\|_1^{1/3}    
            \| e^{2 c \x_1} - S_k^2(c \x_1)\|_4^{2/3}
            \nonumber
            \\
            &\leq  e^{-C' k} e^{O(c^2)} 
            \leq e^{-\Omega(k)}\;,
        \end{align}
        where the last inequality follows from the fact
        that $S_k^2(c \x_1) \leq e^{2 c |\x_1|}$, $\| e^{2 c |\x_1| }\|_4 \leq e^{O(c^2)}$ and $C'$ is sufficiently large absolute constant.
        We can now use Cauchy-Schwarz and the $L_2$ approximation guarantee of $S_k^2(c \x_1)$
        to obtain
        \begin{align}
            \label[ineq]{eq:ell-two-approximation}
            \E_{\x_1 \sim \normal}[&|e^{2 c \x_1} - S_k^2(c \x_1)| |\x_1 - t_1| \1\{\x_1 \in A^+\}]
            \nonumber
            \\
            &\leq  
            \| \x_1 - t_1\|_2 \| e^{2 c \x_1} - S_k^2(c \x_1)\|_2
            \nonumber
            \\
            &\leq 2 t_1 e^{-\Omega(k)} 
            \,,
        \end{align}
        where we used that 
        $\|\x_1 - t_1\|_2 \leq 2 t_1$ for $t_1\geq 1$.
        We have that 
        \begin{align}
            \label[ineq]{eq:A-plus-lower-bound}
            \E_{\x_1 \sim \normal}[e^{2 c \x_1} |\x_1 - t_1| \1\{\x_1 \in A^+\}]
            &\geq 
            \E_{\x_1 \sim \normal}[e^{2 c \x_1} |\x_1 - t_1| \1\{\x_1 \leq t^\ast\}]
            \nonumber
            \\
            &\geq (t_1 - t^\ast) e^{-(2 c - t^\ast)^2}/4
            \geq t_1 e^{-4 c^2}/8
            \,,
            & &
            \text{$\rhd~ t^\ast \leq c$ and $t^\ast \leq t_1/2$}
        \end{align}
        where for the second inequality we used the Gaussian tail lower bounds of \Cref{fct:gaussian-tails}.
        Combining \Cref{eq:ell-two-approximation} and
        \Cref{eq:A-plus-lower-bound}, we obtain that for $k$ larger than some
        constant multiple of $c^2$ the first inequality of
        \Cref{eq:exponential-weighted-polynomial-ratios}
        holds.
        The proof of the second inequality of \Cref{eq:exponential-weighted-polynomial-ratios} is similar.
        We obtain that \Cref{eq:large-threshold-weight-ratio-bound} holds for the polynomial 
        $S_k^2(c \x_1)$ for $k = \Theta(1/\log(1/(\beta \gamma) ) )$ 
        and $c = \Theta(1/\sqrt{\log(1/(\beta \gamma) ) })$.
        
        Finally, using \Cref{fct:taylor-expansion} 
        we have that $\|S_k^2(c \x_1)\|_2 \leq 
        \|e^{2 c |\x_1|}\|_2 = e^{O(c^2)}$.
        Using the estimate of \Cref{eq:denominator-lower-bound} 
        and the $L_2$ approximation guarantee of 
        \Cref{eq:ell-two-holder} we obtain that 
        \[
        \frac{\E_{\x_1 \sim \normal}[S_k^2(c \x_1) |\x_1 - t_1| \1\{\x_1 \in A^-\}]}
        {\|S_k^2(c \x_1)\|_2}
        = \frac{e^{-\Omega(c^2) } - e^{-\Omega(k)}}{e^{O(c^2)}}
        = e^{-O(c^2)} = \poly(\gamma \beta) \,,
        \]
        where we used the fact that $k$ is a sufficiently large
        constant multiple of $c^2$ and $c^2 = \Theta(\log(1/(\beta \gamma)) )$.
        We conclude that 
        \begin{align*}
            \E_{(\x_1, \x_2) \sim \normal_2}
            \Big[
            &( \vec x_1 - t_1 )
            \beta(\vec x_1,  \vec x_2)  \sign(\vec x_1 - t^\ast)  
            \frac{S_k^2(c \vec x_1)}{\|S_k^2(c \vec x_1)\|_2 } \Big]  
            \\
            &\leq 
            \beta \frac{\E_{\x_1 \sim \normal}[S_k^2(c \x_1) |\x_1 - t_1| \1\{\x_1 \in A^-\}]}
            {\|S_k^2(c \x_1)\|_2}
            - 
            \frac{\E_{\x_1 \sim \normal}[S_k^2(c \x_1) |\x_1 - t_1| \1\{\x_1 \in A^+\}]}
            {\|S_k^2(c \x_1)\|_2}
            \\
            &\leq - t_1 \poly(\beta \gamma) \,.
        \end{align*}
        The proof of the upper bound on the coefficients of the polynomial $S_k^2(c \x_1) /\| S_k^2(c \x_1)\|_2$
        is similar to that of the constant hypothesis case, see \Cref{lem:zero-case}.
        
    \end{proof}
    This completes the proof of \Cref{lem:large-threshold-hypotheses}.
\end{proof}

\subsubsection{Certificate Against ``Small Threshold'' Halfspaces}
\label[sub]{sub:small-threshold}
\usetikzlibrary{calc,patterns,angles,quotes}

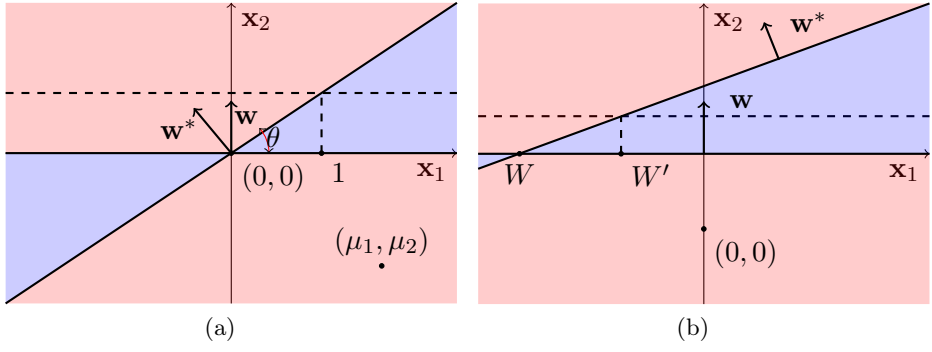
\begin{figure}[ht]
    \centering
    \subfloat[
    ]{
        \label{fig:small-threshold-1}
        \begin{tikzpicture}[scale=1]
            \coordinate (start) at (1,0);
            \coordinate (center) at (0,0);
            \coordinate (end) at (0.9,0.8);
  \draw[->] (-3,0) -- (3,0) node[anchor=north west,black,below left] {$\x_1$};
            \draw[->] (0,-2) -- (0,2) node[anchor=south east,below right] {$\x_2$};
\draw[fill=red,opacity=0.2,draw=none] (-3,-2)--(0,0)--(3,0)--(3,-2);
             \draw[fill=red,opacity=0.2,draw=none] (3,2)--(0,0)--(-3,0)--(-3,2);
                      \draw[fill=blue,opacity=0.2,draw=none] (0,0)--(3,0)--(3,2);
                      \draw[fill=blue,opacity=0.2,draw=none] (0,0)--(-3,0)--(-3,-2);
            \draw[black,thick] (-3,-2) -- (3,2);
            \draw[thick ,->] (0,0) -- (0,0.7) node[right=2mm,below] {$\vec w$};
            \draw[black,thick] (-3,0) -- (3,0);
            \draw[thick ,->] (0,0) -- (-0.5,0.6) node[left=2mm,below] {$\vec w^\ast$};
              \draw[fill=black] (2,-1.5) circle (0.03) node[above] {$(\mu_1,\mu_2)$};
              		\draw[black,dashed, thick](-3,0.8) -- (3,0.8);
              \draw[fill=black] (1.2,0) circle (0.03) node[below right] {$1$};
               \draw[black,dashed, thick](1.2,0) -- (1.2,0.8);
                      \pic [draw=red, <->, "$\theta$", angle eccentricity=1.2] {angle = start--center--end};
                                    \draw[fill=black] (0,0) circle (0.03) node[below right] {$(0,0)$};
    \end{tikzpicture}}
    \centering
    \subfloat[
    ]{ \label{fig:small-threshold-2}
        \centering
        \begin{tikzpicture}[scale=1]
              \draw[black,dashed, thick](-3,0.5) -- (3,0.5);
            \coordinate (start) at (0.5,0);
            \coordinate (center) at (0,0);
            \coordinate (end) at (0.5,0.5);
  \draw[->] (-3,0) -- (3,0) node[anchor=north west,black,below left] {$\x_1$};
            \draw[->] (0,-2) -- (0,2) node[anchor=south east,below right] {$\x_2$};
\draw[fill=red,opacity=0.2,draw=none] (-3,0)--(-2.45,0)--(3,2)--(-3,2);
             \draw[fill=red,opacity=0.2,draw=none] (-3,-0.2)--(-2.45,0)--(3,0)--(3,-2)--(-3,-2);
             \draw[fill=blue,opacity=0.2,draw=none] (-3,0)--(-2.45,0)--(-3,-0.2);
             \draw[fill=blue,opacity=0.2,draw=none] (-2.45,0)--(3,2)--(3,0);
\draw[black,thick] (-3,-0.2) -- (3,2);
            \draw[thick ,->] (0,0) -- (0,0.7) node[right=2mm] {$\vec w$};
            \draw[black,thick] (-3,0) -- (3,0);
            \draw[thick ,->] (1,1.25) -- (0.8,1.75) node[right=2mm] {$\vec w^\ast$};
            \draw[fill=black] (-2.45,0) circle (0.03) node[below] {$W$};
              \draw[fill=black] (-1.1,0) circle (0.03) node[below right] {$W'$};
               \draw[black,dashed, thick](-1.1,0) -- (-1.1,0.5);
              \draw[fill=black] (0,-1) circle (0.03) node[below right] {$(0,0)$};
    \end{tikzpicture}}
    \caption{
        \textbf{(a)}
        Our certificate when the angle $\theta$ between the two halfspaces is large.
        Notice that we have changed coordinates so that the point $(0,0)$ 
        is their crossing point.  The mean of the Gaussian is now moved to the point
        $(\mu_1, \mu_2) = (-t/\tan(\theta) + t^\ast/\sin \theta, -t)$.  Notice that 
        since in this case we have assumed that $|t| \leq C\sloglog$, 
        the $\vec e_2$-coordinate of the mean of the Gaussian
        cannot be very far from the origin.
        In this case, we overestimate the agreement region between the two halfspaces (red area) 
        considering to be $\{\x_1 \leq 1\}$.  
        In order to put more weight to the disagreement area, we again use the 
        polynomial shift (see \Cref{lem:taylor-exponential-polynomial}) 
        in the direction $\vec e_1 = \wperp$. Observe that this differs from the cases 
        of \Cref{sub:constant-hypotheses} and \Cref{sub:large-threshold},
        where the polynomial shift was along the direction of $\vec w^\ast$.
        \\
        \textbf{(b)} 
        The case where the angle $\theta$ between the two halfspaces is small, 
        $\theta= O(\eps^2 \gamma \beta)$. In that case, the crossing point $W$
        of the two halfspaces is very far from the origin, 
        $|W| \geq \Omega(1/(\eps \gamma \beta))$.  In this case, simply taking 
        a band around $\ell(\x)$ works as a certificate.
    }
\end{figure}
\begin{lemma}[Certificate against ``Small Threshold'' Hypotheses] 
    \label{lem:small-threshold}
    Let $\D$ be a distribution on $\R^d \times \{\pm 1\}$ with standard normal $\x$-marginal.  Assume that $\D$ satisfies the $\eta$-Massart noise
    condition with respect to some at most $(1-\gamma)$-biased optimal halfspace.
    Define the linear function $\ell(\x) = \sgn(\vec w \cdot \x - t)$
    and assume that $t/\snorm{2}{\vec w} \leq C \sloglog $ for some absolute constant $C>0$.
    Moreover, assume that $\pr_{(\x, y) \sim \D}[ \sign(\ell(\x)) \neq y ] \geq \opt + \eps$.
    Then, there exists polynomial $q(\x) = \sum_{|\alpha| \leq k} c_\alpha \x^\alpha$ 
    of degree $\Theta(\log(\frac 1 {\beta}))$, norm $\|q(\x)\|_2 = 1$, and 
    sum of (absolute) coefficients $\sum_{|\alpha| \leq k} |c_\alpha| \leq d^{O(k)}$,
    and $r_1, r_2 \in \R$ with $|r_1 - r_2| \geq \poly(\eps \beta \gamma)$
    such that
    \[
   \E_{{(\x, y) \sim \D} }[ \ell(\x) y \1\{r_1 \leq \ell(\x) \leq r_2 \} q^2(\x) ]
    \leq - \eps^4 \poly(\beta \gamma) ~ \|\ell(\x)\|_2
    \,.
    \]
\end{lemma}
\begin{proof}
    Denote $\ell^\ast(\x) = \vec w^\ast \cdot \x - t^\ast$ the optimal halfspace
    and denote $\ell(\x) = \vec w \cdot \x - t$ for some vectors $\vec w \in \R^d$ 
    and some threshold $t \in \R$.
    Moreover, denote $\theta = \theta(\vec w^\ast, \vec w)$ the angle between
    $\vec w^\ast$ and $\vec w$.
    We first observe that since $\opt = \pr_{(\x, y) \sim \D}[\sgn(\ell^*(\x) \neq y ]$
    it holds that 
    \begin{align*}
        \pr_{(\x,y) \sim \D}[ \sgn(\ell(\x)) \neq y ]
        - \opt 
        = \E_{\x \sim \D_\x}[ \1\{ \sgn(\ell(\x)) \neq \sgn(\ell^\ast(\x)) \} \beta(\x) ]\;.
    \end{align*}
    Thus, the disagreement probability between $\ell(\x)$ and $\ell(\x^\ast)$ 
    $ \E_{\x\sim \D_{\x}}[ \1\{ \sgn(\ell(\x)) \neq \sgn(\ell^\ast(\x)) \} ] \geq \eps$
    from our assumption that 
    $ \pr_{(\x, y) \sim \D}[ \sign(\ell(\x)) \neq y ] \geq \opt + \eps $.
    From \Cref{fct:gaussian-halfspaces} we know that 
    \[
         \pr_{\x \sim \D_x}[ \sign(\ell(\x)) \neq 
         \sign(\ell^\ast(\x)) ] 
         \leq O(\theta) + O(|t^\ast - t|) \,.
         \]
         Thus, we either have that the angle 
         of the two halfspaces is large $\theta = \Omega(\eps)$ or 
         the difference of their thresholds $|t^\ast - t| = \Omega(\eps)$.
    
    Since we can always write
    \begin{align*}
        \E_{(\x, y) \sim \D}[ &(\vec w \cdot \x - t) y \ \1\{r_1 \leq \vec \ell(\x) \leq r_2\} q^2(\x) ]
        \\
        &= \|\vec w\|_2
        \E_{(\x, y) \sim \D} \Big[
        \left(\frac{\vec w}{\|\vec w\|_2} \cdot \x - \frac{t}{\|\vec w\|_2}\right) y \ \1\{r_1 \leq \ell(\x) \leq r_2\} q^2(\x)
        \Big] \,,
    \end{align*}
    we will assume for simplicity that $\|\vec w\|_2 = 1$ and $|t| 
    \leq C \sloglog $.  
    We are going to construct a polynomial $q(\x)$ that only depends on 
    the subspace $V$ spanned by $\vec w, \vec w^\ast$.
    Therefore, as in the case of \Cref{lem:zero-case}, we can project $\x$ to the subspace
    $V$ spanned by $\vec w, \vec w^\ast$ and preserve the $\eta$-Massart noise assumption,
    see \Cref{clm:projections-preserve-massart-noise}.
    Let $\theta = \theta(\vec w, \vec w^\ast)$ be the angle between $\vec w, \vec w^\ast$.
    Without loss of generality, we may assume that $\vec w = \vec e_2$
    and $\vec w^\ast = -\sin \theta \vec e_1 + \cos \theta \vec e_2 $, see
    \Cref{fig:small-threshold-1}.
    Therefore, in this case we have that $\ell(\x) = \x_2 - t$ and 
    $\ell^\ast(\x) = - \sin \theta \vec e_1 + \cos \theta \vec e_2 - t^\ast$.
    
    We first provide a certificate that works when the angle $\theta$ is non-trivial, that is when $\theta\geq C'\eps^2\gamma\beta$, for some small enough constant $C'$ chosen appropriately.
We change the coordinates so that the crossing point of the two linear functions 
    $\ell, \ell^\ast$ is the origin $(0,0)$.  
    This will move the mean of the Gaussian to the point
    $(\mu_1, \mu_2) = (-t/\tan \theta + t^\ast/\sin \theta, -t)$, see \Cref{fig:small-threshold-1}.
    
    Assume first that $C\eps^2\gamma\beta=\theta \leq \pi/2$ and $\mu_1 \geq 0$, we show later that the other cases are symmetric.
    Set $r_1 = 0$, $r_2 = \tan(\theta) = \Theta(\theta)$, and let $q(\x)$ depend only on $\x_1$, we have that
    \begin{align}
        \label{eq:small-threshold-1}
        \E_{(\x, y) \sim \D}[ &\ell(\x) y \
        \1\{r_1 \leq \ell(\x) \leq r_2\} q^2(\x) ]
        \nonumber
        \\
        &=
        \E_{(\x_1, \x_2) \sim \normal(\mu_1, \mu_2)}[\vec x_2 \sign(-\sin \theta \vec x_1 +
        \cos \theta \vec x_2) \ \beta(\x_1, \x_2)
        \1\{0 \leq \vec \x_2 \leq r_2\} q^2(\x_1) ]
        \nonumber
        \\
        &\leq\E_{(\x_1,\x_2) \sim \normal(\mu_1,\mu_2)}[ \x_2 \sign( 1 - \vec x_1  ) \ \beta(\x_1, \x_2) \
        \1\{0 \leq \vec \x_2 \leq r_2\} 
        q^2(\x_1)    ] \,,
    \end{align}
    where for the inequality we used the fact that $ 0 \leq \vec x_2 \leq r_2 = \tan(\theta)$,
    and therefore it holds that if $-\sin \theta \x_1 + \cos \theta \x_2 \geq 0$ 
    then it holds that $1-\x_1 \geq 0$, for all $(\x_1, \x_2)$ such that $\x_2 \in [0,r_2]$.
    Observe now that in order to maximize the quantity of \Cref{eq:small-threshold-1}
    the ``worst-case'' noise function $\beta(\x_1, \x_2)$ is equal to $\beta$ for 
    all points where the integral is negative, and $+1$ when the integral is positive,
    that is $\beta(\x_1, \x_2) = \beta \1\{ \vec x_1 \geq 1\} + \1\{\vec x_1 < 1\}$.
    Since this ``worst-case'' noise $\beta(\x_1, \x_2)$ is independent of 
    $\x_2$, it follows that we can decompose the expectation of \Cref{eq:small-threshold-1}, i.e.,
\[
    \E_{\x_1 \sim \normal(\mu_1)}[ q^2(\x_1) \sign(1 - \vec x_1 ) \beta(\x_1)] \
    \E_{\x_2 \sim \normal(\mu_2)}[\vec x_2 \1\{0 \leq \vec \x_2 \leq r_2\} ] \,.
    \]
    Since $|\mu_2| = |t| \leq C \sloglog$
    from standards bounds of the Gaussian tail probability (\Cref{fct:gaussian-tails}), we obtain 
    \[
    \E_{\x_2 \sim \normal(\mu_2)}[\vec x_2 \1\{0 \leq \vec \x_2 \leq r_2\} ]
    \geq \theta^2 \poly(\beta \gamma) \,.
    \]
    It remains to bound the term
    $ \E_{\x_1 \sim \normal(\mu_1)}[ q^2(\x_1) \sign(1-\vec x_1) \beta(\x_1) ] $.
    We observe that the ``worst-case'' value of $\mu_1 \geq 0$ in order to
    maximize this expectation is $\mu_1 = 0$.
    Now, we can use the same argument as in the proof of \Cref{lem:zero-case};
    notice that in this case the threshold is $1$ instead of $\sqrt{\log(1/\gamma)}$ 
    and therefore, by picking $k = \Theta(\log(1/(\beta)))$
    we have that  there exists polynomial $q$ 
    such that $\E_{\x_1 \sim \normal(\mu_1)}[ q^2(\x_1) \sign(1-\vec x_1) \beta(\x_1)] 
    \leq - \poly(\beta \gamma) $
    and thus, we obtain the bound:
    \[
    \E_{(\x, y) \sim \D}[ \ell(\x) y \
    \1\{r_1 \leq \vec w \cdot \x \leq r_2\} q^2(\x) ]
    \leq - \theta^2 \poly(\beta\gamma)
    \|\ell(\x)\|_2
    \,.\]

    In the case where $\theta \in [0, \pi/2]$ and
    $\mu_1 < 0$, we may pick $r_1 = - \tan(\theta)$ and $r_2 = 0$
    resulting in a completely symmetric case to the previous one.
    Finally, the case where $\theta \in [\pi/2, \pi]$ is easier than
    the previous two cases since the disagreement region between the two halfspaces is
    now a superset of the corresponding region in the previous cases.
    
    We now handle the case where the angle between the two halfspaces is
    small, i.e., $\theta \leq C'\eps^2 \beta \gamma$.  
    We know that the disagreement between two halfspaces is upper bounded by their angle, i.e.,
    $\E_{\x \sim \D_\x}[h(\x) \neq f(\x)] \leq O(\theta) + O(|t - t^*|)$.
    Since $\E_{\x \sim \D_\x}[h(\x) \neq f(\x)] \geq \eps$, we obtain that
    the two thresholds $t, t^\ast$ cannot be very close, i.e.,
    $|t - t^\ast| = \Omega(\eps)$.
    As in the previous case, we may assume that 
    $\vec w = \vec e_2$ and $\vec w^\ast = -\sin \theta \vec
    e_1 + \cos \theta \vec e_2$.  Recall that the intersection point of the
    two halfspaces has coordinates $(t/\tan \theta - t^\ast/\sin \theta, t)$.
    This means that when $\theta = O(\eps^2 \gamma \beta)$ the intersection point 
    of the two halfspaces is very far from the origin: 
    its first coordinate is $|t/\tan \theta - t^\ast/\sin \theta| = \Omega(1/(\eps \gamma \beta))$,
    since by the triangle inequality, it follows 
    \begin{align*}
    |t \cos \theta - t^\ast|
    &\geq |t- t^\ast| \cos \theta - |t^\ast| |\cos \theta - 1|
    \\&\geq |t- t^\ast| (1-\eps^4 \gamma^2 \beta^2) - O(\sqrt{\log(1/(\beta\gamma))} ) (\eps^4 \gamma^2 \beta^2 -2\eps^8\gamma^4\beta^4)
    = \Omega(\eps)
    \,,
    \end{align*}
    where we used the inequality $1-x^2\leq \cos x \leq 1- x^2/2 +x^4$, for all $x \in [0, \pi/2]$ and the fact that $\eps\leq 1/2$. Combining, the above with $\sin(\theta)= O(\eps^2\beta\gamma)$, we get that $|t/\tan \theta - t^\ast/\sin \theta| = \Omega(1/(\eps \gamma^2 \beta^2))$.
In this case, we do not require a polynomial for the certificate, therefore the certificate is going to be simply a band,
    see \Cref{fig:small-threshold-2}.
    Let $W$ be the $\vec e_1$-coordinate of the crossing point of the optimal
    halfspace with $\ell(\x)$, see \Cref{fig:small-threshold-2}.  We have that
    $|W|= \Omega(1/(\eps \beta \gamma))$.
    Without loss of generality, we may assume that $W < 0$ along with the band $
    \eps^2/2 \leq \ell(\x)\leq \eps^2$ since the proof of the other case is
    similar:
    we just need to consider the band $-\eps^2 \leq \ell(\x) \leq -\eps^2/2$
    instead, see \Cref{fig:small-threshold-2}.  
    Moreover, denote $W'$ the $\vec e_1$-coordinate of the point
    of $\ell^*(x) = \ell(\x) + \eps^2$, see \Cref{fig:small-threshold-2}.
    It holds $W'=W+\eps^2/\tan\theta$, and note that
    $|W'|=\Omega(1/(\eps\beta\gamma))$ since  $ |W+\eps^2/\tan\theta|\geq  (|t
    \cos \theta - t^\ast|-|\eps^2|)/\sin\theta \geq  \Omega(1/(\eps \beta
    \gamma))$.
    We have 
    \begin{align}
        \label{eq:small-threshold-2}
        \E_{(\x_1, \x_2) \sim \normal_2}[ &\ell(\x) \sgn(-\x_1 \sin \theta + \x_2 \cos \theta - t^\ast) \1\{\eps^2/2 \leq \ell(\x) \leq \eps^2\} \beta(\x_1, \x_2)]
        \nonumber
        \\
        &\leq 
        \E_{(\x_1, \x_2) \sim \normal_2}[ \ell(\x) \sign(-\x_1 + W +\eps^2/\tan\theta) \1\{\eps^2/2 \leq \ell(\x) \leq \eps^2\} \beta(\x_1, \x_2)] \;,
    \end{align}
    where we overestimated the contribution of the agreement area
    (red region in \Cref{fig:small-threshold-2}) by the region $\x_1 \leq W'$.
    Since $W'$ is still very far from the origin the contribution of the 
    region $\x_1 \leq W'$ is going to be very small.
To bound the quantity of \Cref{eq:small-threshold-2}, we first bound from
     above the region where $\sign(-\x_1+W')$ is positive:
    \begin{align}\label{eq:small-threshold-positive}
         \E_{(\x_1,\x_2) \sim \normal_2}[ \ell(\x)  \1\{\eps^2/2 \leq \ell(\x)
         \leq \eps^2\}  \beta(\x)  \1\{\x_1 \leq W'\}]
         &\leq   \pr_{\x_1 \sim \normal}[\x_1 \leq W'] \leq e^{-1/(\eps\beta\gamma)^2}  \;,
    \end{align}
    where we used the fact that $|W'|=\Omega(1/(\eps\gamma\beta))$. Next, we bound from below the region where $\sign(-\x_1+W')$ is negative:
\begin{align}\label{eq:small-threshold-negative}
         \E_{(\x_1,\x_2) \sim \normal_2}[ \ell(\x) & \1\{\eps^2/2 \leq \ell(\x) \leq \eps^2\}  \beta(\x)  \1\{\x_1 \geq W'\}]
        \nonumber\\ &\geq  \beta   \E_{\x_2 \sim \normal}[ \ell(\x_2)  \1\{\eps^2/2 \leq \ell(\x_2) \leq \eps^2\}] \pr_{\x_1 \sim \normal}[\x_1 \geq W']\nonumber \\&\geq  \frac{\eps^2}{4}\beta \pr_{\x_2 \sim \normal}[ \eps^2/2 \leq \ell(\x_2) \leq \eps^2]\geq \eps^4\beta \poly(\beta\gamma)\;,
    \end{align}
    where we used that $ \pr_{\x_1 \sim \normal}[\x_1 \geq W']\geq 1/2$ and that $\beta(\x)\geq \beta$.
    Using \Cref{eq:small-threshold-positive} and \Cref{eq:small-threshold-negative} to \Cref{eq:small-threshold-2}, we get
\[
   \E_{(\x_1, \x_2) \sim \normal_2}[ \ell(\x) \sgn(-\x_1 \sin \theta + \x_2 \cos \theta - t^\ast) \1\{\eps^2/2 \leq \ell(\x) \leq \eps^2\} \beta(\x_1, \x_2)] \leq -\eps^4 \poly(\beta \gamma)\;.
    \]
    We combine the above cases to obtain the claimed bound.
\end{proof}

\subsection{Efficiently Computing the Certificate via SDP}
\label[sub]{sub:certificate_optimization}
In this section, we show that we can efficiently compute our polynomial
certificate given labeled examples from the target distribution.
The following is the main proposition of this subsection,
where we bound the number of samples and the runtime needed
to compute the certificate given samples from the distribution $\D$.
The proof is similar to that given in \cite{DKTZ20b}; we adapt it to work
for our certifying function and for general (as opposed to homogeneous) halfspaces.
\begin{proposition}[Certificate Oracle]
    \label{prop:sdp-oracle}
    Let $\D$ be a distribution on $\R^d \times \{\pm 1\}$ with standard normal $\x$-marginal.  Assume that $\D$ satisfies the $\eta$-Massart noise
    condition with respect to some (at least) $(1-\gamma)$-biased optimal halfspace $f(\x)$.
    Let $\ell(\x)$ be  any linear function such that
    $\pr_{\x \sim \D_\x}[\sgn(\ell(\x)) \neq f(\x)] \geq \eps$, for $\eps\in(0,1)$.
    There exists an algorithm that draws $N =d^{O(\log({1}/{(\beta \gamma)}))} \log(1/\delta)/\eps^2$ samples from $\D$, runs in time $\poly(N,d)$, and
    with probability $1-\delta$
    returns a positive function $T(\x)$ with $\|T(\x)\|_{4}\leq 1$ such that 
    $$  \E_{(\bx, y) \sim \D}\left[ T(\x) \ell(\x) y  \right] \leq-  \eps^4  
    d^{-O(\log(1/(\beta \gamma))) }
    \| \ell(\x) \|_2 
    \,.$$
\end{proposition}
\begin{proof}
    From \Cref{pro:certificate-Biased-Massart}, we know that we are looking
    for a certificate function $T(\x)$ of the form $T(\x)=\1_B(\x) q^2(\x)$, where $B$ is a band
    with respect $\ell(\x)$, i.e., $B=\{r_1\leq \ell(\x)\leq r_2\}$ for some
    $r_1,r_2\in \R\cup\{\infty\}$ and $q(\x)$ is a 
    $k=\Theta(\log({1}/{\gamma \beta}))$ degree polynomial.
    We illustrate how we can formulate the search of such function as an SDP.
    For the rest of this section, let $\1_B(\vec x)$ be the indicator function
    of the region $B=\{\x:r_1\leq \ell(\x)\leq r_2\}$, for some appropriate
    choices $r_1,r_2\in \R\cup\{\infty\}$ and $\lambda=\eps^4 d^{- O(k) } \| \ell(\x) \|_2$.
    Denote by $\vec m(\bx)$ the vector containing all monomials up to degree
    $k$, such that $\vec m_S(\bx)\eqdef\x^S$, indexed by the multi-index $S$
    satisfying $|S|\leq k$.  Recall that if $S=(s_1,\ldots,s_d)$, then
    $\x^S=\prod_{i=1}^d\x_i^{s_i}$.
    The dimension of $\vec m(\bx) \in \R^m$ is $m = \binom{d+k}{k}$. 
    Let $ \vec M =  \E_{(\bx, y) \sim \D}\left[ \vec m(\bx) \vec m(\bx)^\top 
    \1_B(\bx) \ell(\x) y  \right]$, for any real matrix $\vec A\in\R^{m \times m}$,
    we define the following function
    \begin{align}
        \label{eq:certificate_objective}
        \mathcal{L}_{\ell} (\vec{A}) =
        \E_{(\bx, y) \sim \D}\left[ \vec m(\bx)^\top \vec A ~
        \vec m(\bx) \1_B(\bx) \ell(\x) y  \right]
        = \tr\left(\vec A \vec M \right)\;.
    \end{align}
    Notice that $\calL_{\ell}$ is linear in its variable $\vec A$. 
    From \Cref{pro:certificate-Biased-Massart}, we know that
    if $\pr_{\x \sim \D_\x}[\sgn(\ell(\x)) \neq f(\x)]\geq \eps$,
    then there exists a (normalized) polynomial $q(\vec x) = \vec m(\vec x)\cdot \vec a$, with $\|\vec a\|_1 = 1$
    such that $\E_{(\x, y)}[ \ell(\x) y \1_B(\x) q(\vec x)] \leq - \lambda $.
    Therefore, for $\vec A=\vec a \vec a^\top$, we have $q^2(\x)=\vec m(\x)^\top \vec A ~\vec m(\x) $ and hence, $\calL_{\ell}(\vec A) \leq -\lambda$.
    It follows that there exists a positive semi-definite rank-$1$ matrix $\vec
    A$ such that $\mathcal{L}_{\ell} (\vec A) \leq-\lambda$. 
    Moreover, $\snorm{F}{\vec A}^2 = \snorm{F}{\vec a \vec a^T} = \|\vec a\|_2^4 \leq 1$.
    Recall that, $\mathcal S^m$ is the set of (symmetric) positive semi-definite matrices of
    $m$-dimension.  We formulate the following semi-definite program
\begin{alignat}{2}
        \text{Find }    \qquad  & \vec A \in \mathcal{S}^m                       && \notag\\ 
        \text{s.\ t.\ } \qquad &  \tr(\vec A \vec M) \leq- \lambda  &&   \label[sdp]{eq:exact_sdp}\\
        &    \snorm{F}{\vec A}^2 \leq 1          \notag
    \end{alignat}
    Moreover, from \Cref{pro:certificate-Biased-Massart}, the
    \Cref{eq:exact_sdp} is feasible if $\pr_{\x \sim \D_\x}[\sgn(\ell(\x)) \neq
    f(\x)]\geq \eps$.
    We define $\widetilde{\vec M}
    = \frac{1}{N} \sum_{i=1}^N
    \vec m(\sample{\bx}{i}) \vec m(\sample{\bx}{i})^\top \1_B(\sample{\bx}{i})\sample{y}{i}
    \ell(\sample{\bx}{i})$, the empirical estimate of $\vec M$ using $N$ samples from $\D$.
    Using the following fact, we bound the sample size required so that 
    $\wt{\vec M}$ is sufficiently close to $\vec M$ which is similar to Lemma 3.8 of \cite{DKTZ20b}. 
    (See \Cref{app:upper-bound-polynomial} for the proof).
    
    \begin{fact}[Estimation of $\vec M$]
        \label{lem:empirical_objective_error}
        Let $\Omega = \{\vec A \in \mathcal{S}^m:\ \snorm{F}{\vec A}
        \leq 1\}$ and $\eps,\delta\in(0,1)$. Let $\ell(\x)=\vec w \cdot \x+t$ with $|\ell(\x)|^2\leq C$ and $\widetilde{\vec M}
        = \frac{1}{N} \sum_{i=1}^N
        \vec m(\sample{\bx}{i}) \vec m(\sample{\bx}{i})^\top \1_B(\sample{\bx}{i})\sample{y}{i}
        \ell(\sample{\bx}{i})$. There exists an algorithm that draws
        $N =
        \frac{ d^{O(\log\frac{1}{\gamma \beta})}}{C\eps^2}
        \log(1/\delta)$
        samples from $\D$, runs in $\poly(N,d)$ time and
        with probability at least $1-\delta$ outputs a matrix $\wt{\vec M}$ such
        that
        $$
        \pr
        \left[
        \sup_{\vec A \in \Omega}
        \left| \tr(\vec A \wt{\vec M}) - \tr(\vec A \vec M) \right|
        \geq \eps
        \right]
        \leq 1-\delta\, .
        $$
    \end{fact}
    
    Using \Cref{lem:empirical_objective_error}, we replace the matrix $\vec M$ in
    \Cref{eq:certificate_objective} with the estimate $\widetilde{\vec M}$ and define the following ``empirical'' SDP
    \begin{alignat}{2}
        \text{Find }    \qquad  & \vec A \in \mathcal{S}^m                       && \notag\\ 
        \text{such that} \qquad &  \tr(\vec A \widetilde{\vec M}) \leq -\lambda/2  &&   \label[sdp]{eq:sample_sdp}\\
        &    \snorm{F}{\vec A}^2 \leq 1         \notag
    \end{alignat}  
    From \Cref{lem:empirical_objective_error}, we obtain that with $N$
    samples we can get a matrix $\wt{\vec M}$ such that $|\tr(\vec A \wt{\vec
        M} - \tr(\vec A \vec M)| \leq -\lambda/2$ with probability at least
    $1-\delta'$.
    From \Cref{pro:certificate-Biased-Massart}, we know
    that with the given bound for $k$ and $\snorm{F}{\vec A}$,
    there exists $\vec A^*$ such that
    $$
    \tr(\vec A^* \vec M) \leq -\lambda.
    $$
    Therefore, the \Cref{eq:exact_sdp} is feasible.  Moreover, from
    \Cref{lem:empirical_objective_error} we get that
    $$
    \tr(\vec A^* \wt{\vec M}) \leq -\lambda/2\;.
    $$
    Thus, the \Cref{eq:sample_sdp} is feasible.
    Since the dimension of the matrix $\vec A$ is smaller than the number of
    samples, we have that the runtime of the SDP is polynomial in the number of
    samples. Solving the SDP with tolerance $\lambda/4$, we obtain
    an almost feasible $\wt{\vec A}$, in the sense that
    $\tr(\wt{\vec A} \wt{\vec M}) \leq -\lambda/4$.
    Using again the guarantee of \Cref{lem:empirical_objective_error}, we get that solving the
    \Cref{eq:sample_sdp}, we obtain a positive-semi definite
    matrix $\wt{\vec A}$ such that
    $
    \tr(\wt{\vec A} \vec M) \leq -\lambda/4
    $.
    Moreover, we have that for the matrix $\vec A$ returned by our SDP it holds that 
    \begin{equation}
        \label[ineq]{eq:4-norm-bound}
        \E_{\x \sim \normal}[(\vec m(\x)^T\vec A \vec m(\x))^4]
        \leq \|\vec A\|_F^4 \E_{\x \sim \normal}[(\vec m(\x)^T\vec m(\x))^4]
        = d^{\log(1/(\beta \gamma))}
    \end{equation}
    
    To complete the proof, we need to show how to guess the band $B$.  For some
    large enough constant $C>0$, let $\mathcal T=\{\pm \lambda^2,\pm
    2\lambda^2,\ldots, C\sqrt{\log(1/\lambda)}\}$. Assume that for some
    $B=\{\x:r_1\leq\ell(\x)\leq r_2\}$, with $r_1,r_2\in \R$ and some polynomial
    $q(\x)$ with $\|q^2(\x)\|_2\leq 1$, such that $\E_{(\bx, y) \sim
    \D}\left[\1_{B}(\x)q^2(\x) \ell(\x) y  \right] \leq -\lambda/4$.
    Then, there exists $\tilde r_1,\tilde r_2\in \mathcal T$, with $|r_1-\tilde
    r_1|\leq \lambda^2$ and $|r_2-\tilde r_2|\leq \lambda^2$, such that 
    \begin{align*}
        \E_{(\bx, y) \sim \D}\left[ \1\{\tilde r_1\leq\ell(\x)\leq \tilde
        r_2\}q^2(\x) \ell(\x) y  \right] \leq \E_{(\bx, y) \sim
        \D}\left[\1_{B}(\x)q^2(\x) \ell(\x) y  \right] +2\lambda^2\leq
        -\lambda/4\;,
    \end{align*}
    where we used Cauchy–Schwarz inequality and the Gaussian concentration.
    Thus, by setting $\delta'=\Theta(\delta \lambda^4/\log(1/\lambda))$ solving
    the \Cref{eq:sample_sdp} for all the different choices of 
    $r_1,r_2\in\cal T$, with $r_1\leq r_2$, we guarantee that the algorithm will return a
    polynomial $q(\x)$ and some thresholds $\tilde r_1,\tilde r_2\in\cal T$, such
    that $T(\x) = \1\{\tilde r_1\leq\ell(\x)\leq \tilde r_2\}q^2(\x)$
    is a certifying function, i.e., it holds 
    $$\E_{(\bx, y) \sim \D}\left[ \1\{\tilde r_1\leq\ell(\x)\leq \tilde
    r_2\}q^2(\x) \ell(\x) y  \right] \leq -\lambda/4\;,$$
    with probability $1-\delta$.
    Finally, using \Cref{eq:4-norm-bound} we obtain that the 
    $L_4$ norm of $T(\x)$ is bounded from above as $\| T(\x) \|_4 \leq 
    d^{O(\log(1/(\beta \gamma )))}$.
    Thus, the function $T(\x)/\|T(\x)\|_4$ is an
    $\eps^4 d^{- O(\log(1/(\beta \gamma) ) ) }$-certificate.
    
\end{proof}

\subsection{Learning a Near-Optimal Halfspace via Online Convex Optimization}
\label[sub]{ssec:online}

In this subsection, we present a black-box approach to learn halfspaces with
$\eta$-Massart Noise given a $\rho$-certifying oracle.  A similar reduction for
homogeneous halfspaces based on online-convex optimization was given in
\cite{DKKTZ20}.  Here we adapt it so that it handles non-homogeneous halfspaces.
Another difference is that in \cite{DKKTZ20} the certificate function was bounded
in the $L_\infty$ sense.  Here we have certificates bounded in the $L_4$ norm.
The arguments however are similar, and we provide them here for completeness.
Formally, we prove:
\begin{proposition}
    \label{pro:certificate-reduction-formal}
    Let $\D$ be a distribution on $\R^d\times\{\pm 1\}$ with standard normal $\x$-marginal. Assume that
    $\D$ satisfies the $\eta$-Massart noise condition with respect to some halfspace. Fix $\eps,\delta\in(0,1)$.
    Given a $\rho$-certificate
    oracle $\cal O$ that returns certifying functions with bounded $\ell_4$-norm. There exists an algorithm that makes 
    $T = O(\frac{d\log(1/\eps)}{\rho^2\eps^4}) $ calls to $\mathcal O$, 
    draws $N=\tilde{O}(\frac{d}{\rho^2\eps^4}\log(1/\delta))$ samples from $\D$, 
    runs in $\poly(d,N,T)$ time and computes a hypothesis $h$, 
    such that $\pr_{(\x,y)\sim\D}[h(\x)\neq y]\leq \opt+\eps$, with probability $1-\delta$.
\end{proposition}

We will require the following standard regret bound
from online convex optimization.

\begin{lemma}[see, e.g., Theorem 3.1 of \cite{hazan2016introduction}]\label{lem:online_optimization}
    Let ${\cal V}\subseteq \R^n$ be a non-empty closed convex set with diameter $K$.
    Let $r_1,\ldots, r_T$ be a sequence of T convex functions $r_i: {\cal V}\mapsto \R$ 
    differentiable in open sets containing $\cal V$, and let $G=\max_{i\in[T]}\snorm{2}{\nabla_{\bw} r_i}$.
    Pick any $\vec w^{(1)}\in \cal V$ and set $\eta_i=\frac{K}{G\sqrt{t}}$ for $i\in[T]$. Then, for all 
    $\vec u\in \cal V$, we have that
    $\sum_{i=1}^{T}( r_i(\vec w^{(t)}) -r_i(\vec u))\leq \frac 32 GK\sqrt{T}$.
\end{lemma}

\begin{algorithm}[h]
    \caption{ Learning Halfspaces with $\eta$-Massart noise. } 
    \label{alg:online-gradient}
    
    \centering
    \fbox{\parbox{6in}{
            {\bf Input:} 
            \begin{enumerate}
                \item $\eps, \delta > 0$.
                \item  A distribution $\D$ that satisfies the $\eta$-Massart Noise condition.
                \item Access to a $\rho$-certificate oracle $\cal O$
            \end{enumerate}
            
            {\bf Output:} A vector $\vec w$ and threshold $t\in \R$ such that $\pr_{(\x,y)\sim \D}[\sign(\vec w\cdot \x+t)\neq y]\leq \opt+\eps$ \\
            
            {\bf Define:}  $ T=C~d\log(1/\eps)/(\rho\eps)^2$, $N=C~d/\eps^2\log(1/\eps)\log(T/\delta)$, for some large enough constant $C>0$,  ${\cal V}=\{(\vec w,t) \in \R^{d+1} : \snorm{2}{\vec w}\leq 1\},|t|\leq 4\sqrt{\log(1/\eps)}\}$\\
            \begin{enumerate}
                \item ${\vec w}^{(0)} \gets \vec e_1$, $t^{(0)}\gets 0$
                \item For $i\in [T]$ do
                \begin{enumerate}
                    \item $\eta_i \gets 1/(\sqrt{i} +\rho)$
                    \item If $(\vec {\vec w}\ith,t\ith)=\vec 0$  then $T\ith(\x)=0$.
                    \item Else let $\mathrm{ANS}\gets {\cal O}((\vec {\vec w}\ith,t\ith))$.
                    \item If $\mathrm{ANS}=\mathrm{FAIL}$ then {\bf return} $(\vec w\ith,t\ith)$ else $T\ith\gets \mathrm{ANS}$
                    \item Let $\nabla \hat r_i(\vec w,t)$ be an estimator of $-\E_{(\vec x,y) \sim {\D}}\left[\left(T\ith(\vec x) +\rho \right) y (\x,1) \right]$ (\Cref{lem:algorithm_function_ell}) \item  $({\vec w}^{(i+1)},t^{(i+1)}) \gets \proj_{\cal V}\left(({\vec w}\ith,t\ith) - \eta_i \nabla \hat r_i( {\vec w}\ith,t\ith)\right)$ \label{alg:OPGDstep} 
                \end{enumerate}
                \item {\bf return} $(\vec w^{(T+1)},t^{(T+1)})$
            \end{enumerate}
    }}
\end{algorithm}
We show below that the optimal vector $\wstar$ and threshold $t^\ast$ and our current candidate vector $\vec w\ith$ and threshold $t\ith$  have a separation in the value of $r_i$.
\begin{lemma}[Error of $r_i$]\label{lem:expectation_error} 
    Let $\D$ be a distribution on $\R^d \times \{\pm 1\}$ with standard normal $\x$-marginal. Assume that $\D$ satisfies the $\eta$-Massart noise
    condition with respect to the optimal halfspace $\sign(\wstar \cdot \x +t^\ast)$.
    Let $\sample{\vec w}{i}$ with $\|\vec w\ith\|\leq 1$ and $t\ith\in \R$. 
    Fix $\rho\in(0,1)$ and let
    $r_i(\vec w,t)= -\E_{(\bx,y) \sim \D} [(T\ith(\vec x)+ \rho)y(\x,1)]\cdot(\vec w,t)$, 
    where $T\ith(\vec x)$ is a non-negative function returned from a $(2\rho)$-certificate oracle, we have that
$$ r_i(\vec w\ith,t\ith)-r_i\left( \wstar,t^\ast \right) \geq \rho\eps^2/2\;.$$
\end{lemma}
\begin{proof}
    Let $\ell\ith(\x)=\vec w\ith \cdot \x +t\ith$. Using the fact that $T\ith(\x)\geq 0$, we have 
    $$ \E_{(\vec x,y)\sim \D}[T\ith(\x)(\dotp{\wstar}{\vec x}+t^\ast)y]=\E_{\vec x\sim \D_{\bx}}[T\ith(\x)|\dotp{\wstar}{\vec x}+t^\ast|\beta(\x)]>0\;.$$
    Therefore, we have that for every $i\in[T]$, it holds 
    $r_i(\wstar,t^\star) \leq -\rho\E_{\vec x\sim \D_{\bx}}[|\dotp{\wstar}{\vec x}+t^\ast|\beta(\x)]$. 
    Let $I=\{\x: \wstar\cdot \x \in (-t^\ast-\eps/2,-t^\ast+\eps/2)\}$ and note that it should hold that $ \Pr_{(\x,y)\sim \D}[f(\x)\neq y]=\opt\leq 1/2-\eps$, otherwise the $(2\rho)$-certifying oracle would not be able to return a function, and therefore $\E_{\x\sim \D_\x}[\beta(\x)]\geq 2\eps$. We have that
    \begin{align}
        \E_{\vec x\sim \D_{\bx}}[|\dotp{\wstar}{\vec x}+t^\ast|\beta(\x)]& \geq  \E_{\vec x\sim \D_{\bx}}[|\dotp{\wstar}{\vec x}+t^\ast|\{\x \not\in I\}\beta(\x)]\geq \frac{\eps}{2}\E_{\vec x\sim \D_{\bx}}[\{\x \not\in I\}\beta(\x)]\nonumber
        \\&=\frac{\eps}{2}\E_{\vec x\sim \D_{\bx}}[\beta(\x)]-\frac{\eps}{2}\E_{\vec x\sim \D_{\bx}}[\{\x \in I\}\beta(\x)]
        \geq \eps^2/2\;. \nonumber
    \end{align}
    where in the last inequality we used that $\E_{\vec x\sim \D_{\bx}}[\beta(\x)]\geq 2\eps$ and $\E_{\vec x\sim \D_{\bx}}[\{\x \in I\}]\leq \eps$ from \Cref{fct:gaussian-tails}.

    It remains to bound from below $r_i(\vec w\ith,t\ith)$. Using the fact that $\E_{(\x,y)\sim \D}[T\ith(\x)(\vec w\ith \cdot\x +t\ith)y]\leq -2\rho \|\ell\ith(\x)\|_2$, we have
    \begin{align*}
        r_i(\vec w\ith,t\ith)
        &=- \E_{(\bx,y) \sim \D} \left[(T\ith\left( \bx \right)+\rho) (\vec w\ith\cdot \x + t\ith)y\right]\\  &\ge 2\rho\|\ell\ith(\x)\|_2
        -  \rho \E_{\bx \sim \D_{\x}} \left[(\vec w\ith\cdot \x +t\ith)y\right] 
        \geq \rho  \|\ell\ith(\x)\|_2 \geq 0 \;,
    \end{align*}
    where we used the Cauchy-Schwarz inequality and the fact that $\x$ is standard normal.
\end{proof}

Since we do not have access to $r_i$ precisely, 
we need a function $\hat{r}_i$, which is close to $r_i$ with high probability. 
The following simple lemma gives us an efficient way to compute an approximation $\hat{r}_i$ of $r_i$.
\begin{lemma}[Estimating the function $r_i$]\label{lem:algorithm_function_ell}
    Let $\D$ be a distribution on $\R^d\times\{\pm 1\}$ with standard normal $\x$-marginal and let $T\ith(\x)$ be a non-negative function returned by a $(2\rho)$-certificate oracle. Moreover, assume that $T\ith(\x)$ has bounded $\ell_4$ norm, i.e., $\|T\ith(\x)\|_4\leq 1$.
    Then after drawing $O(d\log(1/\eps)/\eps^2\log(d/\delta))$ samples from $\D$, 
    with probability at least $1-\delta$, we can compute an estimator $\hat{r}_i$ that
    satisfies the following conditions:
    \begin{itemize}
        \item $\snorm{2}{\nabla \hat r_i(\vec w,t) - \E_{(\bx,y) \sim \D} [(T\ith(\vec x) +\rho) y(\x,1)]} \leq \eps/\sqrt{\log(1/\eps)}$

        \item $\snorm{2}{\nabla \hat r_i(\vec w,t)} \leq 2\sqrt{d}$.\end{itemize}
\end{lemma}
The proof of the lemma above can be found on \Cref{app:upper-bound-polynomial}. We now proceed with the proof of \Cref{pro:certificate-reduction-formal}.

\begin{proof}[Proof of \Cref{pro:certificate-reduction-formal}]
    Note for the proof for simplicity, we assume that we have access to $(2\rho)$-certifying oracle, but the same argument works for $\rho'=\rho/2$ and access to $\rho'$-certifying oracle.
    Let $\ell\ith(\x)=\vec w\ith\cdot \x +t\ith$. We define ${\cal V}=\{(\vec w,t)\in \R^{d+1}: \|\vec w\|_2\leq 1, |t|\leq 4\sqrt{\log(1/\eps)}\}$. Let $T$ be the number of optimization steps that our algorithm runs.
    Assume, in order to reach a contradiction, that for all steps $i\in[ T]$ it holds that 
    $\pr_{(\x,y)\sim \D}[\sign(\ell\ith(\x))\neq y]\geq \opt+ \eps$. Let $f(\x)=\sign(\wstar \cdot \x +t^\ast)$. For each step $i$, let define $T\ith(\bx)$ to
    be the non-negative function outputted by the $(2\rho)$-certifying oracle $\cal O$. Thus, we have 
    $$\E_{(\bx,y) \sim \D}[T\ith(\bx)y(\vec w\ith\cdot \x +t\ith)]\leq-\rho \|\ell\ith(\x)\|_2 \;.$$
    Let $ \nabla \hat r_i(\vec w,t)$ be an estimator of
    $\nabla r_i\left( \vec w,t \right) =-\E_{(\bx,y) \sim \D} [\left(T\ith(\bx) +\rho \right) y (\x,1)]$ such that with probability at least $1-\frac{\delta}{T}$ it holds $ \snorm{2}{\nabla \hat r_i(\vec w,t) - \nabla r_i(\vec w,t)}\leq \frac 12 \rho\eps^2$ for all $(\vec w,t)\in \cal V$. Moreover, from \Cref{lem:algorithm_function_ell},
    for $N=\tilde{O}(d/\log(T/\delta)/(\eps^4\rho^2))$ samples, we can achieve that.

    Using the separation between the optimal hypothesis and the current one, i.e., \Cref{lem:expectation_error}, for every  step $i\in[T]$, we have that $ r_i(\vec w\ith,t\ith)-r_i(\vec w^{\ast},t^\ast)\geq \frac 12 \rho\eps^2$.
    Therefore, it holds  
    \begin{equation}\label{eq:online-bound-1}
        \hat r_i(\vec w\ith,t\ith)-\hat r_i(\vec w^{\ast},t^\ast)\geq \frac {1}{4} \rho\eps^2\;,
    \end{equation}
    with probability at least $1-\frac{\delta}{T}$.
    In order to apply the online gradient descent algorithm, i.e., \Cref{lem:online_optimization}, we need to define the parameters of the algorithm. First, the convex set we optimize over is $\cal V$, hence, the diameter of $\cal V$ is $K= O(\sqrt{\log(1/\eps)})$.
    Furthermore, from \Cref{lem:algorithm_function_ell}, we get that $\|\nabla \hat r_i(\vec w,t)\|_2 =O(\sqrt{d})$, and therefore, from  \Cref{lem:online_optimization}, it holds
    \[ \frac 1T \sum_{i=1}^{T}\left(\hat{r}_i( {\vec w^{(i)}},t\ith )   - \hat{r}_i\left( {\vec w^\ast},t^\ast \right)\right) 
    \lesssim \frac{\sqrt d+\sqrt{\log(1/\eps)} }{\sqrt{T}}\;.\]
    By the union bound, it follows that with probability at least $1-\delta$, we have that
    \[ \frac {1}{4} \rho\eps^2 \leq 	\frac 1T \sum_{i=1}^{T}\left(\hat{r}_i( {\vec w^{(i)}},t\ith )   - \hat{r}_i( {\vec w^{\ast}},t^\ast )\right) 
    \lesssim \frac{\sqrt d+\sqrt{\log(1/\eps)}}{\sqrt{T}} \;,\]
    which leads to a contradiction for $T=\Theta(\frac{d\log(1/\eps)}{\rho^2\eps^4})$. 
    
    Thus, either there exists $i\in [T]$ such that
    $\pr_{(\x,y)\sim \D}[\sign(\ell\ith(\x))\neq y]\leq \opt+\eps$, which the algorithm returns, 
    or the $(2\rho)$-certifying oracle $\cal O$  did not provide a correct certificate, which happens with probability at most $\delta$. 
    Moreover, the algorithm calls the certificate oracle $T$ times and the number of samples needed are $N$.
    The runtime is the number of steps and multiplies by the number of samples $T~N$.
    This completes the proof.
\end{proof}

Given  \Cref{pro:certificate-reduction-formal}, the proof of \Cref{thm:massart-upper-bound} follows by showing that the function returned by \Cref{prop:sdp-oracle} is a $\rho$-certifying oracle, for an appropriate choice of $\rho$.
\begin{proof}[Proof of \Cref{thm:massart-upper-bound}]
    In order to prove  \Cref{thm:massart-upper-bound}, we need to construct a $\rho$-certifying oracle. Using $N=O(d^{\log(1/(\gamma\beta))})\log(1/\delta)/\eps^2$ samples, \Cref{prop:sdp-oracle} provides with probability $1-\delta$ a $\rho=\eps^2 \poly(\beta\gamma)$ certifying oracle in runtime $\poly(N,d)$. Therefore, from \Cref{pro:certificate-reduction-formal} the proof follows. The overall samples complexity is $O(N\log(T))$ and the runtime is $\poly(N,d)$.
\end{proof}
 \section{Learning Halfspaces with General Massart Noise}\label{sec:benign}

In this section, we present our algorithm for learning halfspaces with general Massart noise.
The main result of this section is the following theorem.

\begin{theorem}[Learning Halfspaces with General Massart Noise]
    \label{thm:benign}
    Let $\D$ be a distribution on $\R^{d} \times \{\pm 1\}$, with standard normal
    $\x$-marginal, that satisfies the Massart noise condition for $\eta=1/2$ with respect to
    a target (possibly biased) halfspace $f \in {\cal C}$.
    Let $\eps, \delta \in (0,1]$.
    There exists an algorithm that draws $N = d^{O(\log(1/\eps))} \log(1/\delta)$
    samples from $\D$, runs in time $\poly(N, d) 2^{\poly(1/\eps)}$,
    and computes a halfspace $h \in {\cal C}$ such that with probability at least $1-\delta$ it holds
    $
    \pr_{(\x, y) \sim \D}[h(\x) \neq y] \leq \opt  + \eps
    \,.
    $
\end{theorem}

We remark that the algorithm of \Cref{thm:benign} works for any halfspace, regardless of
whether it is biased or not. In the presentation that follows, we will focus on learning
homogeneous halfspaces for the sake of simplicity.  It is not hard 
to generalize the algorithm to work for arbitrary halfspaces. (We present the
general algorithm for arbitrary halfspaces in \Cref{app:general_benign}.)
The proof of \Cref{thm:benign} consists of three main parts;
see \Cref{ssec:high-noise} for a roadmap of the proof.
First, in \Cref{sub:projection} we show that by
restricting on a thin slice and then projecting the distribution
on the subspace $\vec w^\perp$, we obtain an instance where 
the optimal halfspace has again $1/2$-Massart noise everywhere 
apart from a small region close to the halfspace, see \Cref{fig:orth_proj}.
In \Cref{sub:sign-mathcing-polynomial}, we present our main technical
contribution, i.e., that there exists a low-degree, mean-zero, 
polynomial that achieves non-trivial correlation when $\eta(\x) \leq 1/2$
``almost everywhere''.  Finally, in \Cref{sub:chow-tensors} we show
that the existence of such polynomials implies that by finding
the left singular vectors of (a flattened version) of the Chow tensors
of the distribution, we can construct a 
$\poly(1/\eps)$-dimensional subspace
inside which the direction $\wperp$ has $\poly(\eps)$ projection.
Recall that $\wperp$ is the normalized projection of $\vec w^\ast$
onto the orthogonal complement of $\vec w^\perp$:
\[
    \wperp = \frac{\proj_{\vec w^\perp} (\vec w^\ast)}{\|\proj_{\vec w^\perp}(\vec w^\ast)\|_2}
    \,.
    \]
In \Cref{sub:chow-tensors}, we combine everything together to prove \Cref{thm:benign}.

We now state the main technical result of this section. (This is the formal
version of \Cref{prop:tech-overview-warm-start}.)

\begin{proposition} \label{prop:warm-start}
    Let $\D$ be a distribution on $\R^{d} \times \{\pm 1\}$, with standard
    normal $\x$-marginal, that satisfies the Massart noise condition for $\eta =1/2$,
     with respect to some target halfspace $f(\x) = \sgn(\vec w^\ast \cdot
     \x)$.  Let $\vec w \in \R^d$ be a unit vector such that $\pr_{(\x, y) \sim
     \D}[\sign(\vec w \cdot \x) \neq y] \geq \opt + \eps$ and 
     $\theta(\vec w,\wstar)\leq \pi-\eps$ for some $\eps \in (0,1/2]$.
     There exists an algorithm that draws $N=d^{O(\log(1/\eps))}\log(1/\delta)$ samples from $\D$, 
     runs in time $\poly(N,d)$, and with probability at least $1-\delta$ returns a basis of a subspace 
     $V\subseteq \vec w^\perp$ such that $\| \proj_{V}(\wperp)\|_2\geq \poly(\eps)$.
\end{proposition}
Observe that given a subspace $V$ with the above guarantees, we can sample a uniformly
random direction on the unit sphere of $V$ and obtain a unit update vector $v$ such that
$\vec v \in \vec w^\perp$ and $\vec v \cdot \wperp \geq \poly(\eps)$, which given 
\Cref{lem:corr-improv} improves the current guess $\vec w$; see \Cref{sub:chow-tensors}
for the details.

\subsection{The Sign-Matching Polynomial}
\label[sub]{sub:sign-mathcing-polynomial}
Here we give an explicit construction of a mean-zero low-degree
polynomial that achieves non-trivial correlation with the labels $y$.  The
assumptions on the distribution $\D$ over $(\x, y) \in \R^d \times \{\pm 1\}$ are
that it has Gaussian $\x$-marginal and ``almost-Massart'' noise, i.e., that is
$\eta(\x)$ is greater than $1/2$ only on a very small region close to the
optimal halfspace.  We prove the following:  
\begin{proposition}[Correlation via the Sign-Matching Polynomial]
    \label{lem:correlation-polynomial}
    Let $\D$ be a distribution on $\R^{d} \times \{\pm 1\}$
    whose $\x$-marginal is the standard normal distribution.
    Let $f(\bx) = \sgn(\dotp{\vec v^\ast}{\bx} - b)$ be such that
    $\E_{\x\sim \D_{\x}}[\beta(\x)\1\{\sign(b)f(\x)>0\}] \geq \zeta$ 
    for some $\zeta \in (0,1/2]$, and
    $\eta(\x) > 1/2$ only when $0\leq \sign(b)(\dotp{\vec v^\ast}{\bx} -b)\leq\xi$,
    where $\xi$ is a sufficiently small constant multiple of $\zeta^3$. 
    There exists a univariate polynomial $p(z)$ of degree $\Theta(b^2 + 1)$ 
    such that 
    $\E_{\vec x\sim \D_\x}[p(\vec v^\ast \cdot \x)  ]=0$,
    $\E_{\x \sim \D_\x}[ p^2(\vec v^\ast \cdot \x)] = 1$,
    and  $\E_{(\vec x,y)\sim \D}[yp(\vec v^\ast \cdot \vec x)  ]=\poly(\zeta)$.
\end{proposition}
\begin{proof}
    The main ingredient of the proof is the following lemma
    that shows the existence of a mean zero polynomial
    that matches the sign of the linear function $z-b$; see \Cref{fig:polynomial}.

    \begin{lemma}[Sign-Matching Polynomial] \label{clm:zero-mean-polynomial}
        Let $b \in \R$.  There exists a zero mean and unit variance polynomial $p:\R \mapsto \R$
        of degree $ k = \Theta(b^2 + 1)$ such that
        \begin{itemize}
            \item The sign of $p$ matches the sign of the threshold function $\sign(z-b)$, i.e., $\sign(p(z)) = \sign(z - b)$ all $z \in \R$,
            \item for any $\rho\in(0,1)$, it holds $|p(b+\sgn(b)\rho)| \geq 
            c^k\poly(\rho)$, where $c>0$ is a universal constant,
            \item  $p(z)$ is increasing for all $|z| \geq |b|$.
        \end{itemize}  
    \end{lemma}
    \begin{proof}
We first pick some odd integer $k$ large enough such that $2|b|\leq \sqrt{k} \leq 4\max(|b|,1)$. 
        Let $q(z) = z^{3k}, r(z) = z^{2k} - (2k - 1)!!$. We consider the polynomial 
        \[
        p(z) = q(z) - \frac{q(b)}{r(b)} r(z) \,. \] It holds $p(b) = 0$ and $\E_{z \sim
            \normal}[p(z)] = 0$, since $\E_{z \sim \normal}[z^{3k}]=0$, for odd $k$ and $\E_{z \sim
            \normal}[z^{2k}]=(2k-1)!!$.
        
        For the rest of the proof, we need the following simple estimate for double factorials.
        Sharper bounds can be obtained via Stirling's approximation; 
        for our purposes the following rough bounds suffice.
        \begin{fact}\label{fct:double-fact} Let $m$ be a positive integer. Then
            $(m/2)^m\leq (2m-1)!!\leq (2m)^m$.
        \end{fact}

        We next show that $p(z)$ is increasing for all $|z| \geq |b|$.  
        We compute the derivative of $p(z)$, i.e.,  
        \begin{equation}\label{eq:deriv}
            p'(z)=3k z^{3k-1} -2k\frac{q(b)}{r(b)}z^{2k-1}=kz^{2k-1} \left(3z^k - 2\frac{q(b)}{r(b)}\right)\;,  \end{equation} 
        we observe that if $|z^k|\geq \frac{2}{3} |q(b)/r(b)|$, then $p(z)$ is increasing. 
        We show that this is the case for all $|z| \geq |b|$. Since $|z|^k$ is increasing in $|z|$ it suffices to 
        verify the inequality for $|z| = |b|$.  In that case we obtain the inequality $(2k-1)!!/b^{2k}\geq 5/3$, which 
        can be verified by using \Cref{fct:double-fact} and our assumption that $|b|\leq \sqrt{k}/2$.  
        
        We now prove that $\sign(p(z)) = \sgn(z - b)$. Note that $\sign(r(z))<0$ for any $z\in[-|b|,|b|]$ because $|b|\leq \sqrt{k}/2$ and thus, from \Cref{fct:double-fact}, it holds $b^{2k}-(2k-1)!!\leq (k/4)^k -(k/2)^k\leq 0$.
        Without loss of generality, assume that $b>0$.   Assume that there exists  $s \in \R$ with $s\neq b$ such that $p(s)=0$. First, assume that $s\geq 0$,  by \cref{eq:deriv}, since $r(b) < 0$ we get that $p(z)$ is increasing for $z \geq 0$, therefore the polynomial $p(z)$ has only one root at $s=b$, therefore we have a contradiction.
        For the case that $s<0$, first observe that for any $s=0$, it holds that  $q(s)r(b)= q(b)r(s)$ and in particular $\sign(q(s)r(b))=\sign(q(b)r(s))$. Thus, if $-b>s$, then it holds $\sgn(q(b)r(s))=-\sgn(b)$ and  $\sgn(q(s)r(b))=-\sgn(s)$, therefore for any such solution it holds $\sgn(s)=\sgn(b)$ but $s<0$ and $b>0$, so we have a contradiction.
        Finally, for the case that $s\leq -b$, recall that for $|z|\geq |b|$, $p(z)$ is increasing, so for $z<-b$, we have that $p(z)$ is increasing. But $p(-b)<0$ because there is no other root in the interval $[-b,b)$, thus $p(s)<0$ which leads to a contradiction. Therefore, the only root of the polynomial is at $s=b$.
        
        We next prove that $\E_{z \sim \normal}[ p^2(z)] = (O(k))^{3k}$. By applying twice the inequality
        $(a+b)^2\leq 2a^2+2b^2$, we get 
        $$
        \E_{z \sim \normal}[ p^2(z)] 
        \lesssim \E_{z \sim \normal}[ z^{6k}] + \frac{q^2(b)}{r^2(b)} \left(\E_{z \sim \normal}[ z^{4k}] +((2k-1)!!)^2\right)
        \lesssim (6k-1)!!+ \frac{q^2(b)}{r^2(b)}(4k-1)!!\;.
        $$
        Using that $|b|\leq \sqrt{k}/2$ and \Cref{fct:double-fact}, we have that $|b^{2k}-(2k-1)!!|\geq (k/2)^k/2$. Hence, we get that $ (q(b)/r(b))^2\lesssim (2|b|^3/k)^{2k} \lesssim k^k$. Therefore, using \Cref{fct:double-fact}, we get that $\E_{z \sim \normal}[ p^2(z)] =O((6k)^{3k})$.
        
        Without loss of generality, assume that $b>0$. Observe that, for $\rho\in(0,1)$, we have
        \begin{align*}
            p(b+\rho)=(b+\rho)^{3k} -\frac{b^{3k}((b+\rho)^{2k} -(2k-1)!!)}{b^{2k} -(2k-1)!!}=(b+\rho)^{3k}-b^{3k} - \frac{b^{3k}\sum_{m=0}^{2k-1} \binom{2k}{m}b^m\rho^{2k-m}}{b^{2k} -(2k-1)!!}\;.
        \end{align*}
        Recall that using the fact that $|b|\leq \sqrt{k}/2$, it holds that $b^{2k} -(2k-1)!!<0$, thus
        \[
        p(b+\rho)\geq (b+\rho)^{3k}-b^{3k}=\sum_{m=0}^{3k-1} \binom{3k}{m}b^m\rho^{3k-m}\;.
        \]
        From the equation above, we get that $p(b+\rho)\gtrsim b^{3k-1}\rho+\rho^{3k}$.
        Let $\tilde{p}(z)$ be the polynomial $p(z)$ normalized, i.e., $\tilde{p}(z)=p(z)/(\E_{z\sim \normal}[p^2(z)])^{1/2}$. 
        Notice that normalizing $p$ does not affect its sign or monotonicity.
        It remains to prove that $\tilde{p}(b+\rho)= c^k\poly(\rho)$, where $c>0$ is a small enough constant.
        If $b\geq 1$, recall that in that case it holds that $\sqrt{k}\leq 4b$, therefore it holds $\tilde{p}(b+\rho)\gtrsim b^{3k-1}\rho/(6k)^{3k/2}\gtrsim \rho/(24)^{2k}$.
        For the case where $0<b<1$, we have that $k=O(1)$ and similarly, we have $\tilde{p}(b+\rho)\gtrsim \rho^{2k}/(6k)^{3k/2}=\poly(\rho)$, hence $\tilde{p}(b+\rho)= c^k\poly(\rho)$, for some universal constant $c>0$.
    \end{proof}
    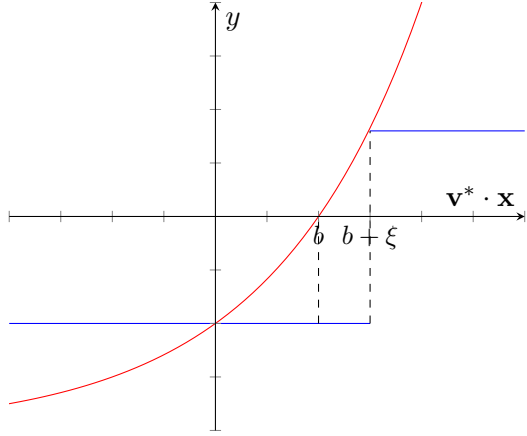
\begin{figure}
    \centering
    	\begin{tikzpicture}[scale=1]
			\begin{axis}[ restrict y to domain=-4:4,xlabel=$\vec v^\ast \cdot \x$,ylabel=$y$, legend pos=north west,axis y line =middle,
				axis x line =middle,
				axis on top=true,
				xmin=-2,
				xmax=3,
				ymin=-2,
				ymax=2,
                xtick distance=0.5,
                ytick distance=0.5,
				yticklabel style = {font=\small, xshift=0.5ex},
				xticklabel style = {font=\small, yshift=0.5ex},
                xticklabels={,,,,,,,$b$,$b+\xi$},
				yticklabels={},
				]
				\addplot[color=red, samples=100, domain=-4:4]
				{2^x-2};
				\addplot[color=blue, samples=10, domain=1.5:3]{0.8};
				\addplot[color=blue, samples=10, domain=-2:1.5]
				{-1};
                  \draw [dashed] (axis cs: 1.5,-1) -- (axis cs: 1.5,0.8);
                  \draw [dashed] (axis cs: 1,-1) -- (axis cs: 1,0);
\end{axis}
	\end{tikzpicture}
    \caption{The sign-matching polynomial $p$ corresponds to the red curve and the threshold function 
    $f(\x)$ to the blue.  Observe that, in the region $I^\xi = \{ b \leq \vec v^\ast \cdot \x \leq b+ \xi \}$,
    the sign of the polynomial $p$ does not agree with $f(\x)$.  This is the region
    where the Massart noise condition is (potentially) violated, as a result of the orthogonal projection
    step, \Cref{fig:orth_proj}.
    Even though the sign of $p$ does not agree with $f(\x)$ everywhere,
    we can take $\xi$ to be small, i.e., $\xi$ is a sufficiently small constant multiple of $\zeta^3$,
    making the negative contribution of the region $I_\xi$ small.  We also use crucially the 
    fact that $p$ is monotone for $\vec v^\ast \cdot \x \geq b$ and not very flat in that region, 
    see Item 2 of \Cref{clm:zero-mean-polynomial}.  
    }
    \label{fig:asign-matching-pol-5}
        \end{figure}
    It remains to show that the polynomial $p$ of the
    \Cref{clm:zero-mean-polynomial}, achieves non-trivial correlation with $y$,
    i.e., $\E_{(\x,y)\sim \D}[yp(\vec v^\ast \cdot \x)]=\poly(\zeta)$.  
    To do this we are going to use the monotonicity properties of $p$,
    the fact that it matches the sign of $f(\x)$, and the fact
    that it is not very flat close to its unique root at $b$, see Item 2
    of \Cref{clm:zero-mean-polynomial}.
    We assume that $b\geq 0$ as the other case follows similarly. Recall that,
    from the assumptions of \Cref{lem:correlation-polynomial}, there is an
    explicitly defined region where the noise is above $1/2$, i.e.,
    $\eta(\x)>1/2$.
    We denote this region $I^\xi$, see also \Cref{fig:asign-matching-pol-5}, that
    is $$I^\xi=\{\x\in \R^d: \eta(\x)>1/2\} \subseteq\{\x\in \R^d: b\leq \vec
    v^\ast \cdot \x\leq b+\xi\}\;.$$ Since $p(\vec v^\ast \cdot \x)$ ($p$ from
    \Cref{clm:zero-mean-polynomial}) matches the sign of $\vec v^\ast \cdot \x -
    b$ we have
    \begin{align}
        \E_{(\x,y) \sim \D}[p(\vec v^\ast \cdot \x) y]
        & = \E_{\x \sim \D_\x}[p(\vec v^\ast \cdot \x) \sgn(\vec v^\ast \cdot \x - b) \beta(\x) (\1\{\x \not \in I^\xi\}-\1\{\x \in I^\xi\})]\nonumber
        \\&\geq \E_{\x \sim \D_\x}[|p(\vec v^\ast \cdot \x)| \beta(\x)\1\{\x \not \in I^\xi\}]- \E_{\x \sim \D_\x}[|p(\vec v^\ast \cdot \x)|\1\{\x \in I^\xi\}] \label{ineq:proof-lem25-1}\,.
    \end{align}
    We first bound from below the correlation of the sign-matching polynomial outside the
    region $I^\xi$, i.e., the region where the Massart noise condition is true.
    From \Cref{clm:zero-mean-polynomial}, we have that $p(z)$ is increasing for
    $z>b$, and therefore it holds 
    \begin{align}
        \E_{\x \sim \D_\x}[&|p(\vec v^\ast \cdot \x)| \beta(\x)]\1\{\x \not \in I^\xi\}] \nonumber
        \\
        & \geq  \E_{\x \sim \D_\x}[|p(\vec v^\ast \cdot \x)| \beta(\x)\1\{\vec v^\ast \cdot \x>b+\zeta^2\}]\nonumber
        \\ &\geq |p(b+\zeta^2)|\E_{\x \sim \D_\x}[ \beta(\x)\1\{\vec v^\ast \cdot \x>b+\zeta^2\}]\nonumber\\
        &= |p(b+\zeta^2)|\left(\E_{\x \sim \D_\x}[ \beta(\x)\1\{\vec v^\ast \cdot \x>b\}]-\E_{\x \sim \D_\x}[\1\{b\leq \vec v^\ast \cdot \x\leq b+\zeta^2\}]\right)\;, \label{ineq:proof-lem25-2}
\end{align}
    where for the first inequality we used the fact that $\1\{\vec v^\ast \cdot
    \x>b+\xi\}\geq \1\{\vec v^\ast \cdot \x>b+\zeta^2\}$, because $\xi=
    \Theta(\zeta^3)\leq \zeta^2$ from the assumptions of
    \Cref{lem:correlation-polynomial}, and for the second, we used the
    monotonicity of $p$.
    We next bound from above the contribution of the sign-matching polynomial
    on the region where the Massart noise condition is violated, i.e., $\x \in I^\xi$.
    Moreover, using again the fact that $p(z)$ is increasing for $z>b$, we get
    \begin{align}
        \E_{\x \sim \D_\x}[|p(\vec v^\ast \cdot \x)|]\1\{\x  \in I^\xi\}] 
        & \leq  \E_{\x \sim \D_\x}[|p(\vec v^\ast \cdot \x)|\1\{b\leq \vec v^\ast \cdot \x\leq b+\zeta^2\}]\nonumber
        \\ &\leq |p(b+\zeta^2)|\E_{\x \sim \D_\x}[\1\{b\leq \vec v^\ast \cdot \x \leq b+\zeta^2\}] \;.\label{ineq:proof-lem25-3}
    \end{align}
    Substituting the bounds of \cref{ineq:proof-lem25-2} and \cref{ineq:proof-lem25-3} to \cref{ineq:proof-lem25-1}, we get that
    \begin{align*}
        \E_{(\x,y) \sim \D}[p(\vec v^\ast \cdot \x) y] &\geq  |p(b+\zeta^2)|\left(\E_{\x \sim \D_\x}[\beta(\x)\1\{\vec v^\ast \cdot \x\geq b\}] -2\E_{\x \sim \D_\x}[\1\{b\leq \vec v^\ast \cdot \x \leq b+\zeta^2\}]  \right)\\
        &\gtrsim |p(b+\zeta^2)|(\zeta - \zeta^2) \gtrsim |p(b+\zeta^2)|\zeta\;,
    \end{align*}
    where we used that from the assumptions of \Cref{lem:correlation-polynomial} we have $\E_{\x\sim \D_{\x}}[\beta(\x)\1\{\sign(b)f(\x)>0\}] \gtrsim \zeta$, and the 
    anti-concentration property of the Gaussian distribution from \Cref{fct:gaussian-tails}, i.e., 
    $\E_{\x \sim \D_\x}[\1\{b<\vec v^\ast \cdot \x<b+\zeta^2\}] \leq \zeta^2/(\sqrt{2\pi})$.
    
    It remains to prove that $|p(b+\zeta^2)|=\poly(\zeta)$. From \Cref{clm:zero-mean-polynomial}, we have that $|p(b+\zeta^2)=c^k\poly(\zeta)$, for some universal constant $c>0$. Note that from the assumption of \Cref{lem:correlation-polynomial} we have that $\E_{\x\sim \D_{\x}}[\beta(\x)\1\{\sign(b)\sgn(\vec v^\ast\cdot \x-b)>0\}] \gtrsim \zeta$, hence, from \Cref{fct:gaussian-tails} we get $|b|\lesssim \sqrt{\log(1/\zeta)}$, therefore $c^k=\poly(\zeta)$, since $k=\Theta(b^2+1)$. Hence, we have $|p(b+\zeta^2)|= \poly(\zeta)$. This completes the proof of \Cref{lem:correlation-polynomial}.
    
\end{proof}

\subsection{Projecting onto $\vec w^\perp$}
\label[sub]{sub:projection}

In this subsection, we show that we can project $\D$ onto the subspace $\wperp$ without violating
the $1/2$-Massart noise condition by a lot.  At the same time we make sure that the ``optimal halfspace'' 
of the projected instance is not very biased.  We show the following lemma.
\begin{lemma}[Orthogonal Projection Inside a Band]
    \label{lem:band-projection}
    Let $\D$ be a distribution on $\R^{d} \times \{\pm 1\}$, with standard
    normal $\x$-marginal, that satisfies the Massart noise condition for
    $\eta=1/2$ with respect to $f(\x)=\sign(\vec w^\ast \cdot \x)$. Let $\vec w$
    be a unit vector such that $\pr_{(\x, y) \sim \D}[\sign(\vec w \cdot \x)
    \neq y] \geq \opt + \eps$, and $\theta(\vec w, \vec w^\ast) \leq \pi - \eps$, 
    for some $\eps\in(0,1]$. Fix some $\rho>0$.
    For $t_1, t_2 \in \R$, consider the band $B=\{t_1\leq\vec w \cdot \x \leq
    t_2\}$.  Denote $\D^\perp=\D_B^{\perp_{\bw}}$, i.e., $\D$ is the orthogonal
    projection onto $\vec w^\perp$ of the conditional distribution on $B$, and
    consider the halfspace $f^\perp: \vec w^\perp \mapsto \{\pm 1\}$, with
    $f^\perp(\x) = \sgn(\dotp{\x}{(\vec w^\ast)^{\perp_{\vec w} }} - b)$, for
    some threshold $b \in \R$.  Moreover, define the noise function
    $ \eta^\perp(\x) = \pr_{(\vec z, y) \sim \D^\perp}[ y \neq f^\perp(\vec z) | \vec z = \x] $.
    There exist $t_1, t_2 \in \R$ multiples of $\rho$, with $|t_1-t_2|=\rho$, such that:
    \begin{enumerate}
\item  $\E_{\x \sim \D_\x^\perp}[(1-2\eta^\perp(\x))\1\{f^{\perp}(\x)\sgn(b) >0\}]\gtrsim \eps/\sqrt{\log(1/\eps)}-\rho/\eps$,
\item if $\eta^\perp(\x)>1/2$ then $0<\sgn(b)(\dotp{\x}{(\vec w^\ast)^{\perp_{\vec w} }}-b)\leq \rho/\eps$.
    \end{enumerate}
\end{lemma}

\begin{remark}
   {\em Observe that for \Cref{lem:band-projection} to give non-trivial
   guarantees we have to pick the size of the band, $\rho$, to be smaller than a
   sufficiently small constant multiple of $\eps/\sqrt{\log(1/\eps) }$.  In
   fact, in what follows, we will be using bands of size $\rho = \poly(\eps)$.
   Moreover, the condition that $\theta(\vec w, \vec w^\ast) \geq \pi - \eps$ is
   mostly technical.  Having a vector $\vec w$ with $\theta(\vec w, \vec w^\ast)
   \geq \pi - \eps$ we can always output $-\vec w$ and achieve classification
   error $O(\eps)$.  See the proof of \Cref{thm:benign}.
    }
\end{remark}
\begin{proof}

    First, let us denote $\beta(\x)=1-2\eta(\x)$. 
    Note that the assumption $\pr_{(\x, y) \sim \D}[\sign(\vec w \cdot \x) \neq y] \geq \opt + \eps$ is equivalent  to $\E_{\x\sim\D_\x}[\1\{f(\x)\neq \sign(\vec w \cdot \x)\}\beta(\x)]\geq \eps$.  We consider the region $B=\{t_1\leq \vec w \cdot \x \leq t_2\}$, for any $t_1,t_2\in \R$ with $|t_1-t_2|=\rho$. We denote by $\x^\perp$ the projection of $\x$ onto the subspace $\vec w^\perp$.
    Define the distribution $\D^\perp = \D_{B}^{\proj_{\vec w^\perp}}$, that is the distribution $\D$ conditioned on the set $B$ and projected onto $\vec w^\perp$,
    the hypothesis $f^\perp(\x^\perp) = \sgn( \dotp{\x^\perp}{\wperp}+b)$ where $b\in \R$ is chosen appropriately below, and the noise function
    $\eta^\perp(\x^\perp) = \pr_{(\vec z, y) \sim \D^\perp}[ y \neq f^\perp(\vec z) | \vec z = \x^\perp] $, see \Cref{fig:orth_proj}.
    
    Note that the distribution $\D^\perp$ does not satisfy the $1/2$-Massart noise condition
    anymore. We first illustrate how the noise function changes. The orthogonal projection on $\vec
    w^\perp$ creates a region where the $1/2$-Massart noise condition is violated, i.e., a region
    where $\eta^\perp(\x^\perp)$ may be more than $1/2$, but we can control the probability that we
    get points inside this region, i.e., the green area of \Cref{fig:orth_proj}. More formally, we
    show that
    \begin{equation} \label[ineq]{eq:agnostic-tiny}
        \E_{(\x_{\vec w},\x^\perp)\sim (\D_{B})_\x}\left[\beta(\x)\1\{f(\x_{\vec w},\x^\perp) \neq f^\perp(\x^\perp)\}\right] \leq \pr_{\x^\perp \sim \D^\perp_\x}[\eta^\perp(\x^\perp) > 1/2] \lesssim \rho/\eps  \;.
    \end{equation}
To show \Cref{eq:agnostic-tiny} notice that $\theta(\vec w, \wstar) \gtrsim \eps$, otherwise we
    would have $\pr_{\x\sim \D_\x}[f(\x)\neq \sign(\vec w\cdot \x)]\leq \eps$ from
    \Cref{fct:gaussian-halfspaces}.  We can assume that 
    $\vec w^\ast= \lambda_1 \vec w+ \lambda_2 (\vec w^\ast)^{\perp_{\vec w} }$,
    where $\lambda_1 = \cos\theta$ and $\lambda_2 = \sin \theta$. Note that $\lambda_2\gtrsim \eps$, since $\theta\in [\Omega(\eps),\pi-\eps]$.
Next we set $\vec x=(\x_{\vec w},\vec x^\perp)$, where $\x_{\vec w}=\dotp{\vec w}{\vec x}$. We
    show that the hypothesis $f^\perp(\x)=\sign((\vec w^\ast)^{\perp_{\vec w} }\cdot \vec\x^\perp +
    t_1\lambda_1/\lambda_2)$ is almost as good as the $f(\x)$ for the distribution $\D^\perp$. Note
    that in the statement of \Cref{lem:band-projection}, we write $f^\perp(\x)=\sign((\vec
    w^\ast)^{\perp_{\vec w} }\cdot \vec\x^\perp -b)$, where $b=-t_1\lambda_1/\lambda_2$.
    
    Conditioned on $\x \in B$, i.e., $\x_{\vec w} = \x \cdot \vec w \in [t_1, t_1 + \rho]$, it holds that
    $$
    \dotp{\vec w^\ast}{\vec x}
    =  \lambda_1 \x_{\vec w} + \lambda_2 \dotp{(\vec w^{\ast})^{\perp_{\vec w}}}{\vec x^\perp}
    = \lambda_1 t_1+ \lambda_2 \dotp{(\vec w^{\ast})^{\perp_{\vec w}}}{\vec x^\perp}+ s \rho
    \,,
    $$
    for some $s \in [-1,1]$ (recall that $|\lambda_1|\leq 1$).  Notice that when $0<\sgn(-\lambda_1
    t_1)(\lambda_1 t_1 + \lambda_2 \dotp{(\vec w^{\ast})^{\perp_{\vec w}}}{\vec x^\perp}) < \rho$,
    $f^\perp(\x^\perp)$ is not necessary equal to the sign of $(\dotp{\wstar}{\x})$ (recall that
    $\lambda_2 > 0$), and therefore we are inside the region that the $1/2$-Massart noise condition
    is violated, otherwise the  $f^\perp(\x^\perp)$ always matches the sign of $\wstar \cdot \x$.
    This proves the second statement of \Cref{lem:band-projection}.
    
    To prove \Cref{eq:agnostic-tiny}, we need to bound the probability of the event that
    $f^\perp(\x^\perp)$ is not necessary equal to the sign of $(\dotp{\wstar}{\x})$, or equivalently
    if $\sgn(-\lambda_1t_1)>0$ then $\dotp{(\vec w^{\ast})^{\perp_{\vec w}}}{\vec x^\perp} \in
    [(-\lambda_1 t_1 - \rho)/\lambda_2,- \lambda_1 t_1/\lambda_2] =: I^{\rho}$, and otherwise 
    $\dotp{(\vec w^{\ast})^{\perp_{\vec w}}}{\vec x^\perp} \in [-\lambda_1 t_1/\lambda_2,-
    (\lambda_1 t_1+\rho)/\lambda_2] =: I^{\rho}$.  We have that
    \begin{align*}
        \pr_{\x^\perp \sim \D^\perp_\x}[\eta^\perp(\x^\perp) > 1/2]\leq \pr_{\x^\perp \sim \D^\perp_\x}\left[
        \dotp{(\vec w^{\ast})^{\perp_{\vec w}}}{\vec x^\perp} \in I^{\rho}  \right]
        =
        \frac{ \pr_{\x \sim \D_\x}\left[ \dotp{(\vec w^{\ast})^{\perp_{\vec w}}}{\vec x} \in I^{\rho} \right]}
        {\pr_{\x \sim \D_\x}[\x \in B]}
        \lesssim \rho/\lambda_2\lesssim \rho/\eps
        \,,
    \end{align*}
    where we used the anti-concentration property of the Gaussian distribution and the last inequality holds because we have that $\lambda_2\gtrsim \eps$.
    This proves~\Cref{eq:agnostic-tiny}.
    
    It remains to prove that there exist thresholds $t_1,t_2$ and a band $B=\{t_1\leq \vec w \cdot \x \leq t_2\}$ with respect the $t_1,t_2$ such that $\E_{\x\sim \D^\perp_\x}[(1-2\eta^\perp(\x))\1\{f^\perp(\x) \neq \sgn(-b)\}]\gtrsim \frac{\eps}{\sqrt{\log(1/\eps)}}$, where $\D^\perp$ is the distribution $\D$ conditioned on the set $B$ and projected onto $\vec w^\perp$. We first show the following claim
    \begin{claim}\label{clm:random_band}
        Let $\vec v$ be a unit vector and let $S$ be a subset of $\R^d$ such
        that $\E_{\x\sim \D_\x}[\1\{\x\in S\}\beta(\x)]\geq \eps$, for some
        $\eps>0$. Fix $\rho>0$. There exist $t_1,t_2\in \R$ such that
        $t_1, t_2$ are multiples of $\rho$, $|t_1-t_2|=\rho$, and
        \[
        \E_{\x\sim\D_\x}[\1\{\x\in S\}\1\{t_1\leq\vec v \cdot \x \leq t_2\}\beta(\x)]\gtrsim \frac{\eps\rho}{\sqrt{\log(1/\eps)}}\;.
        \]
    \end{claim}
    \begin{proof}
        Note that since $\beta(\x)\1\{\x\in S\}\in[0,1]$, we have
        $\E_{\x\sim\D_\x}[\1\{\x\in S\}\1\{C\sqrt{\log(1/\eps)}\leq|\vec v \cdot
        \x |\}\beta(\x)]\leq \E_{\x\sim\D_\x}[\1\{C\sqrt{\log(1/\eps)}\leq|\vec
        v \cdot \x |\}] \leq \eps^{2}$, where $C$ is a large enough
        universal constant. 
        To get this we used the upper bound on the Gaussian tails, \Cref{fct:gaussian-tails}.
        Therefore, we can ignore the region far away from the origin
        and split the interval $|\vec v \cdot \x| \leq C \sqrt{\log(1/\eps)}$,
        into slices of length $\rho$.  There must exist one
        slice inside which the expectation of $\beta(\x) \1\{ \x \in S\}$ is
        roughly $\eps\rho/\sqrt{\log(1/\eps)}$.
        Let $t_i=i\rho$ and $t_{-i}=-i\rho$, for $0\leq
        i\leq C\sqrt{\log(1/\eps)}/\rho$, we define $B_i=\{t_i\leq\vec v \cdot
        \x \leq t_{i+1}\}$ and $B_{-i}=\{-t_{i+1}\leq\vec v \cdot \x \leq
        -t_{i}\}$. We have that
        \begin{equation}\label{eq:clm-proof}
            \sum_{|i|\leq C\sqrt{\log(1/\eps)}/\rho}\E_{\x\sim\D_\x}[\1\{\x\in
            S\}\1\{\x\in B_i\}\beta(\x)]\geq \eps - \eps^2 \geq \eps/2\;, 
        \end{equation}
        where we used that 
        $\E_{\x\sim\D_\x}[\1\{\x\in S\}\1\{C\sqrt{\log(1/\eps)}\leq|\vec v \cdot
        \x |\}\beta(\x)] \leq \eps^{2}$. Using the fact that all the
        quantities we add in \Cref{eq:clm-proof} are positive, there exists
        $B_i'$, for some $i'$, such that
        \[
        \E_{\x\sim\D_\x}[\1\{\x\in S\}\1\{\x\in B_{i'}\}\beta(\x)]\gtrsim\frac{\eps\rho}{\sqrt{\log(1/\eps)}}\;,
        \]
        which completes the proof.
    \end{proof}
    An application of \Cref{clm:random_band} to the set $\{\x\in \R^d:f(\x)\neq \sign(\vec w \cdot \x)\}$, gives us that there exists a band $B=\{t_1\leq \vec w \cdot \x \leq t_2\}$ with $|t_1-t_2|=\rho$ such that
    \[  \E_{\x\sim\D_\x}[\1\{f(\x)\neq \sign(\vec w \cdot \x)\}\1\{\x\in B\}\beta(\x)]\gtrsim\frac{\eps\rho}{\sqrt{\log(1/\eps)}}\;.\]
    Moreover, note that for the distribution $\D_B$, that is the distribution $\D$ conditioned on $B$, it holds
    $\E_{\x\sim (\D_{B})_\x}[\beta(\x)\1\{f(\x)\neq \sign(\vec w \cdot \x)\}]
    \gtrsim \frac{\eps}{\sqrt{\log(1/\eps)}}$, where we used the Gaussian anti-concentration 
    property, i.e., that $\pr_{x \sim \normal}[t_1 \leq x \leq t_2] \lesssim |t_2 - t_1|$.
    
    We have that $f^\perp(\x)$ agrees almost everywhere with $f(\x)$ with respect to the distribution $\D_{B}$, i.e., we have that $\E_{(\x_{\vec w},\x^\perp)\sim (\D_{B})_\x}[\1\{f((\x_{\vec w},\x^\perp))\neq f^\perp(\x^\perp)\}]\leq \pr_{\x^\perp \sim \D_\x^\perp}[\eta^\perp(\x^\perp)>1/2] \lesssim \rho/\eps$ (\Cref{eq:agnostic-tiny}).
    Thus, using the triangle inequality, we have
    \[ \E_{(\x_{\vec w},\x^\perp)\sim (\D_{B})_\x}\left[(1-2\eta^\perp(\x^\perp))\1\{\sign(\x_{\vec w} )\neq f^\perp(\x^\perp)\}\right] \gtrsim \frac{\eps}{\sqrt{\log(1/\eps)}}-\frac{\rho}{\eps}
    \;.\]
    The proof concludes by noting that it holds $\sign(\x_{\vec w})=\sign(-b)$, by the definition of $f^{\perp}(\x)$.
    This completes the proof of  \Cref{lem:band-projection}.
\end{proof}

\subsection{Using the Low-Order Chow Tensors}
\label[sub]{sub:chow-tensors}

In this subsection, we use our structural result of \Cref{sub:sign-mathcing-polynomial} to 
construct the sampling oracle of \Cref{prop:warm-start} that provides us with good update vectors $\vec v$.
To do so, we shall find a subspace $V$ of $\R^d$ that contains non-trivial part of $\vec v^\ast$.
In order to get a good update with non-trivial probability, we will sample a unit vector uniformly
at random from $V$. The description of the corresponding algorithm is given
in \Cref{alg:chow}.

\begin{proposition}\label{prop:nontriavial_angle}
    Let $\D$ be a distribution on $\R^{d} \times \{\pm 1\}$ with standard normal $\x$-marginal.
    Let $f(\bx) = \sgn(\dotp{\vec v^\ast}{\bx} - b)$ be such that
    $\pr_{(\x,y)\sim \D}[y \neq f(\vec x)| \vec x] = \eta(\vec x)$,
    $\E_{\x\sim \D_{\x}}[\beta(\x)\1\{\sign(b)f(\x)>0\}] \geq \zeta$ and
    $\eta(\x) > 1/2$ only when $\sgn(b)(\dotp{\vec v^\ast}{\bx} - b)\leq \xi$,
    where $\xi=\Theta(\zeta^3)$.
There exists an algorithm with sample complexity and runtime
    $d^{O(\log(1/\zeta))} \log(1/\delta)$ that with probability at least $1-\delta$
    returns a basis of a subspace $V$ of dimension $\poly(1/\zeta)$ such that 
    $\| \proj_V(\vec v^\ast) \|_2 = \poly(\zeta)$.
\end{proposition}
\begin{proof}
    
    The following lemma shows that the existence of a mean-zero polynomial $p$
    that achieves non-trivial correlation with $y$ implies that the 
    subspace spanned by the top singular vectors of the Chow tensors of $\D$
    will contain non-trivial part of $\vec v^\ast$.
    We remark that we do not rely on tensor SVD to obtain the singular vectors:
    in what follows we flatten the order-$m$ Chow tensors and treat them as 
    $d \times d^{m-1}$ matrices.  We show the following lemma.
    
    \begin{algorithm}[h]
        \caption{ Creating a Random Oracle with Good Correlation, see \Cref{prop:warm-start}. } 
        \label{alg:chow}
        
        \centering
        \fbox{\parbox{6in}{
{\bf Input:} 
                \begin{enumerate}
                    \item $\eps, \delta > 0$.
                    \item  An empirical distribution $\widehat\D$ of the distribution $\D$ that satisfies the general Massart noise condition with respect to $f(\x)=\sign(\vec w^\ast\cdot \x)$. 
                    \item A unit vector $\vec v \in \R^d$  such that 
                    $\pr_{(\x, y) \sim \D}[\sgn(\vec w \cdot \x) \neq y] \geq \opt + \eps$.
                \end{enumerate}

                {\bf Output:} A subspace $V$ with $\dim(V) = \poly(1/\eps)$, 
                and $\|\proj_V(\wperp) \|_2 \geq \poly(\eps)$. \\
                
                {\bf Define:} $k=C\log(1/\eps)$, $\sigma =\poly(\eps)/C$, $\rho = \eps^3/C$, $l=C\log(1/\eps)/\rho$, for $C>0$ be a sufficiently large constant.\\
                \begin{enumerate}
                    \item For $i\in [l]$, set $B_i=\{i\rho \leq \vec w \cdot \x \leq (i+1)\rho\}$ and $B_{-i}=\{-(i+1)\rho \leq \vec w \cdot \x \leq -i\rho\}$.
\item For $i\in \{-l,\ldots, l\}$, repeat
                    \begin{enumerate}
                        \item Let $\wh{\D}_i^\perp$ be the distribution $\widehat\D$ conditioned on $B_i$ and projected onto $\vec w^\perp$.
                        \item For any $m\in [k]$, calculate the $m$-order Chow tensor $\vec T^m(\widehat{\D}_i^\perp)$, see 
                        \Cref{def:chow-parameters}.
                        \item Flatten each $\vec T^m(\widehat{\D}_i^\perp)$ for $m\in[k]$, to a $d\times d^{m-1}$ matrix $\vec M^m = (\vec T^m(\widehat{\D}_i^\perp) )^\flat$, see \Cref{sec:prelims}.
                        \item Let $V_i^m$ be the set of the left singular vectors 
                        of $\vec M^{m}$, for all $m\in[k]$, whose singular values have absolute value at least $\sigma$.
                    \end{enumerate}
                    \item {\bf Return:} A basis of  $V = \bigcup_{i=1}^\ell \bigcup_{m=1}^k V_i^m$.
                \end{enumerate}
        }}
    \end{algorithm}

    \begin{lemma}
        \label{lem:tensor-flattening}
        Let $\D$ be a distribution over $\R^d \times \{ \pm 1\}$ with standard normal $\x$-marginal. 
        Let $p: \R \mapsto \R$ be a univariate, mean zero, unit variance polynomial of degree $k$  
        such that for some unit vector $\vec v^\ast \in \R^d$ it holds
        $\E_{(\x, y) \sim \D}[y p(\vec v^\ast \cdot \x)] \geq \tau $ for some $\tau \in (0,1]$.  
        Let ${\vec T'}^m$ be an approximation of the order-$m$ Chow-parameter tensor $\vec T^m$ of $\D$ such that
        $\| {\vec T'}^m - \vec T^m \|_F \leq \tau/(4 \sqrt{k}) $.
        Denote by $V_m$ the subspace spanned by the left singular vectors of $({\vec T'}^m)^{\flat}$ whose singular values
        are greater than $\tau/(4 \sqrt{k})$.  Moreover, denote by $V$ the union
        of $V_1,\ldots,V_k$.  Then we have that 
        \begin{enumerate}
            \item $\dim(V) \leq 4 k/\tau^2$, and
            \item  $\|\proj_{V}(\vec v^\ast) \|_2 \geq \tau/(4\sqrt{k}) $.
        \end{enumerate}
    \end{lemma}
    \begin{proof}

        We note that $\|\vec T^m \|_F = \sup_{p\in {\cal H}_k}
        \E[ y p(\x) / \sqrt{\E[p^2(\x)}] ] \leq  \sup_{p\in {\cal H}_k} \E[p^2(\x)]^{1/2}/\E[p^2(\x)]^{1/2} = 1$, where ${\cal H}_k$ is the set of polynomials formed from a linear combination of $k$-degree Hermite polynomials.
        Therefore, by the assumptions of \Cref{lem:tensor-flattening} 
        and the triangle inequality we obtain that $\|\vec T'^m\|_F \leq \|\vec T^m\|_F +1\leq 2$.
        
        Recall that by $(\vec T'^m)^\flat$ we denote the $d\times d^{m-1}$ flattening 
        of $\vec T'^m$, see \Cref{sec:prelims}.
        To simplify notation write $\vec M^m = (\vec T'^m)^\flat \in R^{d \times d^{m-1}}$.
        We have that $V_m$ is the span of the
        left singular vectors of $\vec M^m$ with singular value at least $\tau$ and $V$ be the union of all the $V_m$.
        We first show that the dimension of $V$ is not very large.
        \begin{claim}
            It holds that $\dim(V) \leq 4 k/\tau^2$.
        \end{claim}
        \begin{proof}
We have that $\| \vec M^m \|_F = \|\vec T'^{m}\|_F \leq 2$.
            Therefore, since $\| \vec M^m \|_F$ is equal to the sum of the squares of the 
            singular values 
            $\sigma_1,\ldots, \sigma_d$ of $\vec M^m$ we obtain that
            the number of singular values $\sigma_i$ with $|\sigma_i| \geq \tau$ is 
            \[
            \frac{1}{\tau^2} \sum_{i=1}^d \1\{\sigma_i^2 \geq \tau^2\} \tau^2
            \leq
            \frac{1}{\tau^2} \sum_{i=1}^d \sigma_i^2 =
            \frac{\|\vec M^m\|_F^2} {\tau^2} 
            \leq \frac{4}{\tau^2}
            \,.
            \]
            We can do the same calculation for all $\vec T^m$ to conclude that the dimension
            of the union of all subspaces is at most $4 k/\tau^2$.
        \end{proof}
        
        Next, we show that $V$ contains a non-trivial part of the optimal direction $\vec v^\ast$,
        i.e., it holds that $\|\proj_{V}(\vec v^\ast)\|_2 \geq \tau/(4 \sqrt{k})$.
        We first note that since the univariate polynomial $p$ is mean-zero, 
        it can be written as a linear combination of the univariate Hermite 
        polynomials $h_i(z)$ of degree greater than or equal to $1$.
        Write $p(z)=\sum_{i=1}^k c_i h_i(z)$.  
        We note that since $\E_{z \sim \normal}[p(z)^2] = 1$, we have that 
        $\sum_{i=1}^k c_i^2 = 1$.  
        Moreover, since the order-$i$ Chow tensor is defined as $\vec T^{i} = \E_{(\x, y) \sim \D}[\vec H^{i}(\x) y]$, it holds
        \[
        \E_{(\x,y)\sim \D}[yp(\vec v^\ast \cdot \x)] = \sum_{i=1}^k c_i \vec T^i \cdot (\vec v^\ast)^{\otimes i}  \geq \tau \,.
        \]
        Therefore, by Cauchy-Schwarz, there exists $j \in  \{1,\ldots, k\}$ such that
        $\vec T^j \cdot (\vec v^\ast)^{\otimes j} \geq \tau/\sqrt{k}$.
        Write $\vec v^\ast$ as $\vec v + \vec u$ where
        $\vec v\in V$ and $\vec u\in V^\perp$.  
        Moreover, denote $\vec r \in \R^{d^{j-1}} $ to be the flattening of 
        the tensor $\vec v^{\otimes (j-1)}$.
        We note that
        \begin{align*}
            \vec T^j \cdot (\vec v^\ast)^{\otimes j} & \leq \|\vec T^j-\vec
            T'^j\|_F + |\vec T'^j\cdot (\vec v\otimes
            (\vec v^\ast)^{\otimes j-1})| + |\vec T'^j\cdot (\vec u\otimes(\vec
            v^\ast)^{\otimes j-1})|
            \\
            & \leq 
            \|\vec T^j-\vec T'^j\|_F  +  |\vec v^\top \vec M^j \vec r|  +  |\vec u^\top \vec M^j \vec r|
            \\
            &\leq  2 \frac{\tau}{4 \sqrt{k}} + \|\vec v\|_2 \|\vec M^j \|_F 
            \leq \frac{\tau}{2 \sqrt{k}}  + 2\|\vec v\|_2\;,
        \end{align*}
        where we used that $\vec u$ belongs on the subspace spanned from the left singular vectors of $\vec M^{j}$
        whose singular values are less than $\tau/(4 \sqrt{k})$, and therefore 
        it holds $|\vec u^\top \vec M^j  \vec v| \leq \tau/(4 \sqrt{k})$. Thus, we have $\snorm{2}{\vec v}\geq \frac{\tau}{4\sqrt{k}}$ and therefore, $\|\proj_{V}(\vec v^\ast)\|_2 \geq \tau/(4 \sqrt{k})$.
        
\end{proof}
    The proof of \Cref{prop:nontriavial_angle} follows from an application of \Cref{lem:correlation-polynomial} and \Cref{lem:tensor-flattening} with the appropriate parameters, which we state below.
    To conclude the proof, we need the following lemma which bounds the number of samples needed to estimate the order-$m$ Chow parameters. Its proof can be found in \Cref{app:sec3}.
    \begin{lemma}\label{fct:estimate-chow} 
       Fix $m\in \Z_+$ and $\eps,\delta\in(0,1)$. Let $\D$ be a distribution in $\R^d\times\{\pm 1\}$ with standard normal $\x$-marginals. There is an algorithm that with $N=d^{O(m)} \log(1/\delta)/\eps^2$ samples and $\poly(d,N)$ runtime, outputs an approximation $\vec T'^m$ of the order-$m$ Chow-parameter tensor $\vec T^m$ of $\D$ such that with probability $1-\delta$, it holds
       $
       \|\vec T'^m-\vec T^m\|_F\leq \eps\;.
       $
    \end{lemma}
 Recall that 
    from our assumptions $|b|\leq \sqrt{\log(1/\zeta)}$. Moreover, from \Cref{lem:correlation-polynomial}, we have that for
     $k=O(\log(1/\zeta))$, there exists a degree-$k$ polynomial $p$, such that $\E_{(\x,y)\sim \D}[yp(\vec v^\ast \cdot \x)]\geq \poly(\zeta)$. 
     Thus, in order to apply \Cref{lem:tensor-flattening}, we need to estimate the first order-$k$ Chow tensors. Hence, the sample complexity is $N=d^{O(\log(1/\zeta))}\poly(1/\zeta)\log(1/\delta)$ and the runtime is $\poly(N,d)$.  This completes the proof of \Cref{prop:nontriavial_angle}.
    
\end{proof}

Combining \Cref{lem:band-projection} and \Cref{prop:nontriavial_angle}, we can prove \Cref{prop:warm-start}.
\begin{proof}[Proof of \Cref{prop:warm-start}]
    From the assumption that $ \pr_{(\x, y) \sim \D}[\sign(\vec w \cdot \x) \neq y] \geq \opt + \eps$,
    we have $\E_{\x\sim \D_\x}[\beta(\x) \1\{\sign(\vec w \cdot \x)\neq f(\x) \}]\geq \eps$.
    To prove this, note that  $ \pr_{(\x, y) \sim \D}[\sign(\vec w \cdot \x) \neq y]=  \E_{\x\sim \D_\x}[\beta(\x)\1\{\sign(\vec w \cdot \x)\neq f(\x) \}] +\E_{\x\sim \D_\x}[\eta(\x)]$ and $\opt=\E_{\x\sim \D_\x}[\eta(\x)]$, by combining with  $\pr_{(\x, y) \sim \D}[\sign(\vec w \cdot \x) \neq y] \geq \opt + \eps$, it follows that $$\E_{\x\sim \D_\x}[\beta(\x) \1\{\sign(\vec w \cdot \x)\neq f(\x) \}]\geq \eps\;.$$
    Let $\widehat \D_N$ be the empirical distribution of $\D$ using $N=d^{\Theta(\log(1/\eps))}\log(1/\delta)$ samples.
    Let $\rho=C\eps^3$, for some small enough constant $C>0$, and from \Cref{lem:band-projection}, we get that there exists a set
    $B=\{t_1\leq \vec v\cdot \x \leq t_2\}$, with $|t_1-t_2|=\rho$, for which we can apply the algorithm of \Cref{prop:nontriavial_angle}  to the
    distribution $\D_B^{\perp_{\vec w}}$ ($\D_B^{\perp_{\vec w}}$ is the distribution $\widehat \D_N$ conditioned on $B$ and projected into $\vec w^\perp$), and get a vector space $V$ with the following properties: (i) $ V\subseteq \vec w^\perp$, (ii) the dimension of $V$ is $\poly(1/\eps)$, and (iii)
    $\|\proj_{V}(\vec w^\ast)\|_2 = \poly(\eps\rho)=\poly(\eps)$ with probability $1-\delta$, if $N\geq d^{O(\log(1/\eps))}\log(1/\delta)$. Moreover, the algorithm runs in $\poly(N,d)$ runtime. 
    
    The problem, we need to overcome, is that we do not know the set $B$, so we are going to apply the algorithm of \Cref{prop:nontriavial_angle} to any possible set and take the union of the outputs. Let $l=\Theta(\log(1/\eps)/\rho)$ and let $t_i=i\rho$ and $t_{-i}=-i\rho$, for $0\leq i\leq l$. Furthermore, we define $B_i=\{t_i\leq\vec v \cdot \x \leq t_{i+1}\}$ and $B_{-i}=\{-t_{i+1}\leq\vec v \cdot \x \leq -t_{i}\}$ (similarly to \Cref{clm:random_band}). We then apply the algorithm of \Cref{prop:nontriavial_angle} to each of the
    distributions $\D_{B_i}^{\perp_{\vec w}}$ and get vector spaces $V_i$. Moreover, the algorithm of \Cref{prop:nontriavial_angle} returns vector spaces $V_i$ such that $V_i\subseteq \vec w^\perp$ and $\dim(V_i)=\poly(1/\eps)$. Finally, from \Cref{lem:band-projection}, we know that there exists an index $j\in\{-l,\ldots,l\}$ such that applying the algorithm \Cref{prop:nontriavial_angle} on the distribution $\D_{B_j}^{\perp_{\vec w}}$ gives a vector space $V_j$, with the additional property that $\|\proj_{V_j}(\wperp)\|_2 =\poly(\eps)$, with probability $1-\delta$.
Thus, we let $V$ to be the union of all subspaces $V_i$, i.e., $V=\bigcup_{i={-l}}^l V_i$. It holds that $\dim(V)\leq \sum_{i=1}^l\dim(V_i)=\poly(1/\eps)$. Moreover, for any $\vec v\in V$, it holds $\vec v\cdot \vec w=0$, this is because $V_i\subseteq \vec w^\perp$ and hence $V\subseteq \vec w^\perp$.
    Finally, we have that $\|\proj_{V_j}(\vec w^\ast)\|_2 =\poly(\eps)$ and because $V_j\subseteq V$ it holds that $\|\proj_{V}(\wperp)\|_2 \geq \poly(\eps)$ with probability $1-\delta$. 
\end{proof}
We now show that by finding a vector $\vec u$ that correlates well with $(\vec v^\ast)^{\perp \vec v}$, we can update our current guess vector $\vec v$ and get one with increased correlation with $\vec v^\ast$. Its proof can be found on \Cref{app:sec3}.
\begin{lemma}[Correlation Improvement]\label{lem:corr-improv}
    Fix unit vectors $\vec v^{\ast}, \vec v \in \R^d$.
    Let $\vec {u }\in \R^d$ such that
    $\dotp{\vec {u }}{\vec v^{\ast}} \geq c$, $\dotp{\vec {u }}{\vec v} = 0$,
    and  $\snorm{2}{\vec {u }}\leq 1$, with $c>0$.
    Then, for $\vec v'=\frac{{\vec v}+\lambda{\vec{ u}}}{\snorm{2}{{\vec v}+\lambda{\vec{u}} }}$,
    with $\lambda=c/2$, we have that
    $\dotp{\vec v'} {\vec v^{\ast}}\geq  \dotp{\vec v} {\vec v^{\ast}}+\lambda^2/2 $.
\end{lemma}

\begin{algorithm}[h]
    \caption{The Biased Random Walk for Learning Halfspaces with General Massart Noise, see \Cref{thm:benign}. }
    \label{alg:benign-noise}
    
    \centering
    \fbox{\parbox{6in}{
{\bf Input:} 
            \begin{enumerate}
                \item  $\eps,\delta>0$.
                \item  Sample access to $\D$ with standard normal $\x$-marginal which 
                satisfies the general Massart noise condition with respect to the hypothesis $f(\x)$.
            \end{enumerate}
            {\bf Output:}  A hypothesis  $h(\x)$ such that 
            $\pr_{(\x,y)\sim\D}[h(\x)\neq y]\leq \opt +\eps$, with probability $1-\delta$.
            \\
            
            {\bf Define:} $L=\{\}$, $N=d^{C\log(1/\eps)}\log(1/\delta)$, $\lambda=\poly(\eps)/C$, $T=2^{C\poly(1/\eps)}$, for $C>0$ a sufficiently large constant. 
            \begin{enumerate}
                \item Initialize $\vec w = \vec e_1$.
                \item  Draw $N$  samples and compute to be the empirical distribution  $\widehat{D}$.
                \item Repeat $T$ times
                \begin{enumerate}
                    \item Using $\vec w$ on \Cref{alg:chow}, generate vector space $V$.
                    \item Pick a random unit vector $\vec v\in V$.
                    \item  Update current hypothesis $\vec w$ by
                    $ \vec w \gets \frac{\vec w + \lambda \vec v}{\|\vec w + \lambda \vec v\|_2}\;.$
                    \item Update list of vectors $L\gets L\cup \{\vec w\}$.
                \end{enumerate}
                \item {\bf Return:} $h(\x)=\sign(\hat{\vec w} \cdot \x)$, where $\hat{\vec w} \in S$ and minimizes the error with respect $y$, i.e.,
                $$ \hat{\vec w}\gets \argmin_{\vec w\in L} \pr_{(\x,y)\sim \widehat{\D}}[\sign(\vec w \cdot \x)\neq y]\;.$$
            \end{enumerate}
            
    }}
    
\end{algorithm}

We need the following standard fact that bounds from below the correlation of any vector with a random one.
\begin{fact}[see, e.g., Remark 3.2.5 of \cite{Ver18}]\label{fct:random_initialization}
    Let $\vec v$ be a unit vector in $\R^d$. For a random unit vector $\vec u\in \R^d$, with at least $1/3$ probability,
    it holds $\dotp{\vec v}{\vec u} \gtrsim 1/\sqrt{d}$.
\end{fact}
We are now ready to prove the \Cref{thm:benign}.
\begin{proof}[Proof of \Cref{thm:benign}]
    First, let $\widehat\D_N$ be the empirical distribution of $\D$ using $N\in \Z_+$ samples. Let
    $\vec w^{(i)}$ be the current guess. If $\pr_{(\x,y)\sim \D}[\sign(\vec
    w^{(i)}\cdot \x)\neq y]\geq \opt +\eps$ and $\theta(\vec w\ith,\wstar \leq \pi-\eps)$, then from \Cref{prop:warm-start}  with $N_1=d^{O(\log(1/\eps))}\log(1/\delta_1)$
    samples from $\widehat \D$ and $\poly(N_1,d)$ time, we compute a subspace $V$ such that $\|\proj_V(\vec w^\ast)\|_2\geq \poly(\eps)$ and from \Cref{fct:random_initialization}, we get a random a unit vector $\vec v\in V$
    such that  $\vec v \cdot \wstar =\poly(\eps)$ and $\vec v \cdot \vec w^{(i)}=0$,
     with probability $(1-\delta_1)/3$. We call this
    event $\mathcal{E}_i$. By conditioning on the event $\mathcal{E}_i$; from \Cref{lem:corr-improv},
    after we update our current hypothesis vector $\vec w^{(i)}$ with $\vec v$, we get a unit vector $\vec
    w^{(i+1)}$ such that $\vec w^{(i+1)} \cdot \wstar \geq \vec w^{(i+1)} \cdot
    \wstar + \poly(\eps)$.
    
    After running the update step $k$ times and conditioning on the events
    $\mathcal{E}_1,\ldots, \mathcal{E}_k$, then $\vec w^{(k)} \cdot \wstar \geq
    \vec w^{(1)} \cdot \wstar + k\,\poly(\eps)$; therefore, for $k=\poly(1/\eps)$, we get
    that the vector $\vec w^{(k)}$ that is competitive with the optimal hypothesis, i.e.,
    $\pr_{(\x,y)\sim \D}[\sign(\vec w^{(k)}\cdot \x)\neq y]\leq
    \opt +\eps$ or $\theta(\vec w\ith,\wstar \in (\pi-\eps,\pi)$ which means that $- \vec w^{(k)}$ is competitive with the optimal hypothesis, see \Cref{fct:gaussian-halfspaces}. The probability that all the events $\mathcal{E}_1,\ldots,
    \mathcal{E}_k$ hold simultaneously is at least $(1-k\delta_1)+(1/3)^k$, and thus by choosing $\delta_1\leq 1/(3k)$, the
    probability of success is at least $\delta_2=(1/3)^k$. By running the algorithm above $M=\log(1/
    \delta) /\delta_2$ times, we get a list of $2M$ vectors, that list contains all the $\vec w\ith$ that generated in every step of the algorithm and the $-\vec w\ith$. By applying Hoeffding's inequality, we get that the
    list $L$ of $2M$ vectors contains a unit vector $\vec w$ such that
    $\pr_{(\x,y)\sim \D}[\sign(\vec w^{(k)}\cdot \x)\neq y]\leq \opt +\eps$, with
    probability $1-\delta/2$.  Finally, to evaluate all the vectors from the list,
    we need a few samples, from the distribution $\D$ to obtain the best
    among them, i.e., the one that minimizes the zero-one loss.
    
    The size of the list of candidates is at most $M\leq 2^{\poly(1/\eps)}\log(1/\delta)$.  Therefore,
    from Hoeffding's inequality, it follows that $ O(\poly(1/\eps)\log(1/\delta))$
    samples are sufficient to guarantee that the excess error of the chosen
    hypothesis is at most $\eps$ with probability at least $1-\delta/2$. Thus, with
    $N=d^{\log(1/\eps)}\log(1/\delta)$ samples and $\poly(N,d,
    2^{\poly(1/\eps)})$ runtime, we get a hypothesis $\hat{\vec w}$ such that
    $\pr_{(\x,y)\sim \D}[\sign(\hat{\vec w}\cdot \x)\neq y]\leq \opt +\eps$ with
    probability $1-\delta$. This
    completes the proof.
\end{proof}
 \newpage
 \section{Statistical Query Lower Bounds for Learning Massart Halfspaces}
\label{sec:massart-biased-lower-bound}
\subsection{Background on SQ Lower Bounds}
\label{ssec:SQ-prelims}

Our lower bound applies for the class of Statistical Query (SQ) algorithms.
Statistical Query (SQ) algorithms are a class of algorithms that are allowed
to query expectations of bounded functions of the underlying distribution
rather than directly access samples. Formally, an SQ algorithm has access to the following oracle.

\begin{definition} \label{def:stat-oracle}\label{def:stat}
Let $\cal{D}$ be a distribution on labeled examples supported on $X \times \{-1, 1\}$, for some domain $X$.
A statistical query is a function $q: X \times \{-1, 1\} \to [-1, 1]$.  We define
\textsc{STAT}$(\tau)$ to be the oracle that given any such query $q(\cdot, \cdot)$
outputs a value $v$ such that $|v - \E_{(\vec x, y) \sim \cal{D}}\left[q(\vec x, y)\right]| \leq \tau$,
where $\tau>0$ is the tolerance parameter of the query.
\end{definition}

The SQ model was introduced by Kearns~\cite{Kearns:98} in the context of supervised learning as a
natural restriction of the PAC model~\cite{Valiant:84}.  Subsequently, the SQ model has been
extensively studied in a plethora of contexts (see, e.g.,~\cite{Feldman16b} and references therein).
The class of SQ algorithms is rather broad and captures a range of known supervised learning
algorithms.  More broadly, several known algorithmic techniques in machine learning are known to be
implementable using SQs. These include spectral techniques, moment and tensor methods, local search
(e.g., Expectation Maximization), and many others (see, e.g.,~\cite{FeldmanGRVX17, FeldmanGV17}).
Recent work~\cite{BBHLS21} has shown a near-equivalence between the SQ model and low-degree
polynomial tests

\paragraph{Statistical Query Dimension}
To bound the complexity of SQ learning a concept class $\cal C$,
we use the SQ framework for problems over distributions~\cite{FeldmanGRVX17}.
\begin{definition}[Decision Problem over Distributions] \label{def:decision}
    Let $\D$ be a fixed distribution and $\mathfrak D$ be a family of distributions.
    We denote by $\mathcal{B}(\mathfrak D, \D)$ the decision (or hypothesis testing) problem
    in which the input distribution $\D'$ is promised to satisfy either
    (a) $\D' = \D$ or (b) $\D' \in \mathfrak D$, and the goal
    is to distinguish between the two cases.
\end{definition}

\begin{definition}[Pairwise Correlation] \label{def:pc}
    The pairwise correlation of two distributions with probability density functions
    $\D_1, \D_2 : \R^n \to \R_+$ with respect to a distribution with
    density $\D: \R^n \to \R_+$, where the support of $\D$ contains
    the supports of $\D_1$ and $\D_2$, is defined as
    $\chi_{\D}(\D_1, \D_2) \eqdef \int_{\R^n} \D_1(\bx) \D_2(\x)/D(\bx)\, \d\bx - 1$.
\end{definition}

\begin{definition} \label{def:uncor}
    We say that a set of $s$ distributions $\mathfrak{D} = \{\D_1, \ldots , \D_s \}$
    over $\R^n$ is $(\gamma, \beta)$-correlated relative to a distribution $\D$
    if $|\chi_\D(\D_i, \D_j)| \leq \gamma$ for all $i \neq j$,
    and $|\chi_\D(\D_i, \D_i)| \leq \beta$ for all $i$.
\end{definition}

\begin{definition}[Statistical Query Dimension] \label{def:sq-dim}
    For $\beta, \gamma > 0$ and a decision problem $\mathcal{B}(\mathfrak D, \D)$,
    where $\D$ is a fixed distribution and $\mathfrak D$ is a family of distributions,
    let $s$ be the maximum integer such that there exists a finite set of distributions
    $\mathfrak{D}_\D \subseteq \mathfrak D$ such that
    $\mathfrak{D}_\D$ is $(\gamma, \beta)$-correlated relative to $\D$
    and $|\mathfrak{D}_\D| \geq s.$ The {\em Statistical Query dimension}
    with pairwise correlations $(\gamma, \beta)$ of $\mathcal{B}$ is defined to be $s$,
    and denoted by $\mathrm{SD}(\mathcal{B},\gamma,\beta)$.
\end{definition}

\begin{lemma}[Corollary 3.12 of \cite{FeldmanGRVX17}] \label{lem:sq-from-pairwise}
    Let $\mathcal{B}(\mathfrak D, \D)$ be a decision problem, where $\D$ is the reference distribution
    and $\mathfrak{D}$ is a class of distributions. For $\gamma, \beta >0$,
    let $s= \mathrm{SD}(\mathcal{B}, \gamma, \beta)$.
    For any $\gamma' > 0$, any SQ algorithm for $\mathcal{B}$ requires queries of tolerance at most $\sqrt{\gamma + \gamma'}$ or makes at least
    $s  \gamma' /(\beta - \gamma)$ queries.
\end{lemma}

We next introduce some definitions related to the hidden-direction proof machinery that we use.
We start with the following definition:

\begin{definition} [High-Dimensional Hidden Direction Distribution] \label{def:pv-hidden}
    {For a distribution $A$ on the real line with probability density function $A(z)$ and}
    a unit vector $\vec v \in \R^d$, consider the distribution over $\R^d$ with probability density function
    \[ 
    \p^A_{\vec v}(\x) = A(\vec v \cdot \x) \exp\left(-\|\x - (\vec v \cdot \x) \vec v\|_2^2/2\right)/(2\pi)^{(d-1)/2}\,.
    \]
    That is, $\p_{\vec v}$ is the product distribution whose orthogonal projection onto the direction of $\vec v$ is $A$,
    and onto the subspace perpendicular to $\vec v$ is the standard $(d-1)$-dimensional normal distribution.
\end{definition}

Since we will be using mixtures of two hidden direction distributions we introduce the following notation.
\begin{definition}[Mixture of Hidden Direction Distributions]
    \label{def:hid-distr}
    Let $A$ and $B$ be distributions on $\R$. For $d \in \Z_+$ and a unit vector $\vec v\in \R^d$,
    define the distribution $\p^{A,B,p}_{\vec v}$ on $\R^d\times \{\pm 1\}$ that returns a sample
    from $(\p^A_{\vec v},1)$ with probability $p$ and a sample from $(\p^B_{\vec v},-1)$ with probability $1-p$.
\end{definition}

We will also use the following fact showing that there exists an exponentially large set of $d$-dimensional 
unit vectors all of which have small correlation.  Moreover, one can show that a set of random vectors 
on the unit sphere will satisfy this property with nontrivial probability.
\begin{fact}[Lemma~3.7 of \cite{DKS17-sq}]\label{fct:near-orth-vec}
    For any constant $0<c<1/2$ there exists a set $S$ of $2^{\Omega(d^c)}$ unit vectors in $\R^d$ such that any pair $\vec u, \vec v \in S$, with $\vec u \neq \vec v$, satisfies $|\vec u\cdot \vec v|\lesssim d^{c-1/2}$.
\end{fact}

The following fact shows that given a one-dimensional marginal $A$ 
that matches $m$ moments with the standard normal, 
the correlation between two hidden direction distributions with
directions  $\vec v$ and $\vec u$ is bounded above roughly by the 
$(m{+}1)$-th power of the correlation of their corresponding directions.
Therefore, using \Cref{fct:near-orth-vec}, we obtain that there exists an exponentially
large (in the dimension $d$) set of distributions with pairwise correlation roughly $d^{-m}$.

\begin{fact}[Lemma~3.4 from~\cite{DKS17-sq}] \label{fct:cor}
    Let $m \in \Z_+$.
    If the univariate distribution $A$ over $\R$ agrees with the first $m$ moments of $\normal(0,1)$,
    then for all $\vec v,\vec u \in \R^d$, we have that
    \begin{equation*}
        |\chi_{\normal(\vec 0,\vec I)}(\p^A_{\vec v}, \p^A_{\vec u})| \leq |\vec v \cdot \vec u|^{m+1} \chi^2(A, \normal(0,1))\;.
    \end{equation*}
\end{fact}

\subsection{SQ Lower Bound for Learning Halfspaces with Constant-Bounded Massart Noise}\label[sub]{sub:massart-lower-bound}

The problem of learning homogeneous halfspaces with Massart Noise under the Gaussian distribution
is by now well understood.  All previous algorithms fit in the SQ framework and show that the SQ complexity
of learning halfspaces with Massart noise is polynomial in the dimension $d$, the accuracy $\eps$, and the
noise rate $\eta$.  
In this section, we show that if the optimal halfspace $f$ is $\gamma$-biased, i.e., 
$\pr[f(\x) = +1] = \gamma$, the SQ complexity of learning $f$ is quasi-polynomial in the bias $\gamma$, 
that is $d^{\Omega(\log(1/\gamma))}$ SQ queries are required.  We prove the following theorem.

\begin{theorem}\label{thm:lower_bound}
    Let $\D$ be a distribution on $\R^d\times\{\pm1\}$ with standard normal $\x$-marginal that
    satisfies the $\eta$-Massart noise condition with parameter $\eta \in (0, 1/2)$  with respect to
    some unknown $(1-\gamma)$-biased optimal halfspace $f(\x)$, for some $\gamma>0$ less than a
    sufficiently small constant.  Any SQ algorithm that, for any such distribution $\D$, learns a
    hypothesis $h:\R^d \mapsto \{\pm 1\}$ such that $\pr_{(\x,y)\sim \D}[h(\x)\neq y]\leq \opt+\eps$
    for $\eps \leq \gamma (1- 2 \eta)$, either requires queries with tolerance at most
    $d^{-\Omega(\log(\eta/ \gamma))}$ or makes at least $2^{d^{\Omega(1)}}d^{-\log(\eta/ \gamma)}$
    queries.
\end{theorem}
\begin{remark}
    {\em We remark that when the bias $\gamma$ is less than $\eps$ one can output the constant guess $-1$ and obtain 
    error at most $\eps$.  Therefore, reasonable values for $\gamma$ are $\Omega(\eps)$.  In the extreme case where 
    $\gamma = \Theta(\eps)$, \Cref{thm:lower_bound} implies a quasi-polynomial SQ lower bound in the accuracy 
    parameter $\eps$.  However, our result is fine-grained with respect to $\gamma$: 
    we show that this quasi-polynomial dependency on the bias is required across the whole regime of $\gamma$.}
\end{remark}

The main structural result of the proof of \Cref{thm:lower_bound} is the proposition that follows.
We prove that we can construct a distribution $\D$ over labeled pairs $(z,y) \in \R \times \{\pm 1\}$ that satisfies the $\eta$-Massart noise assumption with respect to some (biased) halfspace $f$ and whose low-order moments match the moments of
the product distribution of the marginals $\D_z$ and $\D_y$.  In other words, we show that we can construct
and instance $\D$ whose $z$-marginal is a one dimensional Gaussian distribution and $y$ is uncorrelated with 
$z^k$ for any $k$ less than a sufficiently small multiple of $\log(\eta/\gamma)$.

\begin{proposition}\label{pro:construction}
    Fix $\eta \in (0, 1/2)$ and $\gamma>0$ less than a sufficiently small constant.
    There exists a distribution $\D$ on $(z,y) \in \R\times\{\pm 1\}$ whose $z$-marginal is the standard 
    normal distribution with the following properties.
    \begin{itemize}
        \item  $\D$ satisfies the Massart noise condition  with parameter $\eta$
        with respect to a halfspace $f(z)$ with $\pr_{z\sim \D_z}[f(z)=1]=\gamma$.
        
        \item  For any integer $k \leq C\log(\eta/\gamma)$ where $C>0$ is a sufficiently small constant, it holds 
        \(
        \E_{(z, y) \sim \D}[y z^k] =  \E_{y \sim \D_y}[y]  \E_{z \sim \D_z}[z^k]  \,.
        \)
    \end{itemize}
\end{proposition}
\begin{proof}
    Our goal is to construct a distribution $\D$ on $\R \times \{\pm 1\}$ 
    with Gaussian $z$-marginal satisfying the $\eta$-Massart noise condition.  
    Recall that, for any such distribution, we denote $\beta = 1- 2 \eta$ and $\beta(z) = 1- 2 \eta(z)$.
    It holds 
    \[
    \E_{(x, y) \sim \D}[y | x = z]  =  - f(z) \eta(z) + f(z) ( 1- \eta(z))
    = f(z) (1-2 \eta(z)) = f(z) \beta(z) \,.
    \]
    Therefore, we need to prove that there exists a ``noise'' function $\beta(z): \R \mapsto [\beta, 1]$ 
    such that for every zero mean polynomial $p(z)$ of degree at most $k$, it holds 
    $ \E_{(z, y) \sim \D}[p(z) y] = 
    \E_{z \sim \normal}[p(z) \beta(z) f(z) ] =  0$.
    Recall that by ${\cal P}_{k, d}$ we denote the space of polynomials of degree at most $k$ over $\R^d$.
    In what follows, we will be using polynomials on the subspace of univariate 
    mean-zero polynomials (with respect  to the Gaussian measure) which we denote by
    \[
    {\cal P}_{k}^0 \eqdef \{ p \in {\cal P}_{k,1} : \E_{z \sim \normal}[p(z)] = 0 \} \,.
    \]
    Using duality, we first show that such a noise function $\beta(z)$
    exists when there exists no mean-zero polynomial $p$ such that the expectation
    of $(f(z) p(z))^+$ is $1/\beta$ times larger than the expectation of $(f(z) p(z))^{-}$,
    where $z^+ \eqdef \max(0, z)$ and $z^{-} \eqdef |\min(0, z)|$.
    We prove the following lemma.
    \begin{lemma}[Moment-Matching Duality] \label{lem:duality}
        Let $f: \R \mapsto \{\pm 1\}$ be any one-dimensional Boolean function.
        Assume that for any polynomial $p \in {\cal P}_k^0$, with $p\neq 0$, it holds that 
        \[
        \beta \E_{z \sim \normal}[ (f(z) p(z))^+ ] <
        \E_{z \sim \normal}[ (f(z) p(z))^- ] \,.
        \]
        Then, there exists a function $\beta(z): \R \mapsto \R$ such that $\pr_{z \sim \normal}[\beta \leq \beta(z) \leq 1] = 1$
        and for every polynomial $p \in {\cal P}_k^0$ it holds that $\E_{z \sim \normal}[ f(z) p(z) \beta(z) ] = 0$.
    \end{lemma}
    \begin{remark}
        {\em Even though we used the assumption that $f$ is a one-dimensional
        function, this statement is true for any $m$-dimensional function  as
        long as $f$ is biased enough. Therefore, using similar techniques as in
        $\cite{DKPZ21}$, we can embed this $m$-dimensional subspace 
        to $d$ dimensions and get lower bounds for learning more general classes. 
        In fact, for $m=2$, one can show an SQ lower bound for intersections of two
        homogeneous halfspaces under constant-bounded Massart noise.}
    \end{remark}
    \begin{proof}
        Treating the function $\beta(z)$ as an (infinite dimensional) variable  
        we can formulate the following feasibility linear program.
        \begin{alignat}{2}
            \text{Find }    \qquad  &~~ \beta(z) \in L^\infty(\R)                         && \notag\\ 
            \text{such that} \qquad & \E_{z \sim \normal}[f(z)p(z)\beta(z)]  = 0   && \qquad \forall p \in{\cal P}_{k}^0  \label[lp]{eq:primal-LP-1}\\
            &  \pr_{z \sim \normal}[\beta \leq \beta(z) \leq 1] = 1                \notag
        \end{alignat}
        We claim that the \Cref{eq:primal-LP-1} is equivalent to the following LP:
\begin{alignat}{2}
            \text{Find }    \qquad  &~~ \beta(z) \in L^\infty(\R)                         && \notag\\ 
            \text{such that} \qquad & \E_{z \sim \normal}[f(z)p(z)\beta(z)]  = 0   && \qquad \forall p \in{\cal P}_{k}^0  \label[lp]{eq:primal-LP-2}\\
            &  \E_{z\sim \normal}[\beta(z)h(z)] \leq \E_{z\sim \normal}[h(z)]                   && \qquad \forall h \in L_+^1(\R) \notag \\
            \beta &\E_{z\sim \normal}[T(z)]\leq \E_{z\sim \normal}[\beta(z)T(z)]  &&\qquad\forall T \in L_+^1(\R) \notag 
        \end{alignat}
        Recall that $L^1(\R)$ are all the functions that have bounded $L_1$-norm, and we denote $L_+^1(\R)$ to be the positive functions in $L^1(\R)$.
        \begin{claim}
            The \Cref{eq:primal-LP-1} is equivalent to the \Cref{eq:primal-LP-2}.
        \end{claim}
        \begin{proof}
            We now show the equivalence between the two formulations. We claim
            that the second constraint of \Cref{eq:primal-LP-1} is equivalent with the second and the third constraints of   \Cref{eq:primal-LP-2}.
            This follows by introducing the ``dual variables'' $h:\R \to \R$ and $T:\R\to \R$. First, we so that any valid solution $b(z)$ of \Cref{eq:primal-LP-1} is also a valid solution for the \Cref{eq:primal-LP-2}.
             For any valid solution $b(z)$ of \Cref{eq:primal-LP-1}, it should hold that $b(z)\leq 1$ and $b(z)\geq \beta$ with probability 1, thus for $h(z)\in L_+^1(\R)$ it should hold that $b(z) h(z)\leq h(z)$ and by taking expectation in both sides, we get the third inequality, similarly we get the fourth inequality from $b(z)\geq \beta$.
              To prove that any valid solution $b(z)$ of \Cref{eq:primal-LP-2} is also a valid solution for the \Cref{eq:primal-LP-1}, first we assume, in order to reach a contradiction, that there is a set $A$ with non-zero probability that $b(z)>1$. Then, by taking $h(z)=\1\{z\in A\}$, the third inequality becomes 
            \[ \E_{z\sim \normal}[h(z)] <\E_{z\sim \normal}[\beta(z)h(z)] \leq \E_{z\sim \normal}[h(z)] \;, \]
            which is a contradiction. For the case, we assume, to reach a contradiction, that there is a set $B$ with non-zero probability that $b(z)<\beta$. Similarly, we have $T(z)=\1\{z\in B\}$ and get that $  \beta \E_{z\sim \normal}[T(z)]\leq \E_{z\sim \normal}[\beta(z)T(z)]<\beta \E_{z\sim \normal}[T(z)]$, which is again a contradiction.
        \end{proof}

        At this point, we would like to use ``LP duality'' to argue that \Cref{eq:primal-LP-2} is feasible if and only
        if its ``dual LP'' is infeasible. While such a statement turns out to be true, it requires some care to prove
        since we are dealing with infinite LPs (both in number of variables and constraints). More formally, we have that \Cref{eq:primal-LP-2} is feasible if and only if there is no conical combination that yields the contradicting inequality $1\leq 0$.
We define the ``dual LP'' to be the following:
        \begin{alignat}{2}
            \text{Find }    \qquad  & h\in L_+^1(\R), T\in L_+^1(\R), p\in{\cal P}_{k}^0                          && \notag\\ 
            \text{such that} \qquad & f(z)p(z) + h(z) -T(z)  =0 &&\qquad\forall z \in\R  \label[lp]{eq:primal-LP-3}\\
            &\beta\E_{z\sim \normal}[T(z)] -\E_{z\sim \normal}[h(z)] > 0 \notag 
        \end{alignat}
        The following lemma states that the sufficiently conditions so that \Cref{eq:primal-LP-2} is feasible. Its proof can be found on \Cref{app:lower-bounds}.
        \begin{lemma}\label{lem:feasibility}
            If there is no polynomial $p\in \mathcal P_k^0$ such that $\beta\E_{z\sim \normal}[(f(z)p(z))^+]  > \E_{z\sim \normal}[(f(z)p(z))^-] $ then, the \Cref{eq:primal-LP-2} is feasible if only if \Cref{eq:primal-LP-3} is infeasible.
        \end{lemma}

Note that  if \Cref{eq:primal-LP-3} is feasible for some functions  $(h,T,p)$, then it should hold that $h(z)=(f(z)p(z))^-$ and $T(z)=(f(z)p(z))^+$.
        Therefore, we claim that the \Cref{eq:primal-LP-3} is equivalent to the following LP.
\begin{alignat}{2}
            \text{Find }   \qquad &  p\in{\cal P}_{k}^0                         && \notag\\ 
            \text{such that} \qquad&\beta\E_{z\sim \normal}[(f(z)p(z))^+]  > \E_{z\sim \normal}[(f(z)p(z))^-]   \label[lp]{eq:primal-LP-4}
        \end{alignat}
        Observe that if \Cref{eq:primal-LP-4} is infeasible then from \Cref{lem:feasibility}, we get that \Cref{eq:primal-LP-2} is feasible, therefore the proof of \Cref{lem:duality} follows.
    \end{proof}
    
    \begin{lemma}\label{lem:infisibility-dual}
        Let $f: \R \mapsto \{\pm 1\}$ be any one-dimensional Boolean function, $\beta\in(0,1)$ and $k\in \Z_+$. If $\pr_{z\sim \normal}[f(z)=1]\leq 2^{-Ck}(1-\beta)$ for some sufficiently large constant $C>0$, then  for any polynomial $p\in {\cal P}_k^0$, it holds $\beta\E_{z\sim \normal}[(f(z)p(z))^+]  \leq \E_{z\sim \normal}[(f(z)p(z))^-]$.
    \end{lemma}
    \begin{proof}
        First, notice that, since $\E_{z\sim \normal}[p(z)]=0$, it holds that 
        $\E_{z\sim \normal}[p(z)^+]=\E_{z\sim \normal}[p(z)^-]$.
        Moreover, we have that for any $z \in \R$ 
        we either have $f(z) = +1$ or $f(z) = -1$ and therefore, it holds that $f(z)^+ +f(z)^-=1$.
        Thus, $\E_{z\sim \normal}[p(z)^+ f(z)^+]+\E_{z\sim \normal}[p(z)^+ f(z)^-]=\E_{z\sim \normal}[p(z)^-f(z)^+] +\E_{z\sim \normal}[p(z)^-f(z)^-]$.
        We have that
        \begin{align*}
            \E_{z\sim \normal}[(f(z)p(z))^+]  &= 
            \E_{z\sim \normal}[f(z)^+p(z)^+] +\E_{z\sim \normal}[f(z)^-p(z)^-] \,,
            \\
            \E_{z\sim \normal}[(f(z)p(z))^-]  &= 
            \E_{z\sim \normal}[f(z)^+p(z)^-] +\E_{z\sim \normal}[f(z)^-p(z)^+] \,.
        \end{align*}
        Therefore, \cref{eq:primal-LP-4} implies the following inequality
        \begin{align}
            2\beta\E_{z\sim \normal}[f(z)^+p(z)^+] & > (1+\beta)\E_{z\sim \normal}[f(z)^+p(z)^-] +(1-\beta)\E_{z\sim \normal}[f(z)^-p(z)^+]\nonumber\\
            &\geq(1-\beta)\E_{z\sim \normal}[f(z)^-p(z)^+] \label{eq:feasiblity-1}\;,
        \end{align}
        where we used the fact that $\E_{z\sim \normal}[f(z)^+p(z)^-]\geq 0$.
        Notice that $\E_{z\sim \normal}[|p(z)|]=2\E_{z\sim \normal}[p(z)^+]=2\E_{z\sim \normal}[p(z)^-]$, because $p(z)$ is a zero mean polynomial. Moreover, using $f(z)^-=1-f(z)^+$, it holds $\E_{z\sim \normal}[f(z)^-p(z)^+]=\E_{z\sim \normal}[|p(z)|]/2-\E_{z\sim \normal}[f(z)^+p(z)^+]$. Thus, by substituting this in the \Cref{eq:feasiblity-1}, we get
        \begin{align*}
            2\frac{1+\beta}{1-\beta}\E_{z\sim \normal}[f(z)^+p(z)^+] & > \E_{z\sim \normal}[|p(z)|]\;,
        \end{align*}
        or $4\E_{z\sim \normal}[f(z)^+p(z)^+]/(1-\beta)  > \E_{z\sim \normal}[|p(z)|]$, where we used that $\beta<1$.
        From Cauchy–Schwarz inequality, we have
        \begin{align*}
            \E_{z\sim \normal}[f(z)^+p(z)^+] & \leq  \E_{z\sim \normal}[f(z)^+|p(z)|]\leq (\E_{z\sim \normal}[f(z)^+])^{1/2}(\E_{z\sim \normal}[|p(z)|])^{1/2}\;.
        \end{align*}
        Using Cauchy–Schwarz, it holds that $(\E_{z\sim \normal}[p(z)^2])^{1/2}\leq (\E_{z\sim \normal}[|p(z)|])^{1/3}(\E_{z\sim \normal}[p(z)^4])^{1/6}$ and from \Cref{lem:hypercontractivity}, we have that $(\E_{z\sim \normal}[p(z)^4])^{1/4}\leq 3^{k/2}(\E_{z\sim \normal}[p(z)^2])^{1/2}$.
        Putting everything together, we have that $(\E_{z\sim \normal}[p(z)^2])^{1/2}\leq 2^{Ck}\E_{z\sim \normal}[|p(z)|]$, for $C$ some large enough positive constant, thus
        \begin{align*}
            \E_{z\sim \normal}[|p(z)|]\leq\frac{4}{1-\beta} \E_{z\sim \normal}[f(z)^+p(z)^+] \leq \frac{4}{1-\beta}(\pr[f(z)=1])^{1/2}2^{C k}\E_{z\sim \normal}[|p(z)|]\;,
        \end{align*}
        which is a contradiction if $\pr[f(z)=1]\leq 2^{-C k}(1-\beta)$ for large enough positive constant $C$. Thus, the \Cref{eq:primal-LP-4} is infeasible and thus, the \Cref{eq:primal-LP-1} is feasible.
    \end{proof}
    
    Using \Cref{lem:duality} and \Cref{lem:infisibility-dual}, we get that if $\gamma=\pr_{z\sim \normal}[f(z)=1]\leq 2^{-Ck}(1-\beta)$ 
    for some large enough constant $C>0$, then there exists a function $\beta:\R\mapsto \R$ such that for every polynomial $p\in {\cal P}_k^0$
    it holds $\E_{z\sim\normal}[f(z)p(z)\beta(z)]=0$, thus, for any $m\leq k$, we have $$\E_{(z,y)\sim \D}[yz^m]=\E_{(z,y)\sim \D}[y(z^m-\E_{z\sim \D_z}[z^m])]+\E_{(z,y)\sim \D_y}[y]\E_{z\sim \D_z}[z^m]=\E_{(z,y)\sim \D_y}[y]\E_{z\sim \D_z}[z^m]\;.$$ Moreover, using the fact that $\beta=1-2\eta$, we get that the degree $k$ is less than a sufficiently small multiply of $\log(\eta/\gamma)$, which completes the proof of \Cref{pro:construction}.

\end{proof}

\begin{proposition}[SQ Complexity of Hypothesis Testing]\label{prop:sq-test-massart}
    Fix $\eta  \in (0,1/2)$ and $\gamma>0$ less than a sufficiently small constant. 
    There exist:
    \begin{itemize}
        
        \item  a family of distributions $\mathfrak{D}$ such that every $\D \in \mathfrak{D}$ 
        is a distribution over $(\x,y) \in \R^d \times \{\pm 1\}$, its $\x$-marginal is the standard normal
        distribution, and $\D$ satisfies the $\eta$-Massart noise condition with respect to some $\gamma$-biased halfspace, and
        
        \item  a reference distribution $\cal R$  over $(\x,y) \in \R^d \times \{\pm 1\}$, whose $\x$-marginal is the standard normal and $y$ is independent of $\x$,
    \end{itemize}
    such that any SQ algorithm that decides whether the input distribution belongs to $\mathfrak{D}$ or is equal 
    to the reference $\cal R$ either requires queries with tolerance at most $d^{-\Omega(\log(\eta/\gamma))}/\sqrt{\gamma}$
    or makes at least $2^{d^{\Omega(1)}}d^{-\log(\eta/\gamma)}$ queries.
\end{proposition}
\Cref{prop:sq-test-massart} follows from the lemma below showing that we can construct 
a family of distributions with small pairwise correlation 
\begin{lemma}[Correlated Family of Distributions]
    \label{lem:hard-dist-massart}
    Fix $t \in \R$ such that $|t|$ is larger than some absolute constant.
    Define $\gamma(t) \eqdef \pr_{z\sim\normal}[z\geq |t|]$.  There exist: 
    \begin{itemize}
        \item 
        a set $\mathfrak{D}_t$ of $2^{d^{\Omega(1)}}$ distributions on $\R^d\times\{\pm 1\}$, such that
        every $\D\in \mathfrak{D}_t$ is a distribution over $(\x,y) \in \R^d \times \{\pm 1\}$, its $\x$-marginal is the standard normal distribution, and $\D$ satisfies the $\eta$-Massart condition with $\eta\in(0,\frac{1}{2})$ with respect to a halfspace $f(\x)=\sign(\vec v\cdot \x+t)$ where $\vec v$ is a unit vector in $\R^d$.
        Moreover, all $\D \in \mathfrak{D}_t$ have the same $y$-marginal,
        \item a reference distribution ${\cal R}_t$ in $\R^d\times\{\pm 1\}$,  where for $(\x, y) \sim {\cal R}_t$ we have that
        $\x$ is distributed according to the standard normal $ \normal(\vec 0, \vec I)$, $y$ is independent of $\x$, 
        and the distribution of $y$ is equal with the $y$-marginal of any $\D \in \mathfrak{D}_t$.
    \end{itemize}
    Moreover, the set $\mathfrak{D}_t$ is $(d^{-\Omega(\log(\eta/\gamma(t)))}/\gamma(t),4/\gamma(t))$-correlated with respect to the reference distribution ${\cal R}_t$.
\end{lemma}
\begin{proof}  
    From \Cref{pro:construction}, we know that for every $t$, with $|t|$ larger than a sufficiently large constant, there exists a distribution $\D_t$ on $ \R \times \{\pm 1\}$ 
    whose $z$-marginal is a standard normal, $\D$ satisfies the $\eta$-Massart noise distribution with respect to a 
    $\gamma$-biased halfspace $f: \R \mapsto \{\pm 1 \}$, with $f(z)=\sign(z+t)$, and for every $k =\Theta( \log(\eta/\gamma))$ it holds that 
    $\E_{(z,y) \sim \D_t}[y z^k] = \E_{y \sim {\D_t}_y}[y] \E_{z \sim \normal}[ z^k] $.
    Let $\beta_t(z)$ be the noise function corresponding to $\D_t$ and $\phi(z)$ be the density function of the
    single dimensional standard normal distribution.
    We define the following densities on $\R$:
    $A_{t}(z) = (1+\beta(z)f(z))\phi(z)/(1+c)$ and 
    $B_{t}(z)=(1-\beta(z)f(z))\phi(z)/(1-c)$ where $c=\E_{z\sim \normal}[\beta(z)f(z)]
    =\E_{y \sim {\D_t}_y}[y]
    $.
    It holds that 
    \begin{align*}
        \E_{z \sim A_t}[z^k]
        &=
        \frac{ \int z^k \phi(z) + z^k \beta(z) f(z) \d z }{1+c}
        = 
        \frac{\E_{z \sim \normal}[z^k]  + \E_{(z,y) \sim \D_t}[z^k y] }{1+c}
        \\
        &= 
        \frac{\E_{z \sim \normal}[z^k]  + \E_{z \sim {\D_t}_z}[z^k] \E_{y \sim {\D_t}_y}[y] }
        {1+ \E_{y \sim {\D_t}_y}[y]}
        = \E_{z \sim \normal}[z^k] \,.
    \end{align*}
    Similarly, we have that  $\E_{z \sim B_t}[z^k] =  \E_{z \sim \normal}[z^k]$.
    Therefore, the distributions $A_t$ and $B_t$ match the first $\Theta(\log(\eta/\gamma))$ moments with $\normal$. 
    
    Moreover, we have that
    \begin{align*} 
        \chi^2(A_t, \normal) 
        &=  \E_{z \sim \normal}\Big[ \Big(\frac{1 + \beta(z) f(z) }{1 + c} \Big)^2 -1\Big] 
        \leq  \frac{4} {(1 + c)^2}\;,
    \end{align*}
    and
    \begin{align*} 
        \chi^2(B_t, \normal) 
        &=  \E_{z \sim \normal}\Big[ \Big(\frac{1 - \beta(z) f(z) }{1 + c} \Big)^2 -1\Big] 
        \leq  \frac{4} {(1 - c)^2}\;,
    \end{align*}
    which means that $\max((1+c)	\chi^2(A_t, \normal) ,(1-c)\chi^2(B_t, \normal) )\lesssim 1/\gamma$.
    Let S be as in \Cref{fct:near-orth-vec}.  
    Choose $p=(1+c)/2$ and consider the mixture distributions $\p_{\vec v}^{A_t,B_t,p}$, for $\vec v \in S$.
    We set the hard family of distributions $\mathfrak{D}_t = \{\p_{\vec v}^{A_t,B_t,p}: \vec v \in S \}$. First, we show that $\p_{\vec v}^{A_t,B_t,p}$ corresponds to a distribution that satisfies the $\eta$-Massart noise condition. We have
    \begin{align*}
        \E_{(\x,y)\sim \p_{\vec v}^{A_t,B_t,p}}[y|\x]
        &=\frac{(1+c)}{2} (1+f(\x \cdot \vec v)\beta(\x \cdot \vec v))\frac{\phi(\x)}{1+c} - \frac{(1-c)}{2} (1-f(\x \cdot \vec v)\beta(\x \cdot \vec v))\frac{\phi(\x)}{1-c} 
        \\ &=f(\x \cdot \vec v)\beta(\x \cdot \vec v)\;,
    \end{align*}
    thus, indeed the distribution $ \p_{\vec v}^{A_t,B_t,p}$ satisfies the $\eta$-Massart noise condition,
    because from \Cref{pro:construction}, we know that $\beta = 1-2\eta \leq \beta(z) \leq 1$ almost surely
    for all $z \in \R$.
    
    Let ${\cal R}_t$ be a distribution on $\R^d\times\{\pm 1\}$ such that if $(\x,y)\sim {\cal R}_t$ then 
    $\x\sim \normal_d$, $y$ is independent of $\x$, and $y=1$ with probability $p$ and $y=-1$ otherwise.  We need to show that for $\vec u , \vec v \in S$ we have
    that $|\chi_{{\cal R}_t}(\p^{A_t,B_t,p}_{\vec v},\p^{A_t,B_t,p}_{\vec u})|$ is small.
    Since ${\cal R}_t,\p^{A_t,B_t,p}_{\vec v},$ and $\p^{A_t,B_t,p}_{\vec u}$ all assign $y=1$ with probability $p$,
    we have that 
    \begin{align*}
        \chi_{{\cal R}_t}(\p^{A_t,B_t,p}_{\vec v},\p^{A_t,B_t,p}_{\vec u})
        = & \; p \; \chi_{{\cal R}_t \mid y=1}\left( (\p^{A_t,B_t,p}_{\vec v} \mid y=1) , (\p^{A_t,B_t,p}_{\vec u} \mid y=1) \right) +    \\
        & (1-p) \; \chi_{{\cal R}_t \mid y=-1} \left( (\p^{A_t,B_t,p}_{\vec v} \mid y=-1) , (\p^{A_t,B_t,p}_{\vec u} \mid y=-1) \right) \\
        = & \; p \; \chi_{\normal_d}(\p^{A_t}_{\vec v} , \p^{A_t}_{\vec u}) + (1-p) \; \chi_{\normal_d}(\p^{B_t}_{\vec v}, \p^{B_t}_{\vec u}).
    \end{align*}
    By \Cref{fct:cor}, it follows that
    \begin{align*}
        \chi_{{\cal R}_t}(\p^{A_t,B_t,p}_{\vec v},\p^{A_t,B_t,p}_{\vec u}) 
        &\leq d^{-\Omega(\log(\eta/\gamma))}(p\,\chi^2(A_t, N(0,1))+(1-p)\chi^2(B_t, N(0,1))) 
        \\
        & = d^{-\Omega(\log(\eta/\gamma))}/\gamma \;.
    \end{align*}
    A similar computation shows that
    \[\chi_{{\cal R}_t}(\p^{A_t,B_t,p}_{\vec v},\p^{A_t,B_t,p}_{\vec v}) = \chi^2(\p^{A_t,B_t,p}_{\vec v},{\cal R}_t)\leq p\,\chi^2(A_t, N(0,1))+(1-p)\chi^2(B_t, N(0,1)\leq 4/\gamma \;.\]
    Thus, the set $\mathfrak{D}_t$ is $(d^{-\Omega(\log(\eta/\gamma))}/\gamma,4/\gamma)$-correlated with respect the reference distribution ${\cal R}_t$.
\end{proof}
\begin{proof}[Proof of \Cref{prop:sq-test-massart}]
    Fix a $\gamma>0$ less than a sufficiently small constant, and let $t\in \R$ such that $\pr_{z\sim\normal}[z\geq |t|]=\gamma$
    Moreover, because $|\mathfrak{D}_t|=2^{d^{\Omega(1)}}$ and the set $\mathfrak{D}_t$ is $(d^{-\Omega(\log(\eta/\gamma))}/\gamma,4/\gamma)$-correlated with respect ${\cal R}_t$ we have that $\mathrm{SD}({\cal B},d^{-\Omega(\log(\eta/\gamma))}/\gamma,4/\gamma)=2^{d^{\Omega(1)}}$, thus an application of \Cref{lem:sq-from-pairwise} completes the proof.
\end{proof}

\begin{proof}[Proof of \Cref{thm:lower_bound}] 
    Let $\mathfrak{D}$ and ${\cal R}$ be as in \Cref{prop:sq-test-massart}. Note that from the construction, for any $\D\in\mathfrak{D}$, it holds that $\pr_{(\x,y)\sim {\cal R}}[y=i]=\pr_{(\x,y)\sim \D}[y=i]=p_i$, for $i\in\{\pm 1\}$ Let $\cal A$ an algorithm that outputs a hypothesis $h$ with respect a distribution $\D$ that satisfies the $\eta$-Massart Noise condition such that $\pr_{(\x,y)\sim \D}[h(\x)\neq y]\leq \opt +\eps$, where $\opt$ is the error achieved by the best classifier. Moreover, from for any classifier $h'$ it holds $\pr_{(\x,y)\sim {\cal R}}[h'(\x)\neq y]\geq \min_{i\in\{\pm 1\}}p_i$. Thus, it holds that $\pr_{(\x,y)\sim {\cal R}}[h'(\x)\neq y]-\opt\geq \E_{\x\sim \D_\x}[\beta(\x)f(\x)]\geq (1-2\eta)\gamma$. Thus, algorithm $\cal A$ for $\eps \leq \gamma\beta$ would solve the  decision problem
    $\mathcal{B}(\mathfrak{D}, {\cal R})$ (with one additional query $\E_{(\x,y)\sim \D}[h(\x)y]$ up to accuracy $\gamma\beta/2$) and from \Cref{prop:sq-test-massart} the result follows. 
\end{proof}
 
\subsection{SQ Lower Bound for Learning Homogeneous Halfspaces with General Massart Noise}\label[sub]{sub:lower-bound-benign}

Learning homogeneous halfspaces under $\eta$-Massart noise for any
constant $\eta < 1/2$ is known to be solvable with polynomially many statistical queries~\cite{DKTZ20}.  
When $\eta = 1/2$, we show that any SQ algorithm for
homogeneous Massart halfspaces requires $d^{\Omega(\log(1/\eps))}$ queries.
For comparison, we recall that the SQ complexity of learning halfspaces 
under adversarial noise is $d^{\Omega(1/\eps^2)}$~\cite{DKPZ21}.  

Our formal result is the following theorem.
\usetikzlibrary{patterns}

\begin{theorem}[$1/2$-Massart Noise SQ-Lower Bound]\label{thm:lower_bound_benign}
Let $\D$ be a distribution on $\R^d\times\{\pm 1\}$ that satisfies the Massart noise condition for
$\eta=1/2$, for homogeneous halfspaces, where the $\x$-marginal is distributed according to the
standard normal.  Any SQ algorithm that, for any such distribution $\D$, finds a hypothesis $h$ such
that $\pr_{(\x,y)\sim \D}[h(\x)\neq y]\leq \opt +\eps$ either requires queries with tolerance at
most $d^{-\Omega(\log(1/\eps))}$ or makes at least $2^{d^{\Omega(1)}}$ queries.
\end{theorem}
To prove the theorem above, we are going to use a ``reduction'' to the problem of 
learning general halfspaces with constant-bounded Massart noise.

\begin{proposition}[SQ Complexity of Hypothesis Testing]\label{prop:sq-test-benign}
    Let $\eps>0$ less than a sufficiently small constant. There exist:
    \begin{itemize}
        
        \item  a family of distributions $\mathfrak{D}_B$ such that every $\D \in \mathfrak{D}$ 
        is a distribution over $(\x',y) \in \R^{d+1} \times \{\pm 1\}$, its $\x'$-marginal is the standard normal
        distribution, and $\D$ satisfies the Massart noise condition for $\eta=1/2$, with respect to some unbiased halfspace, and
        
        \item  a reference distribution $\cal R$  over $((\x,z),y) \in \R^{d+1} \times \{\pm 1\}$, whose $(\x,z)$-marginal is the standard normal and $y$  depends only on $z$,
    \end{itemize}
    such that any SQ algorithm that decides whether the input distribution belongs to $\mathfrak{D}$ or is equal 
    to the reference $\cal R$ either requires queries with tolerance at most $d^{-\Omega(\log(1/\eps))}$
    or makes at least $2^{d^{\Omega(1)}}d^{-\log(1/\eps)}$ queries.
\end{proposition}

\begin{proof}
    \Cref{prop:sq-test-benign} follows from the lemma below showing that we can construct 
    a family of distributions with small pairwise correlation 
    \begin{lemma}[Correlated Family of Distributions]
        \label{lem:hard-dist-benign}
        Let $\eps>0$ less than a sufficiently small constant. There exist:
        \begin{itemize}
            \item 
            a set $\mathfrak{D}_B$ of $2^{d^{\Omega(1)}}$ distributions on $\R^{d+1}\times\{\pm 1\}$, such that
            every $\D\in \mathfrak{D}_B$ is a distribution over $(\x',y) \in \R^{d+1} \times \{\pm 1\}$, its $\x'$-marginal is the standard normal distribution, and $\D$ satisfies the Massart noise condition for $\eta=1/2$ with respect to a halfspace $f(\x')=\sign((\vec v,1)\cdot \x')$ where $\vec v$ is a unit vector in $\R^d$.
            Moreover, all $\D \in \mathfrak{D}_B$ have the same $y$-marginal.
            \item A reference distribution ${\cal R}$ in $\R^{d+1}\times\{\pm 1\}$,  where for $((\x,z), y) \sim {\cal R}$ we have that
            $\x'$ is a standard Gaussian $ \normal(\vec 0, \vec I)$, $y$ depends only on $z$, 
            and for any function $h:\R^{d+1}\mapsto \{\pm 1\}$ and any $\D\in \mathfrak{D}_B$, it holds that $\pr_{(\x',y)\sim {\cal R}}[h(\x')\neq y] - \min_{f\in \mathcal C}\pr_{(\x',y)\sim {\D}}[f(\x')\neq y]\geq 2\eps$.
        \end{itemize}
        Moreover, the set $\mathfrak{D}_B$ is $(d^{-\Omega(\log(1/\eps))},16\eps)$-correlated with respect to the reference distribution ${\cal R}$.
    \end{lemma}
    \begin{proof}
        Fix $t_0>0$ and $\zeta>0$ such that $\pr_{z\sim \normal}[z\geq t_0]=\gamma$ for  $\gamma=4\sqrt{\eps}$ and $\pr_{z\sim \normal}[t_0+\zeta>z\geq t_0]=\gamma/2$.
        We are going to construct a new set of distributions $\mathfrak{D}_B$ as follows: 
        \begin{enumerate}
            \item  
            For $\eta = 1/4$ and some threshold $t$ denote $(\mathfrak{D}_t,{\cal R}_t)$ be the family of distributions 
            and their corresponding reference distributions from \Cref{lem:hard-dist-massart}.
            Recall that the distributions in the family $\mathfrak{D}_t$ are indexed by 
            a set of unit vectors $S$, i.e., $\mathfrak{D}_t = \{ \D_{\vec u,t}: \vec u \in S \}$.
            Every $\D_{\vec u,t}\in \mathfrak{D}_t$ is a distribution in $\R^d$ that satisfies the $\eta$-Massart 
            noise condition with $\eta=1/4$ and with respect to the halfspace $f(\x)=\sign(\vec u \cdot \x +t)$.
            \item 
            Fix some direction $\vec u \in S$.  
            We define a new distribution $\D_{\vec u}'$ on $((\x, z), y) \in \R^{d+1}\times\{\pm 1\}$
            where the ``extra'' coordinate $z$ is drawn from the standard normal distribution $\normal(0,1)$.
            When $z$ falls inside some thin interval $[t_0, t_0 + \zeta]$ we sample $(\x, y)$ from
            the $1/4$-Massart noise distribution $\D_{\vec u, z}$.  This corresponds to the blue/red region 
            in \Cref{fig:hardband}.  When $z$ falls outside $[t_0, t_0 + \zeta]$, we draw 
            $\x \sim \normal(\vec 0, \vec I)$ and set $y$ to be $\pm 1$ with probability $1/2$, independently of $\x$.
            This ``high-noise'' area corresponds to the gray area of \Cref{fig:hardband}.
            More formally, we define
            $$\D'_{\vec u}((\x,z),y) =\begin{cases} \D'_{\vec u,z}(\x,y)\phi(z) \quad \text{if } z\in[t_0,t_0+\zeta]\\
                \frac{1}{2}\phi_{d+1}((\x,z)) \quad \text{otherwise} \;.
            \end{cases} $$
            Let $\mathfrak{D}_B$ be the set of these distributions.
            \item 
            We define a reference distribution $\cal R$  on $((\x, z), y) \in \R^{d+1}\times\{\pm 1\}$
            similarly. The ``extra'' coordinate $z$ is again drawn from the standard normal distribution $\normal(0,1)$.
            When $z$ falls inside some thin interval $[t_0, t_0 + \zeta]$, i.e., the blue/red region of \Cref{fig:hardband} 
            we sample $(\x, y)$ from the reference distribution ${\cal R}_t$. 
            When $z$ falls outside $[t_0, t_0 + \zeta]$, i.e., 
            in the gray area of \Cref{fig:hardband}, we draw $\x \sim \normal(\vec 0, \vec I)$ and set $y$ to be $\pm 1$ with probability $1/2$, independently of $\x$.
            We have
            $$
            {\cal R}((\x,z),y) =
            \begin{cases} {\cal R}_z(\x,y)\phi(z) \quad \text{if } z\in[t_0,t_0+\zeta]\\
                \frac{1}{2}\phi_{d+1}((\x,z)) \quad \text{otherwise} \;.
            \end{cases} 
            $$
        \end{enumerate}
        Our first step is to prove that any ``hidden-direction'' $(d+1)$-dimensional 
        distribution $\D'_{\vec u}$ that we create out of the Massart-noise instances $\D_{\vec u, t}$ 
        satisfies the $1/2$-Massart noise condition with a homogeneous optimal halfspace.   We show the following claim.
        \begin{claim}
            \label{clm:homogeneous-benign-from-massasrt}
            For the distribution $\D_{\vec u}'$, the optimal hypothesis is $f'(\x')=\sign(\x' \cdot (\vec u,1))$, 
            thus $\D_{\vec u}'$ satisfies the Massart noise condition for $\eta=1/2$ with respect to a homogeneous $(d+1)$-dimensional halfspace.
        \end{claim}
        \begin{proof}
            In what $A$ be the event that $z\in[t_0,t_0+\zeta]$. 
            Assume in order to reach to a contradiction that $f'(\x')$ is not the optimal hypothesis. Let $h(\x')$ to be the optimal hypothesis for the distribution $\D_{\vec u'}$. 
            First, observe that for any hypothesis $h'(\x')$ it holds
            \begin{equation*}
                \pr_{(\x',y)\sim \D_{\vec u}'}[h'(\x')\neq y, A^c]=\pr_{(\x',y)\sim \D_{\vec u}'}[A^c]/2\;.
            \end{equation*}
            Thus, we need only to consider the error inside the $A$. For the function $f'(\x')$, we have that 
            \begin{align}
                \pr_{(\x',y)\sim \D_{\vec u}'}[f'(\x')\neq y, A]&=\int_{[t_0,t_0+\zeta]} \pr_{((\x,t),y)\sim \D_{\vec u}'}[f'((\x,t))\neq y|t=z] \phi(z)\d z \nonumber\\
                &=\int_{[t_0,t_0+\zeta]} \pr_{(\x,y)\sim \D_{\vec u,z}}[\sign(\vec u \cdot \x + z)\neq y] \phi(z)\d z \label{eq:optimal-err-ben1}\;.
            \end{align}
            Similarly, for any other hypothesis $h'(\x')$, we have that 
            \begin{align}
                \pr_{(\x',y)\sim \D_{\vec u}'}[h'(\x')\neq y, A]=\int_{[t_0,t_0+\zeta]} \pr_{(\x,y)\sim \D_{\vec u,z}}[h'((\x,z))\neq y] \phi(z)\d z \label{eq:optimal-err-ben2}\;.
            \end{align}
            Moreover, for the  $h'(\x')$ it should hold that $ \pr_{(\x',y)\sim \D_{\vec u}'}[f'(\x')\neq y, A]- \pr_{(\x',y)\sim \D_{\vec u}'}[h(\x')\neq y, A]>0$, thus combining with \cref{eq:optimal-err-ben1} and \cref{eq:optimal-err-ben2}, we have that there exists $z\in [t_0,t_0+\zeta]$ such that $\pr_{(\x,y)\sim \D_{\vec u,z}}[\sign(\vec u \cdot \x + z)\neq y]>\pr_{(\x,y)\sim \D_{\vec u,z}}[h'((\x,z))\neq y]$, which is a contradiction, because for each $\D_{\vec u,t}$ the optimal classifier is $f(\x)=\sign(\vec u \cdot \x + t)$.
        \end{proof}
        
        We next have to show that the family of ``hidden-direction'' distributions $\mathfrak{D}_B$ is pairwise correlated.
        We have the following claim.
        \begin{claim}\label{clm:correl-benign}
            For every $\D \in \mathfrak{D}_B$ it holds $\chi_{\cal R}(\D,\D)\leq 4$
            and for every distinct $\D_1, \D_2 \in \mathfrak{D}_B$ it holds 
            $\chi_{\cal R}(\D_1,\D_2)\leq d^{-\Omega(\log(1/\gamma))}$.
            In other words, the family $\mathfrak{D}_B$ is $(d^{-\Omega(\log(1/\gamma))},4)$-correlated with respect to the reference distribution $\cal R$.
        \end{claim} 
        \begin{proof}
            Note that by the construction of the distributions $\D_1,\D_2$, for a point $\x'=(\x,z)\in \R^{d+1}$, 
            such that $z\not\in[t_0,t_0+\zeta]$ the conditional distributions $\D_1\mid z$, $\D_2\mid z$ and $\mathcal R \mid z$ coincide.
            Therefore, since $\D_1,\D_2,\mathcal R$ give the same distribution on $z$, for $z\not\in[t_0,t_0+\zeta]$ it holds $\chi_{\mathcal R \mid z}((\D_1\mid z),(\D_2\mid z))=0$, for $z\not\in[t_0,t_0+\zeta]$.  Thus, to compute the correlation
            of $\D_1, \D_2$ it suffices to consider $z \in [t_0, t_0 + \zeta]$, i.e.,
            \begin{align*}
                \chi_{\cal R}(\D_1,\D_2)&= \E_{z\sim \normal}[\chi_{\mathcal R_z}((\D_1\mid z),(\D_2\mid z))]
                \\&= \int_{t_0}^{t_0+\zeta}\chi_{\mathcal R_z}((\D_1\mid  z),(\D_2 \mid z))\phi(z)\d z\;,
            \end{align*}
            where $(\D_1\mid z)$ (resp. $(\D_2\mid z)$) is the pdf $\D_1$ (resp. $\D_2$) conditioned on $z$. From \Cref{lem:hard-dist-massart}, it holds that 
            $$ 	\chi_{\cal R}(\D_1,\D_2)\leq  \frac{d^{-\Omega(\log(1/\gamma))}}{\gamma}\int_{t_0}^{t_0+\zeta}\phi(z)\d z=d^{-\Omega(\log(1/\gamma))} \;,$$
            where we used that $\pr_{z\sim \normal}[t_0+\zeta>z\geq t_0]=\gamma/2$.
            The second part follows similarly.
        \end{proof}
        Moreover, if $r$ is the best hypothesis for $\cal R$ and $f$ for $\D_{\vec u}'$, then from \Cref{lem:hard-dist-massart}, we have
        \[
        \pr_{(\x',y)\sim {\cal R}}[r(\x')\neq y]-\pr_{(\x',y)\sim \D_{\vec u}'}[f(\x')\neq y]\geq \frac{1}{3} \gamma \int_{t_0}^{t_0 +\zeta}  \phi(z)\d z\geq \frac{\gamma^2}{6}\;,
        \]
        where in the last inequality we used that $t_0, \zeta$ are chosen such that $ \pr_{z\sim \normal}[t_0+\zeta>z\geq t_0]=\gamma/2$, by substituting  $\gamma=4\sqrt{\eps}$ the result follows.
    \end{proof}
    
    To prove \Cref{prop:sq-test-benign}, from \Cref{lem:hard-dist-benign}, we have that the set $\mathfrak{D}_B$ is $(d^{-\Omega(\log(1/\eps))},16)$-correlated with respect ${\cal R}$.
    Thus, we have that $\mathrm{SD}({\cal B},d^{-\Omega(\log(1/\eps))},16)=2^{d^{\Omega(1)}}$ and an application of  \Cref{lem:sq-from-pairwise} completes the proof.
\end{proof}
\begin{proof}[Proof of \Cref{thm:lower_bound_benign}]
    Fix $\eps>0$ less than a sufficiently small constant. Let $\mathfrak{D}_B$ and $\cal R$ as in \Cref{lem:hard-dist-benign}. Let $\cal A$ be an algorithm that given $\eps'>0$ and $\D$ that satisfies the $1/2$-Massart Noise Condition with respect the halfspace $f'(\x')$ computes a hypothesis $h$ such that
    $$\pr_{(\x',y)\sim \D}[h(\x')\neq y]\leq \pr_{(\x',y)\sim \D}[f(\x')\neq y] +\eps'\;.$$ 
    We show that $\cal A$ can solve the Decision Problem ${\cal B}(\mathfrak{D}_B,{\cal R})$. Let $\D\in \mathfrak{D}_B$ with optimal halfspace $f(\x)$, then from \Cref{lem:hard-dist-benign}, we have that for any hypothesis $h$ it holds  $$\pr_{(\x',y)\sim {\cal R}}[h(\x')\neq y] - \pr_{(\x',y)\sim {\D}}[f(\x')\neq y]\geq 2\eps\;.$$

    The algorithm $\cal A$ for $\eps' \leq 2\eps$ would solve the  decision problem
    $\mathcal{B}(\mathfrak{D}_B, {\cal R})$ (with one additional query $\E_{(\x',y)\sim \D}[h(\x')y]$ up to accuracy $\eps$), thus from \Cref{prop:sq-test-benign}, we get our result.
\end{proof}

\bibliographystyle{alpha}
\bibliography{clean2}

\clearpage
\appendix
\section{``Massart'' Noise with $\eta>1/2$ is Equivalent to Agnostic Learning}
\label{app:semi-random-agnostic}
\begin{lemma}[Learning with $\eta > 1/2$]
	Let $\mathcal{C}$ be a class of Boolean functions on $\R^d$ and let $\mathcal{F}$ 
	be a class of distributions over $\R^d$.
	Fix $\eta \in (1/2, 1]$
	and  let $\cal A$ be a learning algorithm that given $m$ 
	samples from a distribution $\D'$ with $\eta$-semi-random noise, learns a hypothesis
	$h: \R^d \mapsto \{\pm 1\}$ such that
	\[
	\pr_{(\bx, y) \sim \D'}[h(\bx) \neq y ] \leq 
	\min_{c \in \mathcal{C}} \pr_{(\bx, y) \sim \D'}[c(\bx) \neq y ]  + \eps
	\,.
	\]
	Then $\mathcal{A}$ can learn $\mathcal{C}$ in the agnostic PAC learning model.
\end{lemma}
\begin{proof}
	Let $\D$ be any distribution on $\R^d \times \{\pm 1\}$.  
	Since $\eta > 1/2$, we can create a 
	distribution $\D'$ with $\eta$-semi-random noise as follows: 
	\begin{enumerate}
		\item  Draw $(\x,y)\sim \D$.
		\item  With probability $2\eta-1$ return $(\x, y)$
		and with probability $2 (1 -\eta)$ return $(\x, \hat{y})$ where
		$\hat{y} \in \{\pm 1\}$ is uniformly random and independent of $\x$.
	\end{enumerate}
	For any $h \in \mathcal{C}$ (in fact for any classifier in general), it holds that 
	$ \pr_{(\x,y) \sim \D'}[h(\x) = y|\x]  \geq  2(1-\eta)/2 \geq 1-\eta$
	and thus we have that $\eta(\x) \leq \eta$.  The classifier $f$ of \Cref{def:massart-learning}
	can therefore be any classifier in $\mathcal{C}$.
	For every classifier $h$ it holds 
	\begin{equation}\label{eq:agnostic-reduction}
		\pr_{(\x,y) \sim \D'}[h(\x) \neq y]  
		= 
		(2 \eta - 1) \pr_{(\x,y) \sim \D}[h(\x) \neq y] 
		+ 2(\eta - 1) (1/2) \,.
	\end{equation}
	We can therefore, use
	$\mathcal{A}$ on the samples from $\D'$ and obtain a classifier $h$
	such that 
	\[
	\pr_{(\x,y) \sim \D'}[h(\x) = y]  \leq   
	\min_{c \in \mathcal{C}} \pr_{(\x,y) \sim \D'}[c(\x) = y]  + \eps \,.
	\]
	Using \Cref{eq:agnostic-reduction} we obtain that for the same classifier $h$
	it holds 
	\[
	\pr_{(\x,y) \sim \D}[h(\x) = y]  \leq   
	\min_{c \in \mathcal{C}} \pr_{(\x,y) \sim \D}[c(\x) = y]  + \eps \,,
	\]
	and therefore $\mathcal{A}$ can learn the class $\mathcal C$ with respect
	to any distribution $\D$, i.e., in the agnostic PAC learning setting.
\end{proof}
\section{Benign Noise is Equivalent to $1/2$-Massart Noise}\label{app:benign-massart}
\begin{fact}
	Let $\D$ be a distribution on $\R^d\times\{\pm 1\}$. $\D$  satisfies the $1/2$-Massart noise condition with respect
	to a halfspace $f(\x)$ if only if the distribution $\D$ satisfies the Benign noise condition with respect to a halfspace $f(\x)$.
\end{fact}
\begin{proof}
We prove each direction separately.
\begin{claim}
	If $\D$ satisfies the Massart noise condition with $\eta=1/2$ with respect to a halfspace $f(\x)$, then $\D$ satisfies the Benign noise condition with respect to the halfspace $f(\x)$.
\end{claim}
\begin{proof}
	Assume in order to reach in contradiction that the optimal classifier is not $f$ and is $h$ the optimal one. Therefore, it holds  $\pr_{(\x,y)\sim \D}[h(\x)\neq y] < \pr_{(\x,y)\sim \D}[h(\x)\neq y]$.
	It holds that 
	\begin{align*}\pr_{(\x,y)\sim \D}[h(\x)\neq y] &= \E_{\x\sim \D_\x}[\1\{h(\x)\neq f(\x)\}(1-\eta(\x))]+\E_{\x\sim \D_\x}[\1\{h(\x)= f(\x)\}\eta(\x)]
	\\ &=\E_{\x\sim \D_\x}[\1\{h(\x)\neq f(\x)\}(1-2\eta(\x))]+\E_{\x\sim \D_\x}[\eta(\x)]\;.
	\end{align*}
	Note that $\pr_{(\x,y)\sim \D}[h(\x)\neq y]=\E_{\x\sim \D_\x}[\eta(\x)]$, therefore, we have
	\[
	\E_{\x\sim \D_\x}[\1\{h(\x)\neq f(\x)\}(1-2\eta(\x))]<0\;,
	\] 
	which means that there is a point (a mass with non-zero measure) with $\eta(\x)>1/2$, which leads to a contradiction. Therefore, if $\eta(\x)\leq 1/2$ then the optimal classifier is $f(\x)$. 
		
\end{proof}
Next, we prove the other direction.
\begin{claim}
	If $\D$ satisfies the Benign noise condition with respect to a halfspace $f(\x)$, then $\D$ satisfies the $1/2$-Massart noise condition with respect to the halfspace $f(\x)$.	
\end{claim}
\begin{proof}
	It suffices to prove that $\E_{\x\sim \D_\x}[\1\{\eta(\x)>1/2\}]=0$. Because $f(\x)$ is the optimal classifier, that means that other classifier $h(\x)$ gets more error, therefore it holds
	\[
	\pr_{(\x,y)\sim \D}[h(\x)\neq y]=	\E_{\x\sim \D_\x}[\1\{h(\x)\neq f(\x)\}(1-2\eta(\x))] + 	\E_{\x\sim \D_\x}[\eta(\x)]\;.
	\]
	Because $f(\x)$ is the optimal classifier, it holds that $\pr_{(\x,y)\sim \D}[h(\x)\neq y]\geq \pr_{(\x,y)\sim \D}[f(\x)\neq y]$, therefore 
	\[
	\E_{\x\sim \D_\x}[\1\{h(\x)\neq f(\x)\}(1-2\eta(\x))]\geq 0 \;.
	\]
	Assume that $\E_{\x\sim \D_\x}[\1\{\eta(\x)>1/2\}]=a>0$, therefore we can set $h(\x)=f(\x)$ when $\eta(\x)\leq 1/2$ and $h(\x)\neq f(\x)$ otherwise. Hence, we have $\E_{\x\sim \D_\x}[\1\{h(\x)\neq f(\x)\}(1-2\eta(\x))]= \E_{\x\sim \D_\x}[(1-2\eta(\x))\1\{\eta(\x)>1/2\}]<0$, which is a contradiction. Therefore, $\E_{\x\sim \D_\x}[\1\{\eta(\x)>1/2\}]=0$, that means that the measure of points that $\eta(\x)>1/2$ is 0, hence, it satisfies the $1/2$-Massart noise condition.
\end{proof}
\end{proof}
\section{A Non-Continuous Certificate}
\label{app:non-continuous-certificate}
\paragraph{Non-Continuous Certificates.}
When we do not restrict our search to continuous functions 
in order to find a certificate we can search over functions 
of the form 
$T(x) = \1\{\sgn(\ell(\x)) \neq \sgn(\vec v \cdot \x - b)\} $
for some $\vec v \in \R^d, b \in \R$.
\begin{fact}
	Let $\D$ be a distribution on $\R^{d} \times \{\pm 1\}$, with standard normal $\x$-marginal, that satisfies the $\eta$-Massart
	noise condition with respect to the optimal halfspace $f(\x)$. Then, for any linear function $\ell(\x)$ such that $\pr_{(\x,y)\sim \D}[\sgn(\ell(\x))\neq y]\geq \opt+\eps$, let $T(\x) = \1\{ \sgn(\ell(\x)) \neq f(\x)\}$, then we have that
	\begin{align*}
		\E_{(\x, y) \sim \D}[\ell(\x) y~T(\x)] \leq - \Omega(\eps^2 ) \|\ell(\x)\|_2 \,.
	\end{align*}
\end{fact}
\begin{proof}
	Using the $\eta$-Massart noise condition, we have
	\begin{align*}
		\E_{(\x, y) \sim \D}[\ell(\x) y~T(\x)] &=
		\E_{\x \sim \D_\x}[ \ell(\x) f(\x) \beta(\x) \1\{ \sgn(\ell(\x)) \neq f(\x)\} ] 
		\\&=-\E_{\x \sim \D_\x}\left[|\ell(\x)|\beta(\x) \1\{ \sgn(\ell(\x)) \neq f(\x)\} \right] \;.
	\end{align*}
	From \Cref{lem:carbery-wright}, we have that $\pr_{\x\sim \D_\x}[|\ell(\x)|\geq \|\ell(\x)\|_2\eps/C]\geq 1-\eps/2$, for some absolute constant $C>0$. From the assumptions, we have that $\E_{\x\sim \D_\x}[|\beta(\x) \1\{ \sgn(\ell(\x)) \neq f(\x)\} ]\geq \eps$, therefore it holds $$\E_{\x \sim \D_\x}\left[\{|\ell(\x)|\geq \|\ell(\x)\|_2\eps/c\}\beta(\x) \1\{ \sgn(\ell(\x)) \neq f(\x)\} \right]\geq \eps/2\;,$$ and we have
	\begin{align*}
		\E_{\x\sim \D_\x}[|\ell(\x)|\beta(\x) &\1\{ \sgn(\ell(\x)) \neq f(\x)\} ]\\&\geq 
		(\|\ell(\x)\|_2\eps/C)\E_{\x \sim \D_\x}[\{|\ell(\x)|\geq \|\ell(\x)\|_2\eps/c\}\beta(\x) \1\{ \sgn(\ell(\x)) \neq f(\x)\} ]
		\\& \gtrsim \|\ell(\x)\|_2\eps^2/C\;.
	\end{align*}
	And this completes the proof.
\end{proof} \section{Learning General Halfspaces with General Massart Noise}\label{app:general_benign}
In this section, we provide the algorithm of learning biased halfspaces with general Massart Noise.
\begin{theorem}[Learning General Halfspaces with General Massart Noise]
    \label{thm:benign-general}
    Let $\D$ be a distribution on $\R^{d} \times \{\pm 1\}$, with standard normal
    $\x$-marginal, that satisfies the Massart noise condition for $\eta=1/2$ with respect to
    some optimal (possibly biased) halfspace $f \in {\cal C}$.
    Let $\eps, \delta \in (0,1]$.
    There exists an algorithm that draws $N = d^{O(\log(1/\eps))} \log(1/\delta)$
    samples from $\D$, runs in time $\poly(N, d) 2^{\poly(1/\eps)}$,
    and computes a halfspace $h \in {\cal C}$ such that with probability at least $1-\delta$,
    \[
    \pr_{(\x, y) \sim \D}[h(\x) \neq y] \leq \opt  + \eps
    \,.
    \]
\end{theorem}

The proposition below is similar to \Cref{prop:warm-start}.
\begin{proposition} \label{prop:warm-start-general}
	Let $\D$ be a distribution on $\R^{d} \times \{\pm 1\}$, with standard
	normal $\x$-marginal, that satisfies the Massart noise condition for $\eta =1/2$,
	 with respect to some optimal halfspace $f(\x) = \sgn(\vec w^\ast \cdot
	 \x+t^\ast)$.  Let $\vec w \in \R^d$ be a unit vector and $t\in \R$ such that $\pr_{(\x, y) \sim
	 \D}[\sign(\vec w \cdot \x+t) \neq y] \geq \opt + \eps$ and $\theta(\vec w),\wstar)\leq \pi-\eps$, for some $\eps \in
	 (0,1]$.  There exists an algorithm that draws
	 $N=d^{O(\log(1/\eps))}\log(1/\delta)$ samples from $\D$, runs in time
	 $\poly(N,d)$ and, with probability at least $1-\delta$ returns a basis of a subspace $V\subseteq \vec w^\perp$ such that 
	 $\|\proj_V(\wperp)\|_2=\poly(\eps)$.
\end{proposition}
The proof of \Cref{prop:warm-start-general} is the same as \Cref{prop:warm-start} with the only difference be that, instead of
\Cref{lem:band-projection} we use the lemma below.

\begin{lemma}
	\label{lem:band-projection_app}
	Let $\D$ be a distribution on $\R^{d} \times \{\pm 1\}$, with standard normal $\x$-marginal,
	that satisfies the Massart noise condition with $\eta=1/2$ with
	respect to $f(\x)=\sign(\vec w^\ast \cdot \x+t)$ with $t\in \R$. Fix $\eps>0$ and $\rho>0$ such that $\rho\lesssim \eps$. Let $\vec w$ be a unit vector
	such that for any $t'\in \R$, it holds $\E_{\x\sim \D_\x}[\beta(\x)\1\{f(\x)\neq \sign(\vec w \cdot
	\x+t')\}]\geq \eps$.
	For $t_1, t_2 \in \R$ consider the band $B=\{t_1\leq\vec w \cdot \x \leq t_2\}$.
	Denote $\D^\perp=\D_B^{\perp_{\bw}}$, i.e., $\D$
	is the orthogonal projection onto $\vec w^\perp$ of the conditional distribution on $B$,
	and consider the halfspace $f^\perp: \vec w^\perp \mapsto \{\pm 1\}$,
	with $f^\perp(\x) = \sgn(\dotp{\x}{(\vec w^\ast)^{\perp_{\vec w} }} -b)$, for
	some threshold $b \in \R$.
	Moreover, define the noise function
	$$
	\eta^\perp(\x) = \pr_{(\vec z, y) \sim \D^\perp}[ y \neq f^\perp(\vec z) | \vec z = \x] \,.$$
	There exist $t_1, t_2 \in \R$ multiples of $\rho$, with $|t_1-t_2|=\rho$, such that:
	\begin{itemize}
\item  $\E_{\x \sim \D_\x^\perp}[(1-2\eta^\perp(\x))\1\{f^{\perp}(\x)\sgn(b) >0\}]\gtrsim \eps/\sqrt{\log(1/\eps)}-\rho/\eps$,
		\item if $\eta^\perp(\x)>1/2$, then  $0<\sign(b)(\dotp{\x}{(\vec w^\ast)^{\perp_{\vec w} }}-b)\leq \rho/\eps$.
	\end{itemize}
\end{lemma}
\begin{proof}
	First, we consider the region $B=\{t_1\leq \vec w \cdot \x \leq t_2\}$, for any $t_1,t_2\in \R$ with $|t_1-t_2|=\rho$. We denote by $\x^\perp$ the projection of $\x$ onto the subspace $\vec w^\perp$.
	Define the distribution $\D^\perp = \D_{B}^{\proj_{\vec w^\perp}}$, that is the distribution $\D$ conditioned on the set $B$ and projected onto $\vec w^\perp$,
	the hypothesis $f^\perp(\x^\perp) = \sgn( \dotp{\x^\perp}{\wperp}-b)$ where $b\in \R$ is chosen appropriately below, and the noise function
	$\eta^\perp(\x^\perp) = \pr_{(\vec z, y) \sim \D^\perp}[ y \neq f^\perp(\vec z) | \vec z = \x^\perp] $.
	
	Note that the distribution $\D^\perp$ does not satisfy the $1/2$-Massart noise condition anymore. We first illustrate how the noise function changes. The orthogonal projection on $\vec w^\perp$ creates a region where the Massart condition is violated, i.e., a region where $\eta^\perp(\x^\perp) \geq 1/2$, but we can control the probability that we get points inside this region. More formally, we show that$
		\pr_{(\x^\perp,y) \sim \D^\perp}[\eta^\perp(\x^\perp) \geq 1/2] \lesssim \rho/\eps $.
	
	To show that first notice that $\theta(\vec v, \wstar) \geq \eps$, otherwise we would have $\pr_{\x\sim \D_\x}[f(\x)\neq \sign(\vec w\cdot \x)]\leq \eps$.
	We can assume that 
	$\vec w^\ast= \lambda_1 \vec w+ \lambda_2 (\vec w^\ast)^{\perp_{\vec w} }$,
	where $\lambda_1 = \cos\theta$ and $\lambda_2 = \sin \theta$.
Next we set $\vec x=(\x_{\vec w},\vec x^\perp)$, where $\x_{\vec w}=\dotp{\vec w}{\vec x}$. We show that the hypothesis $f^\perp(\x)=\sign((\vec w^\ast)^{\perp_{\vec w} }\cdot \vec\x^\perp + (t+t_1\lambda_1)/\lambda_2)$ is almost as good as the $f(\x)$ for the distribution $\D^\perp$.
	
	Conditioned on $\x \in B$, i.e., $\x_{\vec w} \in [t_1, t_1 + \rho]$, it holds that
	$$
	\dotp{\vec w^\ast}{\vec x}
	=  \lambda_1 \x_{\vec w} + \lambda_2 \dotp{(\vec w^{\ast})^{\perp_{\vec w}}}{\vec x^\perp}
	= \lambda_1 t_1+ \lambda_2 \dotp{(\vec w^{\ast})^{\perp_{\vec w}}}{\vec x^\perp}+ s \rho
	\,,
	$$
	for some $s \in [-1,1]$ (recall that $|\lambda_1|\leq 1$). Let $b=-(t+\lambda_1 t_1)/\lambda_2$.  Notice that when
	$0\leq \sign(b)(\dotp{(\vec w^{\ast})^{\perp_{\vec w}}}{\vec x^\perp}-b) < \rho/\lambda_2$,
	$f^\perp(\x^\perp)$ is not equal to the sign of $(\dotp{\wstar}{\x}+t)$ (recall that $\lambda_2 > 0$),
	and therefore we are inside the region that Massart noise is violated.
	Thus, we need to bound the probability of this event $I^\xi$.
	We have that
	\begin{align*}
	\pr_{(\x^\perp,y) \sim \D^\perp}[\eta^\perp(\x^\perp) > 1/2]
		=
		\frac{ \pr_{\x \sim \D_\x}\left[ \dotp{(\vec w^{\ast})^{\perp_{\vec w}}}{\vec x} \in I^{\xi} \right]}
		{\pr_{\x \sim \D_\x}[\x \in B]}
		\lesssim \rho/\lambda_2\lesssim \rho/\eps
		\,,
	\end{align*}
	where we used the anti-concentration property of the Gaussian distribution and the last inequality holds because we have that $\lambda_2\gtrsim \eps$.
	
	It remains to prove that there is a choice $t_1,t_2$ and a band $B=\{t_1\leq \vec w \cdot \x \leq t_2\}$ with respect the $t_1,t_2$ such that $\E_{\x\sim (\D_B^\perp)_\x}[(1-2\eta^\perp(\x))\1\{f(\x)^\perp \neq \sgn(-b)\}]\gtrsim \frac{\eps}{\sqrt{\log(1/\eps)}}$, where $\D_B^\perp$ is the distribution $\D$ conditioned on the set $B$ and projected onto $\vec w^\perp$. We first show the following claim
Let $t_i=i\rho$ and $t_{-i}=-i\rho$, for $0\leq i\leq C\log(1/\eps)/\rho$ where $C>0$ is a large enough constant. We define $B_i=\{t_i\leq\vec v \cdot \x \leq t_{i+1}\}$ and $B_{-i}=\{-t_{i+1}\leq\vec v \cdot \x \leq -t_{i}\}$.
	For each $B_i$, we define the distributions $\D_{B_i}^\perp$, the hypothesis $f_i^\perp(\x^\perp)= \sgn( \dotp{\x^\perp}{\wperp}-b_i)$ and the noise functions $\eta_i^\perp(\x^\perp)$. 
	We remind that from the assumptions, we have that for any $t'\in \R$, it holds $\E_{\x\sim \D_\x}[\beta(\x)\1\{f(\x)\neq \sign(\vec w \cdot \x+t')\}]\geq \eps$.
	Choose $t'=t/\sin\theta$, where $\theta=\theta(\vec w,\vec w^\ast)$ and an application of \Cref{clm:random_band} to the set $\{\x\in \R^d:f(\x)\neq \sign(\vec w \cdot \x+t')\}$, gives us that there exists an index $i'$ such that
	\[  \E_{\x\sim\D_\x}[\1\{f(\x)\neq \sign(\vec w \cdot \x+t')\}\1\{\x\in B_{i'}\}\beta(\x)]\gtrsim\frac{\eps\rho}{\sqrt{\log(1/\eps)}}\;.\]
	Moreover, note that for the distribution $\D_{B_{i'}}$, that is the distribution $\D$ conditioned on $B$, it holds
	$\E_{\x\sim (\D_{B_{i'})_\x}}[\beta(\x)\1\{f(\x)\neq \sign(\vec w \cdot \x+t')\}]\gtrsim \frac{\eps}{\sqrt{\log(1/\eps)}}$,
	where we used the Gaussian concentration.
	We have $f_{i'}^\perp(\x)$ agrees almost everywhere with $f(\x)$ with respect the distribution $\D_{B_{i'}}$, i.e., we have that $\E_{(\x_{\vec w},\x^\perp)\sim (\D_{B_{i'}})_i}[\1\{f((\x_{\vec w},\x^\perp))\neq f^\perp(\x^\perp)\}]\lesssim \rho/\eps$.
	Thus using the triangle inequality, we have
	\[ \E_{(\x_{\vec w},\x^\perp)\sim (\D_{B_{i'}})_\x}\left[(1-2\eta^\perp(\x^\perp))\1\{\sign(\x_{\vec w} +t')\neq f^\perp(\x^\perp)\}\right] \gtrsim \frac{\eps}{\sqrt{\log(1/\eps)}}-\rho/\eps
	\;.\]
	The proof concludes by noting that in each $B_{i'}$, it holds $\sign(\x_{\vec w} +t')=\sign(-b_i)$ from construction of $f^{\perp}(\x)$.
	This completes the proof of  \Cref{lem:band-projection_app}.

\end{proof}

We are now ready to prove the \Cref{thm:benign-general}. Note that this is the same proof as \Cref{thm:benign} the only difference is that we need to test the different thresholds.
\begin{proof}[Proof of \Cref{thm:benign-general}]
    First, let $\widehat\D_N$ be the empirical distribution of $\D$ using $N\in \Z_+$ samples. Let
    $\vec w^{(i)}$ be the current guess. If $\min_{t\in \R}\pr_{(\x,y)\sim \D}[\sign(\vec
    w^{(i)}\cdot \x+t)\neq y]\geq \opt +\eps$ and $\theta(\vec w\ith,\vec w^ast)\leq \pi-\eps$, then from \Cref{prop:warm-start-general} with $N_1=d^{O(\log(1/\eps))}\log(1/\delta_1)$
    samples from $\widehat \D$ and $\poly(N_1,d)$ time, we compute a subspace $V$ such that $\|\proj_V(\vec w^\ast)\|_2\geq \poly(\eps)$ and from \Cref{fct:random_initialization}, we get a random a unit vector $\vec v\in V$
    such that  $\vec v \cdot \wstar =\poly(\eps)$ and $\vec v \cdot \vec w^{(i)}=0$,
     with probability $(1-\delta_1)/3$. We call this
    event $\mathcal{E}_i$. By conditioning on the event $\mathcal{E}_i$; from \Cref{lem:corr-improv},
    after we update our current hypothesis vector $\vec w^{(i)}$ with $\vec v$, we get a unit vector $\vec
    w^{(i+1)}$ such that $\vec w^{(i+1)} \cdot \wstar \geq \vec w^{(i+1)} \cdot
    \wstar + \poly(\eps)$.
    
    After running the update step $k$ times and conditioning on the events
    $\mathcal{E}_1,\ldots, \mathcal{E}_k$, then $\vec w^{(k)} \cdot \wstar \geq
    \vec w^{(1)} \cdot \wstar + k\,\poly(\eps)$; therefore, for $k=\poly(1/\eps)$, we get
    that the vector $\vec w^{(k)}$ that is competitive with the optimal hypothesis, i.e.,
    $\min_{t\in \R}\pr_{(\x,y)\sim \D}[\sign(\vec w^{(k)}\cdot \x+t)\neq y]\leq
    \opt +\eps$. The probability that all the events $\mathcal{E}_1,\ldots,
    \mathcal{E}_k$ hold simultaneously is at least $(1-k\delta_1)+(1/3)^k$, and thus by choosing $\delta_1\leq 1/(3k)$, the
    probability of success is at least $\delta_2=(1/3)^k$. By running the algorithm above $M=\log(1/
    \delta) /\delta_2$ times and along with an application of Hoeffding's inequality, we get a
    list $L$ of $M$ vectors such that contains all the unit vectors $\vec w$ the algorithm calculated along with $-\vec w$, therefore $L$ contains a vector such $\vec w$ that such that
    $\min_{t\in \R}\pr_{(\x,y)\sim \D}[\sign(\vec w\cdot \x+t)\neq y]\leq \opt +\eps$, with
    probability $1-\delta/2$. Moreover, from \Cref{fct:gaussian-halfspaces} we get that if we set $\mathcal T=\{\pm \eps^2,\pm 2\eps^2,\ldots, \pm 4\sqrt{\log(1/\eps)}\}$ we have that 
	 $\min_{t\in \mathcal T}\pr_{(\x,y)\sim \D}[\sign(\vec w\cdot \x+t)\neq y] \leq \min_{t\in \R}\pr_{(\x,y)\sim \D}[\sign(\vec w\cdot \x+t)\neq y] +\eps^2$.
	  Therefore, we construct the list $\mathcal H=\{(\vec w,t): \vec w\in L,t\in \mathcal T \}$. Finally, to evaluate all the vectors from the list,
    we need a few samples, from the distribution $\D$ to obtain the best
    among them, i.e., the one that minimizes the zero-one loss.
    
    The size of the list of candidates is at most $M\leq 2^{\poly(1/\eps)}\log(1/\delta)$.  Therefore,
    from Hoeffding's inequality, it follows that $ O(\poly(1/\eps)\log(1/\delta))$
    samples are sufficient to guarantee that the excess error of the chosen
    hypothesis is at most $\eps$ with probability at least $1-\delta/2$. Thus, with
    $N=d^{\log(1/\eps)}\log(1/\delta)$ samples and $\poly(N,d,
    2^{\poly(1/\eps)})$ runtime, we get a hypothesis $(\hat{\vec w},\hat{t})$ such that
    $\pr_{(\x,y)\sim \D}[\sign(\hat{\vec w}\cdot \x+\hat{t})\neq y]\leq \opt +\eps$ with
    probability $1-\delta$. This
    completes the proof.
\end{proof}

\section{Omitted Proofs from \Cref{sec:benign}}\label{app:sec3}
We restate and prove \Cref{fct:estimate-chow}.
    \begin{lemma}
       Fix $m\in \Z_+$ and $\eps,\delta\in(0,1)$. Let $\D$ be a distribution in $\R^d\times\{\pm 1\}$ with standard normal $\x$-marginals. There is an algorithm that with $N=d^{O(m)} \log(1/\delta)/\eps^2$ samples and $\poly(d,N)$ runtime, outputs an approximation $\vec T'^m$ of the order-$m$ Chow-parameter tensor $\vec T^m$ of $\D$ such that with probability $1-\delta$, it holds
       \[
       \|\vec T'^m-\vec T^m\|_F\leq \eps\;.
       \]
    \end{lemma}
	\begin{proof}

\begin{fact} Fix $m\in \Z_+$ and $\eps,\delta\in(0,1)$. Let $\D$ be a distribution in $\R^d\times\{\pm 1\}$ and let $N=d^{O(m)}\log(1/\delta)/\eps^2$. Let $\alpha$ be a multi-index satisfying $|\alpha|\leq m$. Then there is an algorithm that with $N$ samples and runtime $poly(d,N)$, computes with probability $1-\delta$ estimates $\hat{h}_\alpha$ such that
	\[|\E_{(\x,y)\sim \D}[yh_\alpha(\x)]-\hat h_{\alpha}|\leq \eps \;,\]
	for all $|\alpha|\leq m$.
\end{fact}
\begin{proof}
	Let $\widehat \D_N$ be the empirical distribution of $\D$ with $N$ samples. Using Markov's inequality, we have
	\begin{align*}
		\pr[|\E_{(\x,y)\sim \widehat\D}[h_\alpha(\x)y]- \E_{(\x,y)\sim \D}[h_\alpha(\x)y]|\geq \eps]
		& \leq \frac{1}{N\eps^2}\var[h_\alpha(\x)y] \\         
		& \leq \frac{1}{N\eps^2}\E_{(\x,y)\sim \D}[h_\alpha^2(\x)]\\
		& \leq O\left(\frac{1}{N\eps^2}\right)\;,
	\end{align*}
	where we used the fact that for the Hermite polynomials it holds $\E_{\x\sim \normal_d}[h_\alpha^2(\x)]=1$. Moreover, all possible choices of $\alpha$ are at most $d^m$. Hence, using the fact that $N=O(d^m/\eps^2)$ and applying a union bound for all different $h_\alpha$, we get that with constant probability it holds
	$$|\E_{(\x,y)\sim \D}[yh_\alpha(\x)]-\E_{(\x,y)\sim \widehat\D}[yh_\alpha(\x)]|\leq \eps \;,$$
	for all $|\alpha|\leq m$.
	By applying a standard probability amplification technique (see, e.g., Ex. 1, Chapter 13 of \cite{SB14}), 
	we can boost the confidence to $1-\delta$ with $N'=O(N\log(1/\delta))$ samples.
\end{proof}

    In order to estimate up to order-$m$ Chow tensors, with accuracy $\eps$, we need to learn each order-$m$ Chow 
    parameters up to accuracy $\eps^2 d^{-m}$, therefore we need to $d^{O(m)}\poly(1/\eps)\log(1/\delta)$ samples.
	\end{proof}
We restate and prove \Cref{lem:corr-improv}.
	 \begin{lemma}
    For unit vectors $\vec v^{\ast}, \vec v \in \R^d$,
    let $\vec {u }\in \R^d$ such that
    $\dotp{\vec {u }}{\vec v^{\ast}} \geq c$, $\dotp{\vec {u }}{\vec v} = 0$,
    and  $\snorm{2}{\vec {u }}\leq 1$, with $c>0$.
    Then, for $\vec v'=\frac{{\vec v}+\lambda{\vec{ u}}}{\snorm{2}{{\vec v}+\lambda{\vec{u}} }}$,
    with $\lambda=c/2$, we have that
    $\dotp{\vec v'} {\vec v^{\ast}}\geq  \dotp{\vec v} {\vec v^{\ast}}+\lambda^2/2 $.
\end{lemma}
\begin{proof}
    We will show that $\dotp{\vec v'}{\vec v^\ast}=\cos\theta'\geq\cos \theta+\lambda^2/2$, where
    $\cos\theta=\dotp{\vec v}{\vec v^{\ast}}$. We have that
    \begin{equation}\label{eq:square_bound}
        \snorm{2}{\vec v+\lambda\vec{ u}}= \sqrt{1+\lambda^2\snorm{2}{\vec{ u} }^2
            +2\lambda \dotp{\vec{ u} }{\vec{ v} } }\leq 1+\lambda^2\snorm{2}{\vec{ u} }^2 \;,
    \end{equation}
    where we used that $\sqrt{1+a}\leq 1+a/2$.
    Using the update rule, we have
    \begin{align*}
        \dotp{ \vec v'}{\vec v^{\ast}} & =  \dotp{ \vec v'}{(\vec v^{\ast})^{\perp_{\vec v}}}\sin\theta +  \dotp{ \vec v'}{\vec v}\cos\theta
        = \frac{\lambda\dotp{\vec{ u}}{ (\vec v^{\ast})^{\perp_{\vec v}} }}{\snorm{2}{\vec v+\lambda{\vec{u}} }}\sin\theta +
        \frac{\dotp{ {\vec v}+\lambda\vec{ u}}{{\vec v}}}{\snorm{2}{{\vec v}+\lambda{\vec{u}} }}\cos\theta \;.
    \end{align*}
    Now using \Cref{eq:square_bound}, we get
    \begin{align*}
        \dotp{\vec v'}{\vec v^{\ast}}
        & \geq \frac{\lambda\dotp{\vec{ u}}{(\vec v^{\ast})^{\perp_{\vec v}} }}{1+\lambda^2\snorm{2}{\vec{ u} }^2}\sin\theta +
        \frac{\cos\theta}{1+\lambda^2\snorm{2}{\vec{ u} }^2}
        = \cos\theta + \frac{\lambda\dotp{\vec{ u}}{ (\vec v^{\ast})^{\perp_{\vec v}} }}{1+\lambda^2\snorm{2}{\vec{ u} }^2}\sin\theta
        +\frac{-\lambda^2\snorm{2}{\vec{ u} }^2\cos\theta}{1+\lambda^2\snorm{2}{\vec{ u} }^2}\;.
    \end{align*}
    Then, using that $\dotp{\vec {u}}{\vec v^{\ast}}=\dotp{\vec{u}}{(\vec v^{\ast})^{\perp_{\vec v}}\sin\theta}$,
    we have  that $\dotp{\vec {u}}{(\vec v^{\ast})^{\perp_{\vec v}}} \geq \frac{c}{\sin\theta}$,
    thus
    \begin{align*}
        \dotp{{\vec v}'}{\vec v^{\ast}}
        & \geq \cos\theta + \frac{\lambda  c-\lambda^2\snorm{2}{\vec{ u}}^2}{1+\lambda^2\snorm{2}{\vec{u} }^2}
        \geq 	\cos\theta + \frac{\lambda  c-\lambda^2 }{1+\lambda^2\snorm{2}{\vec{ u}}^2}
        = \cos\theta + \frac {1}{4} \frac{c^2}{1+\lambda^2\snorm{2}{\vec{ u} }^2} \;,
    \end{align*}
    where in the first inequality we used that $\snorm{2}{\vec{u}} \leq 1$ and in the
    second that for $\lambda= c/2$ it holds $c-\lambda\geq  c/2$.
    Finally, we have that
    \begin{align*}
        \cos\theta'=\dotp{{\vec v}'}{\vec v^{\ast}} \geq \cos\theta + \frac{1}{4} \frac{c^2}{1+\lambda^2\snorm{2}{\vec u}^2 }
        \geq \cos\theta +\frac{1}{8}c^2 = \cos\theta +\frac{1}{2} \lambda^2 \;.
    \end{align*}
    This completes the proof.
\end{proof}

\section{Omitted Proofs from \Cref{sec:constant-bounded-arbitrary}}\label{app:upper-bound-polynomial}
\begin{lemma}[Theorem 8 of \cite{CW:01}]\label{lem:carbery-wright} Let $p:\R^d \mapsto R$ be a polynomial of degree at most $n$.
	Then there is an absolute constant $C>0$ such that for any $0<q<\infty$ and $t\geq 0$, it holds
	$ \pr_{\x\sim \normal_d}[|p(\x)|\leq \gamma]\leq C q \gamma^{1/n} (\E_{\x\sim \normal_d}[|p(\x)|^{q/n}])^{-1/q}\;.$
	
	\begin{lemma}
		
		\label{lem:taylor-exponential-polynomial-upper}
		Let $t \geq 0$. There exists an absolute constant $C'\geq 1$ such that for any univariate polynomial of degree $k$ it holds
		\[
		\frac{ \E_{x \sim \normal}[p^2(x) \1\{x \geq t\}]
		}
		{
			\E_{x \sim \normal}[p^2(x) \1\{x \leq t\}] 
		} 
		\leq  e^{C'k\log k - t^2/C' }  \,.
		\]
	\end{lemma}
	
	\begin{proof}
		We start by bounding from above the $\E_{x \sim \normal}[p^2(x) \1\{x \geq t\}]$. Using the Cauchy-Schwarz inequality, we get 
		\begin{equation}
			\label{eq:taylor-exp-app-eq1}
			\E_{x \sim \normal}[p^2(x) \1\{x \geq t\}]\leq  (\E_{x \sim \normal}[p^4(x) ])^{1/2}(\pr_{x\sim \normal}[x\geq t])^{1/2}\lesssim (\E_{x \sim \normal}[p^4(x) ])^{1/2}e^{-t^2/4}\;.
		\end{equation}
		In order to bound the $ \E_{x \sim \normal}[p^2(x) \1\{x \leq t\}] $ from below, we are going to use \Cref{lem:carbery-wright}. By setting $q=4k, n=k$ and $\gamma= (\frac{1}{12Ck})^{k} (\E_{x\sim \normal}[p^4(x)])^{1/4}$, on \Cref{lem:carbery-wright}. We get that
		\[
		\pr_{x \sim \normal}\left[|p(x)|\leq \left(\frac{1}{12Ck}\right)^{k} \left(\E_{x\sim \normal}[p^4(x)]\right)^{1/4}\right] \leq 1/3\;.
		\]
		Therefore, by squaring we get
		\begin{equation}
			\label{eq:taylor-exp-app-eq2}
			\pr_{x \sim \normal}\left[p^2(x)\geq \left(\frac{1}{12Ck}\right)^{2k} \left(\E_{x\sim \normal}[p^4(x)]\right)^{1/2}\right] \geq 2/3\;.
		\end{equation}
		Furthermore,  using the assumption that $t\geq 0$, we have $ \E_{x \sim \normal}[p^2(x) \1\{x \leq t\}]\geq \E_{x \sim \normal}[p^2(x) \1\{x \leq 0\}]$, hence, combing the with \Cref{eq:taylor-exp-app-eq2}, we get
		\begin{equation}
			\label{eq:taylor-exp-app-eq3}
			\E_{x \sim \normal}[p^2(x) \1\{x \leq t\}]\gtrsim \left(\frac{1}{12Ck}\right)^{2k} \left(\E_{x\sim \normal}[p^4(x)]\right)^{1/2}\;.
		\end{equation}
		Combining \cref{eq:taylor-exp-app-eq1} and \cref{eq:taylor-exp-app-eq3}, we get
		\[\frac{
			\E_{x \sim \normal}[p^2(x) \1\{x \geq t\}]
		}
		{
			\E_{x \sim \normal}[p^2(x) \1\{x \leq t\}] 
		}  \lesssim 
(12Ck)^{2k}e^{-t^2/4} \leq 
		e^{C'k\log k - t^2/C' } \,,
		\]
		for some $C'\geq 1$ absolute constant.
	\end{proof}
	
\end{lemma}
We prove below \Cref{lem:empirical_objective_error}, we restate it for convenience.
    \begin{fact}[Estimation of $\vec M$]
        Let $\Omega = \{\vec A \in \mathcal{S}^m:\ \snorm{F}{\vec A}
        \leq 1\}$ and $\eps,\delta\in(0,1)$. Let $\ell(\x)=\vec w \cdot \x+t$ with $|\ell(\x)|^2\leq C$ and $\widetilde{\vec M}
        = \frac{1}{N} \sum_{i=1}^N
        \vec m(\sample{\bx}{i}) \vec m(\sample{\bx}{i})^\top \1_B(\sample{\bx}{i})\sample{y}{i}
        \ell(\sample{\bx}{i})$. There exists an algorithm that draws
        $N =
        \frac{ d^{O(\log\frac{1}{\gamma \beta})}}{C\eps^2}
        \log(1/\delta)$
        samples from $\D$, runs in $\poly(N,d)$ time and
        with probability at least $1-\delta$ outputs a matrix $\wt{\vec M}$ such
        that
        $$
        \pr
        \left[
        \sup_{\vec A \in \Omega}
        \left| \tr(\vec A \wt{\vec M}) - \tr(\vec A \vec M) \right|
        \geq \eps
        \right]
        \leq 1-\delta\, .
        $$
    \end{fact}
\begin{proof}
	Using the Cauchy-Schwarz inequality, we get
	$$
	\tr\left( \vec A (\vec M - \widetilde{\vec M}) \right)
	\leq
	\snorm{F}{\vec A} \snorm{F}{\vec M - \widetilde{\vec M}}
	\;. $$
	Therefore, it suffices to bound the probability
	that $\snorm{F}{\vec M - \widetilde{\vec M}} \geq \eps$.
	From Markov's inequality, we have
	\begin{equation}
		\label{eq:moment_matrix_markov}
		\pr\left[\snorm{F}{\vec M - \widetilde{\vec M}}
		\geq \eps \right]
		\leq \frac{1}{\eps^2} \E\left[\snorm{F}{\vec M - \widetilde{\vec M}}^2\right]\,.
	\end{equation}
	Using multi-indices $S_1$, $S_2$ that correspond to the monomials
	$\bx^{S_1}, \bx^{S_2}$ (as indices of the matrix $\vec M$), we have
	$$
	\E\left[\snorm{F}{\vec M - \widetilde{\vec M}}^2\right]
	= \sum_{S_1, S_2: |S_1|, |S_2|\leq k}
	(\vec M_{S_1,S_2} - \wt {\vec M}_{S_1,S_2})^2
	= \sum_{S_1, S_2: |S_1|, |S_2|\leq k}
	\var[\wt{\vec M}_{S_1,S_2}]\,.
	$$
	Using the fact that the samples $(\sample{\bx}{i},\sample{y}{i})$ are
	independent, we can bound from above the variance of each entry $(S_1,S_2)$ of
	$\wt{\vec M}$
	\begin{align*}
		\var[\wt{\vec M}_{S_1,S_2}]
		&\leq \frac{1}{N}
		\E_{(\bx, y) \sim \D}
		\left[ \bx^{2(S_1+S_2)}
		\left(\1_B(\bx) \ell(\x) y\right)^2
		\right] \leq
		\frac{2}{N}
		\E_{\bx \sim \D_{\bx}}
		\left[ \bx^{2(S_1+S_2)} (\snorm{2}{\bx}^{2} +t^2)\right] \\
		&\leq
		\frac{2 C}{N}
		\E_{\bx \sim \D_{\bx}}
		\left[ (\snorm{2}{\bx}^{2})^{|S_1+S_2|+1} \right] \,.
	\end{align*}
	For every $n \geq 1$, we have $
	\E_{\bx \sim \D_{\bx}}
	\left[ (\snorm{2}{\bx}^{2})^{n} \right]= d^{O(n)}$.
	Using the above bound for the variance and summing over all pairs $S_1, S_2$
	with $|S_1|, |S_2| \leq k$, we obtain
	\begin{align}
		\label{eq:momemt_matrix_frobenius_bound}
		\E\left[\snorm{F}{\vec M - \widetilde{\vec M}}^2\right]
		&=
		\frac{2C}{N}d^{O(k)}
	\end{align}
	Combining
	\Cref{eq:moment_matrix_markov} and
	\Cref{eq:momemt_matrix_frobenius_bound} we obtain that with $N = d^{O(m)}/(C\eps^2)$ samples we can estimate $\vec M$
	within the target accuracy with probability at least $3/4$. To amplify the
	probability to $1-\delta$, we can simply use the above empirical estimate $\ell$
	times to obtain estimates $\sample{\wt {\vec M}}{1}, \ldots, \sample{\wt{\vec
			M}}{\ell}$ and keep the coordinate-wise median as our final estimate. It
	follows that $O(\log(m/\delta))$ repetitions suffice to guarantee
	confidence probability at least $1-\delta$.

\end{proof}
We restate and prove \Cref{lem:algorithm_function_ell}.
\begin{lemma}[Estimating the function $r_i$]
    Let $\D$ be a distribution on $\R^d\times\{\pm 1\}$ with standard normal $\x$-marginal and let $T\ith(\x)$ be a non-negative function returned by a $(2\rho)$-certificate oracle. Moreover, assume that $T\ith(\x)$ has bounded $\ell_4$ norm, i.e., $\|T\ith(\x)\|_4\leq 1$.
    Then after drawing $O(d\log(1/\eps)/\eps^2\log(d/\delta))$ samples from $\D$, 
    with probability at least $1-\delta$, we can compute an estimator $\hat{r}_i$ that
    satisfies the following conditions:
    \begin{itemize}
        \item $\snorm{2}{\nabla \hat r_i(\vec w,t) - \E_{(\bx,y) \sim \D} [(T\ith(\vec x) +\rho) y(\x,1)]} \leq \eps/\sqrt{\log(1/\eps)}$

        \item $\snorm{2}{\nabla \hat r_i(\vec w,t)} \leq 2\sqrt{d}$.\end{itemize}
\end{lemma}
\begin{proof}
	Let $\D_N$ be the empirical distribution of $\D$ with $N$ samples. It suffices to find $N$, such that with probability $1-\delta$ the
	$$ \left\| \E_{(\x,y)\sim \D_N}[(T(\x)+\rho)y\x ] -\E_{(\x,y)\sim \D}[(T(\x)+\rho)y\x ]  \ \right\|_2\leq \eps\;.$$
	Let $\widehat{\vec R}= \E_{(\x,y)\sim \D_N}[(T(\x)+\rho)y\x ]$ and $\vec R=\E_{(\x,y)\sim \D}[(T(\x)+\rho)y\x ]$. From Markov's inequality, we have that
	\begin{align*}
		\pr\left[\left\| \widehat{\vec R}-\vec R \ \right\|_2\geq \eps\right]\leq\frac{1}{\eps^2} \E[\| \widehat{\vec R} -\vec R\|_2^2] \leq \frac{1}{N \eps^2}\E_{(\x,y)\sim \D}[\|(T(\x)+\rho)y\x\|_2^2] \;.
	\end{align*}
From Cauchy-Schwarz, we have that $\E_{\x\sim \D_\x}[\|(T(\x)+\rho)\x\|_2^2]\leq (\E_{\x\sim \D_\x}[(T(\x)+\rho)^4])^{1/2} (\E_{\x\sim \D_\x}[\|\x\|_2^4])^{1/2}$. From the fact that $\|T\|_4\leq 1$, we have that $\E_{\x\sim \D_\x}[(T(\x)+\rho)^4]\leq 4(1+\rho^4)$. Moreover, $(\E_{\x\sim \D_\x}[\|\x\|_2^4])^{1/2}\lesssim d$, hence
	\[
	\pr\left[\left\| \widehat{\vec R}-\vec R \ \right\|_2\geq \eps\right]\lesssim\frac{4d} {N\eps^2}\;.
	\] 
	For $N'=O(\frac{d}{\eps^2})$ samples, we get that the above holds with constant probability. Like in \Cref{lem:empirical_objective_error}, we can amplify the probability to $1-\delta$, using $N=N'\log(d/\delta)$ samples.
	
	Therefore, let $\widetilde{\vec R}$ be the estimator of $\vec R$, we have  
	\begin{align}
		\snorm{2}{\widetilde{\vec R}} \leq \snorm{2}{\vec R}+\eps \leq  4\sqrt{(1+\rho^2)}\sqrt{d} +\eps\leq 8\sqrt{d}\;.
	\end{align}
	Finally, it remains to show that we can compute an estimator $\widetilde{\vec M}=(\widetilde{\vec R},\tilde{m})$ such that for any unit vector $\vec v\in \R^d$ and any $|t|\leq 4\sqrt{\log(1/\eps)}$, we have with probability $1-\delta$ that
	\[
	| \widetilde{\vec R}\cdot \vec v +\tilde{m}t -\E_{(\x,y)\sim \D}[(T(\x)+\rho)y(\x\cdot \vec v+t) ] |\leq \eps\;.
	\]
	From the above, for any unit vector $\vec v\in \R^d$, we have with probability $1-\delta$ that
	$ |\widetilde{\vec R}\cdot \vec v -\vec R\cdot \vec v|\leq \|\widetilde{\vec R} -\vec R\|_2\leq \eps$.
	Therefore, it remains to estimate $\widetilde{m}$. Let $\D_{N'}$ be the empirical distribution of $\D$ with $N'$ samples. From Markov's inequality we have for any $\eps'>0$ that
	\[
	\pr[|\E_{(\x,y)\sim \D_{N'}}[ (T(\x)+\rho)yt] -\E_{(\x,y)\sim \D}[ (T(\x)+\rho)yt]|\geq \eps' ]\leq \frac{\E_{(\x,y)\sim \D}[ (T(\x)+\rho)^2]t^2}{N'\eps^2}\leq \frac{4t^2}{N'\eps^2}\;.
	\]
	Therefore, using $N'=O(\log(1/\eps)/\eps^2)$ we get an estimator with constant probability and hence by using the boosting technique as before, we get that with overall $N'\log(1/\delta)$ samples, we can boost the probability to $1-\delta$.
\end{proof}

 \section{Omitted Proofs from \Cref{sec:massart-biased-lower-bound}} \label{app:lower-bounds}
In this section, we prove \Cref{lem:feasibility}. We restate the lemma for convenience.
        \begin{lemma}
            If there is no polynomial $p\in \mathcal P_k^0$ such that $\beta\|(pf)^+ \|_1\geq \|(pf)^-\|_1 $ then, the \Cref{eq:primal-LP-2} is feasible if only if \Cref{eq:primal-LP-3} is infeasible.
        \end{lemma}
  \begin{proof}
         First we introduce some notation.
        We use $(\tilde{h}, c)$ for the inequality $\E_{z \sim \normal}[\beta(z) \tilde{h}(z)] +c\leq 0$,
        where $\tilde{h}\in L^1(\R)$ and $c\in \R$. Moreover, let $\cal S$ be the set that contains all such tuples that describe the target system. For the set $\cal S$, the closed convex cone over $L^1(\R)\times R$ is the smallest closed set ${\cal S}_+$ satisfying, if $A\in {\cal S}_+$ and $B\in {\cal S}_+$ then $A+B \in {\cal S}_+$ and, if $A\in {\cal S}_+$ then $\lambda A \in {\cal S}_+$ for all $\lambda\geq 0$. Note that the ${\cal S}_+$ contains the same feasible solutions as $\cal S$.
        In order to prove this, we need the following functional analysis result from \cite{Fan68}.
        \begin{fact}[Theorem 1 of \cite{Fan68}]\label{fct:fan}
            If $\cal X$ is a locally convex, real separated vector space then, a linear system described by $\cal S$ is feasible (i.e., there exists a $g\in {\cal X}^*$) if and only if $(0,1)\not\in {\cal S}_+$.
        \end{fact}
        Our LP is defined by the following inequalities: $(pf,0)$ for $p\in {\cal P}_{k}^0$, $(h,-\|h\|_1)$ for all $h\in L_+^1(\R)$, $(-T,\beta\|T\|_1)$ for all $T\in L_+^1(\R)$. By taking the convex cone defined from the above inequalities, we get the following set
        $${\cal S}_+'=\{(pf+h-T,-\|h\|_1+\beta \|T\|_1):p\in{\cal P}_{k}^0,h\in L_+^1(\R),T\in L_+^1(\R)\}\;.$$  We have the following lemma.
        \begin{lemma}\label{lem:feasibility-app}
            If  there is no polynomial $p\in \mathcal P_k^0$ such that $\beta\|(pf)^+ \|_1\geq \|(pf)^-\|_1 $ then, the $LP$ described by $\cal S$ is feasible if only if $(0,1)\not\in {\cal S}_+'$.
        \end{lemma}
        \begin{proof}
            This proof follows from \Cref{fct:fan} by showing that ${\cal S}_+'$ is closed, i.e., 
            ${\cal S}_+' = {\cal S}_+$.
             Let $F = fp + h -T$ and $\lambda \geq \|h\|_1-\beta \|T\|_1$, therefore the above LP is equivalent to 
             $\E_{z\sim \normal}[\beta(z)F(z)]\leq \lambda$. 
             Let $(p_n,h_n,T_n,\lambda_n)$ be a sequence that satisfies our constraints, i.e., $(fp_n+h_n-T_n,\lambda_n)\in S_+'$, with $F_n=fp_n+h_n-T_n$ converging with respect the $\ell_1$-norm to $F_0$ and $\lambda_n$ converging to $\lambda_0$.

             Note that $\lambda_n\geq \| (F_n - (fp_n))^+ \|_1 -\beta \|(F_n - (fp_n))^-  \|_1$. This holds because $h_n$ and $T_n$ are positive functions and $F_n -fp_n=h_n-T_n$, therefore $h_n=(F_n - (fp_n))^+$ and
             $T_n= (F_n - (fp_n))^-$. We have that $$ \| (F_n - (fp_n))^+ \|_1 -\beta \|(F_n - (fp_n))^-  \|_1\geq -2 \|F_n\|_1 + \|(fp_n)^-\|_1-\beta \|(fp_n)^+\|_1\;.$$
            Note that from our assumption we have that for any polynomial $p\in \mathcal P_k^0$, it holds
            $\|(pf)^-\|_1 -\beta\|(pf)^+ \|_1>0$. Because this is homogeneous with respect $p$, we can assume that $\|p\|_1=1$.  Therefore, using the fact that the set $\|p\|_1=1$ and $\mathcal P_k^0$ is compact, we have that all the limits are inside $\mathcal P_k^0$, thus there exists a $c>0$, such that  $\|(pf)^-\|_1 -\beta\|(pf)^+ \|_1>c$. Therefore, it holds   $\|(pf)^-\|_1 -\beta\|(pf)^+ \|_1>c\|p\|_1$.

            From the above, we have that
            $\lambda_n +2\|F_n\|_1\geq c \|p_n\|_1\;,$
            and because $F_n\rightarrow F_0$ and $\lambda_n\rightarrow \lambda_0$, that means that $\|p_n\|_1$ is bounded.
             Since, an $L^1$ ball in $\mathcal P_k^0$ is compact, there is a subsequence such that  $p_n\rightarrow p_0$.
        	 By setting $h_0=(F_0-(fp_0))^+$ and $T_0=(F_0-(fp_0))^-$, we find appropriate $p_0,h_0,T_0$ that give the appropriate limit points. Therefore, the set $\mathcal S_+'$ is closed and the lemma follows from \Cref{fct:fan}.
        \end{proof}
        The proof of \Cref{lem:feasibility}, follows from \Cref{lem:feasibility-app}, by noting that $(0,1)\not\in {\cal S}_+'$ is equivalent to the infeasibility of \Cref{eq:primal-LP-3}.
\end{proof}

\end{document}